\documentclass[11pt]{article}
\usepackage{amsmath,amsthm,amsfonts,amssymb}
\usepackage{xhfill}
\usepackage{booktabs}
\usepackage{fancyhdr}
\usepackage{natbib}
\usepackage{hyperref}
\usepackage{titlesec}
\usepackage{parskip}
\usepackage{dsfont}
\usepackage{enumitem}
\usepackage{verbatim}
\usepackage{multirow}
\usepackage{graphicx}
\usepackage{longtable}
\usepackage{makecell}
\usepackage{enumitem}
\usepackage{algorithm}
\usepackage{algorithmic}
\usepackage{caption}
\usepackage{tcolorbox}
\usepackage{etoc}
\usepackage{amsmath}
\tcbuselibrary{skins,breakable}

\newcommand{\R}{\mathbb{R}}
\newcommand{\E}{\mathbb{E}}
\renewcommand{\P}{\mathbb{P}}
\renewcommand{\le}{\leqslant}
\renewcommand{\ge}{\geqslant}
\newcommand{\z}{{\bf z}}
\renewcommand{\v}{{\bf v}}
\renewcommand{\u}{{\bf u}}
\newcommand{\I}{{\bf I}}
\newcommand{\J}{{\bf J}}
\renewcommand{\varsigma}{\boldsymbol{\zeta}}
\DeclareFontFamily{U}{mathx}{}
\DeclareFontShape{U}{mathx}{m}{n}{<-> mathx10}{}
\DeclareSymbolFont{mathx}{U}{mathx}{m}{n}
\DeclareMathAccent{\widecheck}{0}{mathx}{"71}

\AtBeginEnvironment{proposition}{\vspace{6pt}}
\AtBeginEnvironment{lemma}{\vspace{6pt}}
\AtBeginEnvironment{theorem}{\vspace{6pt}}
\AtBeginEnvironment{corollary}{\vspace{6pt}}
\AtBeginEnvironment{remark}{\vspace{6pt}}
\AtBeginEnvironment{definition}{\vspace{6pt}}

\AtEndEnvironment{definition}{\vspace{6pt}}
\AtEndEnvironment{proof}{\vspace{6pt}}
\AtEndEnvironment{remark}{\vspace{6pt}}
\AtEndEnvironment{example}{\vspace{6pt}}

\hypersetup{colorlinks,linkcolor={red!50!black},citecolor={blue!50!black},urlcolor={blue!80!black}}
\theoremstyle{definition}
\newtheorem{definition}{Definition}[section]
\newtheorem{lemma}{Lemma}[section]
\newtheorem{proposition}{Proposition}[section]
\newtheorem{theorem}{Theorem}[section]
\newtheorem{corollary}{Corollary}[section]
\newtheorem{remark}{Remark}[section]

\renewenvironment{proof}{{\bfseries Proof.}}{\unskip\nobreak\hfill $\square$}

\title{
	\huge
	On Finite-sample Concentration of Median of Incomplete U-Statistics
} \date{}
\author{
	\large Nong Minh Hieu\thanks{mh.nong.2024@phdcs.smu.edu.sg}, Antoine Ledent\thanks{aledent@smu.edu.sg} \\
	\small Singapore Management University, School of Computing and Information Systems
}
\date{}

\begin{document}
	\maketitle
	
	\begin{abstract}
		Median-of-means (MoM) is a powerful technique that theoretically enables near sub-Gaussian finite-sample rate for parameter estimation when the underlying data distribution is heavy-tailed (e.g., assumed to have only two first finite moments). A recent work has extrapolated this technique to median-of-\textit{randomized}-U-Statistics (MoRU) and median-of-\textit{incomplete}-U-Statistics (MoIU) for estimating expectations of heavy-tailed pairwise kernels. In \citet{pmlr-v97-clemencon19a}, a concentration rate that scales like $O(n^{-1/2})$ with sample size has been proven for MoRU. However, despite the computational advantage of the latter, the analysis of finite-sample bound for MoIU remains a significant theoretical challenge. As noted by the authors, a straightforward application of McDiarmid's inequality yields a loose bound of order $O(n^{-1/4})$. In this work, we prove a finite-sample concentration bound for the MoIU estimator that scales as $O(n^{-1/2})$ with respect to the sample size using a delicate convex decomposition approach. Furthermore, we show that our proof can be seamlessly extended to geometric median in multivariate settings. Using a Serfling-type argument, we extrapolate our results into a regime where data pairs are selected without replacement across blocks, breaking the usual block-wise independence condition. Then, using a Bernstein-type treatment for U-Statistics, we tighten the dependency of our bounds on the margin $\tau$ from $O(\tau^{-3/2})$ achieved in the previous work to $O(\tau^{-1})$. Finally, we proved anti-concentration results that is applicable for all median estimators presented in this work to demonstrate that $M\le O(n)$ is an intrinsic restriction on block sizes.
	\end{abstract}
	\etocdepthtag.toc{mainbody}     
	
	\section{Introduction}
	Median of Means (MoM) is a powerful technique in parameters estimation that enables near sub-Gaussian finite-sample performance guarantee when the underlying data is heavy-tailed, i.e., assumed to have only (first) two finite moments. In the seminal work of \citet{Lugosi2019-lu}, the finite-sample concentration rate is established for the mean estimation problem using median-of-means computed from non-overlapping blocks. Specifically, given a sample of i.i.d. random variables $X_1,\dots,X_n\simeq X$ with variances $\mathrm{Var}(X)=\sigma^2<\infty$ partitioned into $K\gtrsim \log(1/\delta)$ (for some constant $\delta\in(0,1)$) blocks of size $M\triangleq\lfloor n/K\rfloor$, it is proven that:
	\begin{align}
		\label{eq:subgaussian-rate-mom}
		\left|\widehat\mu_\mathrm{MoM}-\mu\right|\lesssim \sigma\sqrt{\frac{\log(1/\delta)}{n}},
	\end{align}
	\noindent with probability of at least $1-\delta$ where $\widehat\mu_\mathrm{MoM}\triangleq\mathrm{med}(\widehat\mu_1,\dots,\widehat\mu_K)$ with $\widehat\mu_k$ being means computed using data points belonging to block $k\le K$. At the core of the MoM technique is a simple amplification principle: although each block mean only enjoys a weak Chebyshev-type guarantee, with an appropriate block size, each block mean is accurate with probability bounded away from $\frac{1}{2}$. Taking the median across $K$ independent blocks then amplifies this guarantee to an exponentially small failure probability in $K$, yielding the near-sub-Gaussian rate in Eqn.~\eqref{eq:subgaussian-rate-mom}.
	
	\noindent In a later extended work, \citet{pmlr-v97-clemencon19a} extrapolates median-of-means techniques to parameter estimation problems involving pairwise kernels under heavy-tailed settings. Specifically, given an i.i.d. sample $S_n=\{X_1,\dots,X_n\}$ drawn from a distribution over a data space $\mathcal{Z}$ and a pairwise kernel $h:\mathcal{Z}\times\mathcal{Z}\to\R$, the authors proposed different variations of median-of-\textit{U-Statistics} estimators for $\theta(h)\triangleq\E[h(X_1, X_2)]$. Most notably, the authors proposed the median-of-\textit{randomized}-U-Statistics (MoRU) where the estimator $\bar\theta_\mathrm{MoRU}(h)$ of $\theta(h)$ is formulated as the median of $K$ order-two U-Statistics computed from fixed size subsets of $S_n$ randomly chosen without replacement. Particularly, for all $k\in[K]$, let ${\bf I}_k=\{i_1,\dots,i_M\}\subseteq[n]$ be a subset of indices selected without replacement (where the selections of ${\bf I}_1,{\bf I}_2,\dots,{\bf I}_K$ are independent of each other), $\bar\theta_\mathrm{MoRU}$ is defined as:
	\begin{align}
		\label{eq:moru}
		\bar\theta_\mathrm{MoRU}(h) &\triangleq \mathrm{med}\left(\bar U_{M, {\bf I}_1}(h),\dots,\bar U_{M, {\bf I}_K}(h)\right), \\
		\forall k\in[K]: \bar U_{M,{\bf I}_k}(h)&\triangleq\frac{2}{M(M-1)}\sum_{(i,j)\in{\bf I}_k^2, i<j} h(X_i, X_j).
	\end{align}
	\noindent Before presenting our main results, we state and provide the proof sketch of the following result by \citet{pmlr-v97-clemencon19a} to acquaint readers with the general strategy to analyze median-of-means type estimators with block randomization.
	\begin{theorem}[cf. \cite{pmlr-v97-clemencon19a}, Proposition 3]
		\label{thm:laforgue-moru}
		For $\delta\in[e^{1-Cn}, 1)$ and $K\ge \bar C\log(1/\delta)$ where $C,\bar C$ are absolute constants, with probability of at least $1-\delta$, we have:
		\begin{align}
			|\bar\theta_\mathrm{MoRU}(h)-\theta(h)| &\lesssim \sqrt{\frac{\sigma_1^2(h)\log(1/\delta)}{n} + \frac{\sigma_2^2(h)\log^2(1/\delta)}{n(n-\log(1/\delta))}},
		\end{align}
		\noindent where $\sigma_1^2(h)\triangleq\mathrm{Var}(\E[h(X_1,X_2)|X_1])$ and $\sigma_2^2(h)\triangleq\mathrm{Var}(h(X_1,X_2))-2\sigma_1^2(h)$.
	\end{theorem}
	
	\begin{proof}\ (Sketch)
		For all $k\in[K]$, we define $Z_k^\varepsilon=\lambda_\varepsilon\big(\bar U_{M, {\bf I}_k}(h)\big)$ where $\lambda_\varepsilon(z)\triangleq\mathds{1}_{|z-\theta(h)|\ge\varepsilon}$. Then, each $Z_k^\varepsilon$ is a Bernoulli random variable with probability $\bar p_\varepsilon\triangleq\E[Z_k^\varepsilon]=\P(|\bar U_{M, \I_k}(h)-\theta(h)|\ge\varepsilon)$. By Chebyshev's inequality:
		\begin{align}
			\bar p_\varepsilon\le \frac{\mathrm{Var}\big(\bar U_{M,\I_k}(h)\big)}{\varepsilon^2} = \frac{1}{\varepsilon^2}\left[\frac{4\sigma_1^2(h)}{M} + \frac{2\sigma_2^2(h)}{M(M-1)}\right].
		\end{align}
		\noindent Then, notice that we have the following important events inclusion:
		\begin{align}
			\Big\{|\bar\theta_\mathrm{MoRU}(h)-\theta(h)|\ge\varepsilon\Big\} \subseteq \left\{\frac{1}{K}\sum_{k=1}^K Z_k^\varepsilon\ge\frac{1}{2}\right\}.
		\end{align}
		\noindent This means that the deviation of $\bar\theta_\mathrm{MoRU}$ with error $\varepsilon$ means that at least half of the sub-estimators have to deviate from $\theta(h)$ with error $\varepsilon$. However, unlike the regular MoM case where data blocks are independent (non-overlapping), the indicator variables $Z_k^\varepsilon$ are \textbf{not} independent due to the chances of blocks sharing data points. However, they are \textbf{conditionally independent} given the entire dataset. Hence, we can use the following decomposition:
		\begin{align}
			&\P(|\bar\theta_\mathrm{MoRU}(h)-\theta(h)|\ge\varepsilon) \nonumber \\ 
			&\le \P\left(\frac{1}{K}\sum_{k=1}^K Z_k^\varepsilon\ge\frac{1}{2}\right)=\P\left(\frac{1}{K}\sum_{k=1}^K Z_k^\varepsilon-\bar p_\varepsilon\ge\frac{1}{2}-\bar p_\varepsilon\right)\\
			&\le \E_{S_n}\left[\P\left(\frac{1}{K}\sum_{k=1}^K Z_k^\varepsilon-\bar W_n^\varepsilon\ge\frac{1}{2} - \tau\Big|S_n\right)\right] + \P(\bar W_n^\varepsilon-\bar p_\varepsilon\ge \tau-\bar p_\varepsilon).
		\end{align}
		\noindent Here, $\bar W_n^\varepsilon\triangleq\E[Z_k^\varepsilon|S_n]=\P(|\bar U_{M,\J_k}(h)-\theta(h)|\ge\varepsilon|S_n)$ - the deviation probability given the sample $S_n$ and $\tau$ is a free variable to be determined. Since $Z_1^\varepsilon,\dots,Z_K^\varepsilon$ are now independent given the samples, the first failure probability can be treated using either Chernoff or Hoeffding's inequality to yield a $\exp\left({-K(\frac{1}{2}-\tau)^2}\right)$ rate. The more difficult segment of the proof lies in the second failure probability. However, since $\bar W_n^\varepsilon$ is a U-Statistics with bounded kernels, specifically:
		\begin{align}
			\bar W_n^\varepsilon = \frac{1}{\binom{n}{M}}\sum_{{\bf I}\in C_{n,M}}\lambda_\varepsilon\left(\bar U_{M,\I}(h)\right), \qquad \bar U_{M,\I}(h)&\triangleq\frac{2}{M(M-1)}\sum_{(i,j)\in{\bf I}^2, i<j} h(X_i, X_j).
		\end{align}
		\noindent We can invoke the classic Hoeffding-type bound for U-Statistics with bounded kernels and achieve $\P(\bar W_n^\varepsilon-\bar p_\varepsilon\ge \tau-\bar p_\varepsilon)\lesssim \exp\left(-\frac{n}{M}(\tau-\bar p_\varepsilon)^2\right)$. Using the Chebyshev bound to further control $\bar p_\varepsilon$, we obtain the desired result in \citet{pmlr-v97-clemencon19a}, Proposition 3.
	\end{proof}
	
	\noindent Apart from the above result, \citet{pmlr-v97-clemencon19a} also derived finite-sample concentration rate for the median-of-\textit{disjoint}-U-Statistics estimator, denoted $\bar\theta_\mathrm{MoU}(h)$, where each sub-estimator is a complete U-Statistic computed from disjoint data blocks partitioned from the original dataset $S_n$. Unlike the $\bar\theta_\mathrm{MoU}$ and the classical median-of-means \citep{Lugosi2019-lu}, the analysis of $\bar\theta_\mathrm{MoRU}$ is nuanced in the sense that each Bernoulli indicator $Z_k^\varepsilon$ are not independent unless conditioned on $S_n$. This forces the analysis to first consider the concentration of $\frac{1}{K}\sum_{k=1}^KZ_k^\varepsilon$ around the conditional expectation $\bar W_n^\varepsilon$ given the data $S_n$, then analyze the concentration of $\bar W_n^\varepsilon$ around $\bar p_\varepsilon$ marginally over the data. Since $\bar W_n^\varepsilon$, a data-functional, happens to be a U-Statistic with bounded kernel of order $M$ over the original i.i.d. sample $S_n$, the application of Hoeffding-type inequality \citep{article:ginezinn1984, article:arconesgine1993} for U-Statistics is straightforward. However, for this work's primary object of study, the median-of-\textit{incomplete}-U-Statistics (MoIU) estimator, the equivalent conditional expectation becomes significantly more complex.
	
	\section{Problem Formulation}
	\textbf{Open Problem}: In median-of-\textit{incomplete}-U-Statistics (MoIU), instead of selecting data blocks of size $M$ without replacement from $S_n$ then computing full U-Statistics, each sub-estimator is computed by averaging over $M$ pairs selected \textit{with replacement} from the collection $\Lambda_n$ defined as:
	\begin{align}
		\Lambda_n\triangleq\Big\{(i,j): 1\le i <j \le n\Big\}.
	\end{align}
	\noindent Particularly, for each $k\in[K]$, let ${\bf J}_k=\big\{\left(j_{k, 1}, i_{k, 1}\right), \dots, \left(j_{k, M}, i_{k, M}\right)\big\}$ be independent draws with replacement from the set of all pairs $\Lambda_n$. Then, $\widehat\theta_\mathrm{MoIU}(h)$ is defined as:
	\begin{align}
		\label{eq:moiu}
		\widehat\theta_\mathrm{MoIU}(h) &\triangleq \mathrm{med}\left(\widehat U_{M,\J_1}(h), \dots,\widehat U_{M,\J_K}(h)\right), \\
		\forall k\in[K]: \widehat U_{M,\J_k}(h) &\triangleq \frac{1}{M}\sum_{(i, j)\in{\bf J}_k{}} h\left(X_{i}, X_{j}\right),\qquad {\bf J}_k\sim_\mathrm{wr} \Lambda_n.
	\end{align}
	\noindent Here, we use the notation $\sim_\mathrm{wr}$ to denote simple random selection with replacement from the set of all pairs $\Lambda_n$. As usual, the sub-estimators $\widehat U_{M,\J_1}(h),\dots,\widehat U_{M,\J_K}(h)$ are not independent due to the non-zero probability of repeating pairs in each incomplete U-Statistic. Therefore, if we follow the same line of argument as Theorem \ref{thm:laforgue-moru} and let $Z_k^\varepsilon=\lambda_\varepsilon\big(\widehat U_{M,\J_k}(h)\big)$ for all $k\in[K]$, we have to analyze the concentration rate of the conditional expectation $\widehat W_n^\varepsilon\triangleq\E[Z_k^\varepsilon|S_n]$ around $\widehat q_\varepsilon\triangleq\E[\widehat W_n^\varepsilon]$. 
	
	\textbf{The Challenge}: However, as \citet{pmlr-v97-clemencon19a} noted in their work, $\widehat W_n^\varepsilon$ is a very complex data-functional that does not admit the nice U-Statistic formulation like $\bar W_n^\varepsilon$. Particularly:
	\begin{align}
		\label{eq:vstats-ugly}
		\widehat W_n^\varepsilon =\frac{1}{|\Lambda_n^M|}\sum_{{\bf J}\in\Lambda_n^M}\lambda_\varepsilon\left(\widehat U_{M,\J}(h)\right),\qquad \widehat U_{M,\J}(h)\triangleq\frac{1}{M}\sum_{(j,i)\in{\bf J}}h(X_j,X_i).
	\end{align}
	\noindent Essentially, $\widehat W_n^\varepsilon$ is a von-Mises-Statistic (V-Statistic) of order $M$ over the collection of all data pairs $\mathcal{P}_n\triangleq\{(X_j, X_i):(i,j)\in\Lambda_n\}$. Usually, a V-Statistic can be re-written as a U-Statistic with a modified kernel \citep{article:hoeffding1948}. However, even if we utilized such a re-formulation trick, the key technical challenge is that the support $\mathcal{P}_n$ is composed of dependent data pairs. Therefore, leveraging Hoeffding-type arguments from $\widehat W_n^\varepsilon$ directly is fundamentally challenging. As remarked in \citet[Appendix C.2]{pmlr-v97-clemencon19a}, ``the random quantity $\widehat W_n^\varepsilon$ is very difficult to study, due to the complexity of the data functional". The authors also noted that with a straightforward application of McDiarmid's inequality, one can obtain $\P(\widehat W_n^\varepsilon-\widehat q_\varepsilon\ge t)\lesssim \exp\left(-\frac{nt^2}{\textcolor{black}{M^2}}\right)$. However, this concentration rate subsequently yields a dominant dependency of $O(n^{-1/4})$ on the sample size, which, apparently is not sharp enough compared to the $O(n^{-1/2})$ rates proven for $\bar\theta_\mathrm{MoU}$ (independent blocks) and $\bar\theta_\mathrm{MoRU}$ (randomized blocks). In order to recover at least a square-root dependency on $n$, we need a concentration bound that scales like $\P(\widehat W_n^\varepsilon-\widehat q_\varepsilon\ge t)\lesssim \exp\left(-\frac{nt^2}{\textcolor{black}{M}}\right)$. 
	
	\begin{table}[t]
		\centering
		\caption{Summary of results and definitions of different estimators. The first column specifies the notation for different median-of-U-Statistics estimators, which are computed as medians of their corresponding sub-estimators. The sampling operators ($\sim_\mathrm{wr},\sim_\mathrm{wor}$) denote either sampling with replacement or without replacement. The semi-last column specifies whether the sub-estimators are independent of each other conditionally given $S_n$. ``B.I." stands for ``Block-wise Independence".}
		\label{tab:results}
		\begin{tabular}{cccccc}
			\toprule
			\textbf{Est.} & \makecell{\textbf{Block} \textbf{Sampling}} & \textbf{Sub-estimator} & \makecell{\textbf{B.I.}} & \textbf{Results} \\
			\midrule
			\makecell{$\bar\theta_\mathrm{MoRU}$} & $\forall k\in[K], {\bf I}_k\sim_\mathrm{wor}[n]$ & $\bar U_{M,\I_k}\triangleq\frac{2}{M(M-1)}\sum_{i<j\in{\bf I}_k^2} h(X_i, X_j)$ & Yes & Thm. \ref{thm:laforgue-moru}\\
			\makecell{$\widehat\theta_\mathrm{MoIU}$} & $\forall k\in[K], {\bf J}_k\sim_\mathrm{wr}\Lambda_n$ & $\widehat U_{M,\J_k}\triangleq \frac{1}{M}\sum_{(i, j)\in{\bf J}_k} h\left(X_{i}, X_{j}\right)$ & Yes & Thm. \ref{thm:univariate-moiu} \\
			\makecell{$\widetilde\theta_\mathrm{MoIU}$} & $\forall k\in[K], {\bf J}_k\sim_\mathrm{wor}\Lambda_n$ & $\widetilde U_{M, \J_k}\triangleq \frac{1}{M}\sum_{(i, j)\in{\bf J}_k{}} h\left(X_{i}, X_{j}\right)$ & Yes & Thm. \ref{thm:univariate-moiu-swor} \\
			\makecell{$\widecheck\theta_\mathrm{MoIU}$} & $\big\{(i_m^k, j_m^k)\big\}_{m,k=1}^{M,K}\sim_\mathrm{wor}\Lambda_n$ & $\widecheck U_{M,\J_k}\triangleq \frac{1}{M}\sum_{m=1}^Mh\left(X_{i_m^k}, X_{j_m^k}\right)$ & No & Thm. \ref{thm:univariate-moiu-swor-blockwise} \\
			\bottomrule
		\end{tabular}
	\end{table}
	\section{Main Results}
	For reader's convenience, we summarize our main results and the key result of \citet{pmlr-v97-clemencon19a} in Table \ref{tab:results}. We highlight that there are some direct equivalences with the prior work. Specifically:
	\begin{enumerate}[label=(\roman*)]
		\item Our $\widehat\theta_\mathrm{MoIU}$ is directly equivalent to $\mathrm{MoIU}_{\tau,\mathrm{MC}}$ in \citet[Table 4]{pmlr-v97-clemencon19a}.
		\item Our $\widetilde\theta_\mathrm{MoIU}$ is directly equivalent to $\mathrm{MoIU}_{\tau,\mathrm{SWoR}}$ in \citet[Table 4]{pmlr-v97-clemencon19a}.
	\end{enumerate}
	\noindent The estimator $\widecheck\theta_\mathrm{MoIU}$ covered in this work is not directly equivalent to any variant of $\mathrm{MoIU}$ estimator mentioned in \citet{pmlr-v97-clemencon19a}.
	\subsection{Univariate Case}
	\label{sec:univariate-moiu}
	As mentioned earlier, the difficulty of analyzing $\widehat\theta_\mathrm{MoIU}$ lies in the fact that McDiarmid's inequality cannot yield the sharp bound at the rate of $\exp\left(-\frac{nt^2}{M}\right)$ analogous to that in the analysis of $\bar\theta_\mathrm{MoRU}$. However, using a simple segmentation of $\Lambda_n^M$, we are able to achieve the following result.
	\begin{proposition}[cf. Proposition \ref{prop:ugly-vstats-bound-appendix}]
		\label{prop:ugly-vstats-bound}
		Let $\widehat W_n^\varepsilon$ be the V-Statistic defined in Eqn.~\eqref{eq:vstats-ugly} and $\widehat q_\varepsilon=\E[\widehat W_n^\varepsilon]$. Suppose that $n\ge 2M$. Then, for all $t>0$, we have:
		\begin{align}
			\boxed{\P(\widehat W_n^\varepsilon-\widehat q_\varepsilon\ge t)\le \exp\left(-\frac{nt^2}{2M}\right)}.
		\end{align}
	\end{proposition}
	
	\begin{proof} (Sketch)
		The focal point of our proof relies on decomposing the highly complex von-Mises statistic in Eqn.~\eqref{eq:vstats-ugly} into a convex combination of U-Statistics. First, let us note that the set of all possible length-$M$ pairs sequences can be written as follows:
		\begin{align}
			\Lambda_n^{M} \triangleq \bigcup_{k=2}^{2M}\left[\bigcup_{\mathcal{J}_k\in C_{n,k}}\Theta_{\mathcal{J}_k}\right],\quad\Theta_{\mathcal{J}_k}\triangleq\Big\{{\bf J}\in\Lambda_n^M : \eta({\bf J})=\mathcal{J}_k\Big\}.
		\end{align}
		\noindent Here, $\eta({\bf J})\subseteq[n]$ is the set of all unique indices in the sequence ${\bf J}\in\Lambda_n^M$. Hence, $\Theta_{\mathcal{J}_k}$ can be interpreted as the set of all length-$M$ sequences of pairs whose set of distinct indices is \underline{exactly $\mathcal{J}_k$}. As a result, we can write $\widehat W_n^\varepsilon$ as follows:
		\begin{align}
			\label{eq:convex-decomposition}
			\widehat W_n^\varepsilon &\triangleq\frac{1}{|\Lambda_n^M|}\sum_{k=2}^{2M}\sum_{\mathcal{J}_k\in C_{n,k}}\left[\sum_{{\bf J}\in\Theta_{\mathcal{J}_k}}\lambda_\varepsilon\left(\widehat U_{M,\J}(h)\right)\right] = \sum_{k=2}^{2M}\alpha_{n,k}U_{n,k}(\phi_k).
		\end{align}
		\noindent Specifically, $\alpha_{n,k}\triangleq\frac{G_k}{|\Lambda_n^M|}\binom{n}{k}$ where $G_k\triangleq|\Theta_{\mathcal{J}_k}|$ where $\mathcal{J}_k\subseteq[n]$ is an indices subset of size $k$. Trivially, $\sum_{k=2}^{2M}\alpha_{n,k}=1$, making $\widehat W_n^\varepsilon$ a convex combination of smaller data-functionals $U_{n,k}(\phi_k)$, which are order-$k$ U-Statistics of bounded kernels $\phi_k$. Using this structure, the derivation of $\widehat W_n^\varepsilon$'s concentration rate is significantly easier since we can now apply regular Hoeffding-type bounds to attain concentration rates of $\exp\left(-\frac{nt^2}{k}\right)$ (with $k\le 2M$) for each $U_{n,k}(\phi_k)$ separately. This gives rise to the concentration rate (up to constant) of $\exp\left(-\frac{nt^2}{M}\right)$ that significantly improves upon the naive application of McDiarmid's inequality (cf. \citet[Appendix C.2]{pmlr-v97-clemencon19a}). As a result, we can achieve the following concentration rate for $\widehat\theta_\mathrm{MoIU}$.
	\end{proof}
	
	\begin{theorem}
		\label{thm:univariate-moiu}
		Let $\delta\in(0,1)$ and $\tau\in(0,\frac{1}{2})$ be a free parameter. Suppose that $n\ge\frac{9\log(2/\delta)}{2\tau^2}$, $K\ge \frac{\log(2/\delta)}{2(\frac{1}{2}-\tau)^2}$ and $M=\lfloor \frac{2n\tau^2}{9\log(2/\delta)}\rfloor$. Then, we have
		\begin{align}
			|\widehat\theta_\mathrm{MoIU}(h)-\theta(h)|\le \sqrt{\frac{12\sigma_1^2(h)}{n\tau} + \frac{6\sigma_2^2(h)}{\tau n(n-1)} + \frac{27\sigma^2(h)\log(2/\delta)}{n\tau^3}},
		\end{align}
		\noindent with probability of at least $1-\delta$. Here, $\sigma_1^2(h)\triangleq\mathrm{Var}(\E[h(X_1,X_2)|X_1])$, $\sigma^2(h)\triangleq\mathrm{Var}(h(X_1,X_2))$ and $\sigma_2^2(h)\triangleq\sigma^2(h)-2\sigma_1^2(h)$.
	\end{theorem}
	
	\begin{remark}
		We note that when $\widehat\theta_\mathrm{MoIU}(h)$ is computed as the median of averages over pairs selected without replacement instead of with replacement, i.e., each indices set ${\bf J}_k,\forall k\in[K]$ is drawn without replacement from $\Lambda_n$ and ${\bf J}_1,\dots, {\bf J}_K$ are independent of each other, then the concentration rate above also applies \textbf{with the same constants}. In particular, we have an analogous result to Theorem \ref{thm:univariate-moiu} that holds for $\widetilde\theta_\mathrm{MoIU}(h)$ defined as follows:
		\begin{align}
			\label{eq:univariate-moiu-swor}
			\widetilde\theta_\mathrm{MoIU}(h) &\triangleq \mathrm{med}\left(\widetilde U_{M, \J_1}(h), \dots,\widetilde U_{M,\J_K}(h)\right), \\
			\forall k\in[K]: \widetilde U_{M,\J_k}(h) &\triangleq \frac{1}{M}\sum_{(i, j)\in{\bf J}_k{}} h\left(X_{i}, X_{j}\right),\qquad {\bf J}_k\sim_\mathrm{wor} \Lambda_n.
		\end{align}
		\noindent We used the notation $\sim_\mathrm{wor}$ to denote simple random selection \textit{without replacement} from the set of all pairs $\Lambda_n$. Then, we have the following result.
	\end{remark}
	
	\begin{theorem}
		\label{thm:univariate-moiu-swor}
		Let $\delta\in(0,1)$ and $\tau\in(0,\frac{1}{2})$ be a free parameter. Suppose that $n\ge\frac{9\log(2/\delta)}{2\tau^2}$, $K\ge \frac{\log(2/\delta)}{2(\frac{1}{2}-\tau)^2}$ and $M=\lfloor \frac{2n\tau^2}{9\log(2/\delta)}\rfloor$. Then, we have the following concentration rate for $\widetilde\theta_\mathrm{MoIU}(h)$
		\begin{align}
			|\widetilde\theta_\mathrm{MoIU}(h)-\theta(h)|\le \sqrt{\frac{12\sigma_1^2(h)}{n\tau} + \frac{6\sigma_2^2(h)}{\tau n(n-1)} + \frac{27\sigma^2(h)\log(2/\delta)}{n\tau^3}},
		\end{align}
		\noindent with probability of at least $1-\delta$. Here, $\sigma_1^2(h)\triangleq\mathrm{Var}(\E[h(X_1,X_2)|X_1])$, $\sigma^2(h)\triangleq\mathrm{Var}(h(X_1,X_2))$ and $\sigma_2^2(h)\triangleq\sigma^2(h)-2\sigma_1^2(h)$. For further details, please refer to our proof in Section \ref{sec:about-wor-moiu}.
	\end{theorem}
	
	\subsection{SWoR Across Blocks}
	One natural extension of the estimators outlined in Section \ref{sec:univariate-moiu} is when we select $KM$ pairs without replacement from $\Lambda_n$ directly then partition into $K$ blocks of size $M$ to compute the incomplete U-Statistics. This means that, unlike $\widetilde \theta_\mathrm{MoIU}$ where the lack of replacement is only reflected within block, this extension allows lack of replacement across blocks. Formally, we define this setting as an $\mathrm{MoIU}$ estimator formulated as follows:
	\begin{align}
		\label{eq:moiu-swor-blockwise}
		\widecheck\theta_\mathrm{MoIU}(h) &\triangleq \mathrm{med}\left(\widecheck U_{M,\J_1}(h), \dots,\widecheck U_{M,\J_K}(h)\right), \\
		\forall k\in[K]: \widecheck U_{M,\J_k}(h) &\triangleq \frac{1}{M}\sum_{m=1}^Mh\left(X_{i_m^k}, X_{j_m^k}\right),\quad{\bf J}=\big\{(i_m^k, j_m^k)\big\}_{m,k=1}^{M,K}\sim_\mathrm{wor}\Lambda_n.
	\end{align}
	\noindent In this case, if we define $Z_k^\varepsilon=\lambda_\varepsilon(\widecheck U_{M,\J_k}(h))$ for all $k\in[K]$, these random variables are no longer independent Bernoulli indicators \textit{even if they are conditioned on $S_n$}. Therefore, analyzing the concentration of $\frac{1}{K}\sum_{k=1}^KZ_k^\varepsilon$ requires a more delicate argument for selection without replacement from a finite population. Hence, we introduce the following Serfling-type \citep{Serfling1974-lz, Bardenet2013ConcentrationIF} bound for estimation of multi-argument kernels via selection without replacement. While we use Lemma \ref{lem:serfling-swor-blockwise} primarily as a tool to prove Theorem \ref{thm:univariate-moiu-swor-blockwise}, we believe it might also be of independent interest within the literature of concentration inequalities for SWoR.
	\begin{lemma}[cf. Lemma \ref{lem:serfling-swor-blockwise-appendix}]
		\label{lem:serfling-swor-blockwise}
		Let $\mathcal{P}_N=\{\z_1,\dots,\z_N\}\subseteq\mathcal{Z}$ be a finite population. Let $K, M$ be integers such that $KM^2\le N$ and $h:\mathcal{Z}^M\to[0,1]$ is a \textit{symmetric} kernel in its $M$ arguments. Let $\{Z_{k, m}\}_{k,m=1}^{K,M}\sim_\mathrm{wor}\mathcal{P}_N$ and define $W_K, U_M$ as follows:
		\begin{align}
			W_K\triangleq\frac{1}{K}\sum_{k=1}^Kh\left(Z_{k,1},\dots, Z_{k,M}\right), \quad U_M\triangleq\frac{1}{\binom{N}{M}}\sum_{j_1,\dots,j_M\in C_{N,M}}h\left(\z_{j_1},\dots,\z_{j_M}\right).
		\end{align}
		\noindent Then, for any $t>0$, we have:
		\begin{align}
			\boxed{\P(W_K-U_M\ge t) \le \exp\left(-\frac{Kt^2}{8}\right)}.
		\end{align}
	\end{lemma}
	\noindent Using the above result to analyze the concentration of (dependent) indicators $Z_1^\varepsilon,\dots,Z_K^\varepsilon$ and a similar argument to that of Proposition \ref{prop:ugly-vstats-bound} for the conditional expectation, we have the following concentration result for $\widecheck \theta_\mathrm{MoIU}$.
	\begin{theorem}
		\label{thm:univariate-moiu-swor-blockwise}
		Let $\delta\in(0,1)$, $\tau\in(0,\frac{1}{2})$ be a parameter satisfying $g(\tau,\delta)\triangleq\frac{4\tau^4}{9(\frac{1}{2}-\tau)^2\log(2/\delta)}\le\frac{1}{4}$. Suppose that $n\ge\frac{9\log(2/\delta)}{2\tau^2}$, $K=\lceil\frac{8}{(\frac{1}{2}-\tau)^2}\log(2/\delta)\rceil$ and $M=\lfloor\frac{2n\tau^2}{9\log(2/\delta)}\rfloor$. Then, we have the following concentration rate for $\widecheck\theta_\mathrm{MoIU}(h)$:
		\begin{align}
			|\widecheck\theta_\mathrm{MoIU}(h)-\theta(h)|\le \sqrt{\frac{12\sigma_1^2(h)}{n\tau} + \frac{6\sigma_2^2(h)}{\tau n(n-1)} + \frac{27\sigma^2(h)\log(2/\delta)}{n\tau^3}},
		\end{align}
		\noindent with probability of at least $1-\delta$. Here, $\sigma_1^2(h)\triangleq\mathrm{Var}(\E[h(X_1,X_2)|X_1])$, $\sigma^2(h)\triangleq\mathrm{Var}(h(X_1,X_2))$ and $\sigma_2^2(h)\triangleq\sigma^2(h)-2\sigma_1^2(h)$.
	\end{theorem}
	
	\subsection{Multivariate Extension}
	Another natural extension of median-of-\textit{incomplete}-U-Statistics estimators is considering geometric median for multivariate kernels. Specifically, let $h:\mathcal{Z}\times\mathcal{Z}\to\R^d$ be a vector-valued kernel where $\E[h(X_1,X_2)]=\Theta(h)\in\R^d$, we wish to formulate an estimator for $\Theta(h)$ analogous to $\widehat\theta_\mathrm{MoIU}$ in the univariate case. In particular, for each $k\in[K]$, let ${\bf J}_k=\big\{ (j_{k,1}, i_{k, 1}),\dots, (j_{k,M}, i_{k,M}) \big\}$ be indices drawn with replacement from $\Lambda_n$, we define $\widehat\Theta_\mathrm{MoIU}(h)$ as follows:
	\begin{align}
		\widehat\Theta_\mathrm{MoIU}(h) &\triangleq \arg\min_{\z\in\R^d}\sum_{k=1}^K\left\|\z-\widehat U_{M,\J_k}(h)\right\|_2, \\
		\forall k\in[K]: \widehat U_{M,\J_k}(h) &\triangleq \frac{1}{M}\sum_{(i,j)\in{\bf J}_k} h\left(X_{i}, X_{j}\right),\quad{\bf J}_k\sim_\star \Lambda_n.
	\end{align}
	\noindent Here, $\sim_\star$ denotes either selection with or without replacement from the pairs collection $\Lambda_n$. We have demonstrated, in the univariate setting, that the within-block selection scheme does not affect the concentration rate as long as the blocks themselves are independent of each other. Therefore,  we will use a common notation for both within-block sampling schemes.
	\begin{theorem}
		\label{thm:multivariate-moiu}
		Let $\delta\in(0,1)$ and $\tau<\alpha\in(0,\frac{1}{2})$ be a free parameters. Suppose that $n\ge\frac{9\log(2/\delta)}{2\tau^2}$, $K\ge \frac{\log(2/\delta)}{2(\alpha-\tau)^2}$, $M=\lfloor \frac{2n\tau^2}{9\log(2/\delta)}\rfloor$ and $C_\alpha\triangleq\frac{1-\alpha}{\sqrt{1-2\alpha}}$. Then, we have
		\begin{align}
			\left\|\widehat\Theta_\mathrm{MoIU}(h)-\Theta(h)\right\|_2\le C_\alpha\sqrt{\frac{12\varsigma_1^2(h)}{n\tau} + \frac{6\varsigma_2^2(h)}{\tau n(n-1)} + \frac{27\varsigma^2(h)\log(2/\delta)}{n\tau^3}},
		\end{align}
		\noindent wp. of at least $1-\delta$. Here, $\varsigma_1^2(h)\triangleq\mathrm{tr}(\mathrm{Cov}(\E[h(X_1,X_2)|X_1]))$, $\varsigma^2(h)\triangleq\mathrm{tr}(\mathrm{Cov}(h(X_1,X_2)))$ and $\varsigma_2^2(h)\triangleq\varsigma^2(h)-2\varsigma_1^2(h)$.
	\end{theorem}
	\noindent Of course, the above result is also extendable to the case where all $KM$ pairs are selected without replacement from $\Lambda_n$ using the same Serfling-type argument that we presented in Lemma \ref{lem:serfling-swor-blockwise}. However, we will not outline the result as it is mostly a procedural extension.
	
	\begin{table}[ht]
		\centering
		\caption{Dependency of $K,M$ and $n$ on $\tau,\delta$.}
		\label{tab:bernstein-hoeffding-comparison}
		\begin{tabular}{c c c c}
			\toprule
			\textbf{Results} & $K$ & $M$ & $n$ \\
			\midrule
			Theorem \ref{thm:univariate-moiu} & 
			\multirow{2}{*}{$\Omega\!\left(\left(\frac{1}{2}-\tau\right)^{-2}\log(1/\delta)\right)$}
			& $O(\tau^2\log^{-1}(1/\delta))$ & $\Omega(\tau^{-2}\log(1/\delta))$ \\
			
			Theorem \ref{thm:univariate-moiu-bernstein} & 
			& $O(\tau\log^{-1}(1/\delta))$ & $\Omega(\tau^{-1}\log(1/\delta))$ \\
			\bottomrule
		\end{tabular}
	\end{table}
	\section{Sharper Constants with Bernstein}
	In the above results, the free parameter $\tau\in(0,\frac{1}{2})$ is of particular interest as it directly governs the trade-off between efficiency and accuracy. From a computational standpoint, it is encouraged to have smaller values of $\tau$ since this allows $K, M$ to be smaller and reduce calculation load. However, this in turns loosen the accuracy due to the presence of $\tau$ in the denominator of the upper bounds. Currently, results from Theorem \ref{thm:univariate-moiu} to \ref{thm:multivariate-moiu} all depend on $\tau$ in $O(\tau^{-3/2})$ order, making it scale unfavourably for small values of $K$ and $M$. This is a minor artefact when we rely on a Hoeffding-type treatment for the concentration of $\widehat W_n^\varepsilon$. Specifically, in the general trajectory of current proofs, we rely on the boundedness of $\phi_k$ (cf. Eqn.~\eqref{eq:convex-decomposition}) to invoke Hoeffding-type MGF bounds for each $U_{n,k}(\phi_k)$, leading to the final concentration rate of $\exp(-n\tau^2/M)$ for $\widehat W_n^\varepsilon$ and making $M$ scales with $\tau$ quadratically. In Appendix \ref{sec:bernstein-sharpening}, we demonstrate that if we can leverage the small variances of $\phi_k$ and apply Bernstein-type MGF bounds for each $U_{n,k}(\phi_k)$, we can improve the concentration rate of $\widehat W_n^\varepsilon$ to $\exp(-n\tau/M)$, making $M$ scale linearly with $\tau$ and improving the current result's scale from $O(\tau^{-3/2})$ to $O(\tau^{-1})$. Concretely, we have the following Bernstein bound for $\widehat W_n^\varepsilon$.
	
	\begin{proposition}[Bernstein Analogue of Proposition \ref{prop:ugly-vstats-bound}]
		\label{prop:ugly-vstats-bernstein-bound}
		Let $\widehat W_n^\varepsilon$ be the V-Statistic defined in Eqn.~\eqref{eq:vstats-ugly} and $\widehat q_\varepsilon=\E[\widehat W_n^\varepsilon]$. Suppose that $n\ge 2M$. Then, for all $t>0$, we have:
		\begin{align}
			\boxed{\P(\widehat W_n^\varepsilon-\widehat q_\varepsilon\ge t)\le \exp\left(-\frac{nt^2}{8M\left(\widehat\sigma_{M,\varepsilon}^2 + \frac{t}{3}\right)}\right)},\quad\text{where }\widehat\sigma_{M,\varepsilon}^2=\frac{2\sigma^2(h)}{M\varepsilon^2}.
		\end{align}
	\end{proposition}
	\noindent Recall that in the scope of this work's problem, we care about the regime where $t\lesssim\tau$. Hence, if we have $\widehat\sigma_{M,\varepsilon}^2\lesssim\tau$ or $\frac{\sigma(h)^2}{M\tau}\lesssim \varepsilon^2$, the concentration of $\widehat W_n^\varepsilon$ turns into a sub-Exponential rate of $\exp(-n\tau/M)$. As a result, we have the following Bernstein analogue of Theorem \ref{thm:univariate-moiu}.
	\begin{theorem}[Bernstein Analogue of Theorem \ref{thm:univariate-moiu}]
		\label{thm:univariate-moiu-bernstein}
		Let $\delta\in(0,1)$ and $\tau\in(0,\frac{1}{2})$ be a free parameter. Suppose that $n\ge\frac{56\log(2/\delta)}{9\tau}$, $K\ge \frac{\log(2/\delta)}{2(\frac{1}{2}-\tau)^2}$ and $M=\lfloor \frac{9n\tau}{56\log(2/\delta)}\rfloor$, we have:
		\begin{align*}
			|\widehat\theta_\mathrm{MoIU}(h)-\theta(h)|\le \sqrt{\frac{28\sigma_1^2(h)}{\tau n}+\frac{14\sigma_2^2(h)}{\tau n(n-1)}+\frac{88\sigma^2(h)\log(2/\delta)}{\tau^2n}},
		\end{align*}
		\noindent with probability of at least $1-\delta$. Here, $\sigma_1^2(h)\triangleq\mathrm{Var}(\E[h(X_1,X_2)|X_1])$, $\sigma^2(h)\triangleq\mathrm{Var}(h(X_1,X_2))$ and $\sigma_2^2(h)\triangleq\sigma^2(h)-2\sigma_1^2(h)$.
	\end{theorem}
	\noindent In the above result, $M$ (block size) and $n$ (minimum sample size) both exhibit better dependencies on $\tau$ than the Hoeffding counterpart in Theorem \ref{thm:univariate-moiu} (cf. Table \ref{tab:bernstein-hoeffding-comparison}). Of course, we can also use the Bernstein bound in Proposition \ref{prop:ugly-vstats-bernstein-bound} to extend the above result to $\widetilde\theta_\mathrm{MoIU}$ and $\widecheck\theta_\mathrm{MoIU}$. However, the extension is mostly procedural and we omit the details from this manuscript.
	
	\section{Impossibility Results}
	One observe that increasing the block size, for instance, taking $M=O(n^2)$, makes each incomplete U-Statistic $\widehat U_{M,\J_k}$ concentrate around the desired parameter $\theta(h)$ at a faster rate (e.g., quadratically decaying in $n$). However, as Proposition \ref{prop:ugly-vstats-bound} and \ref{prop:ugly-vstats-bernstein-bound} suggest, if $M$ grows faster than an order of $O(n)$, then $\widehat W_n^\varepsilon$ loses its exponential concentration. Hence, this raises a natural question: is the current proof trajectory, which requires $\widehat W_n^\varepsilon$ to concentrate, sub-optimal or is the condition $M\le O(n)$ an intrinsic limitation of the problem. In the following anti-concentration results, we demonstrate that the latter is true. We first present a distribution-independent lower bound where too large block size $M$ will prohibit $\mathrm{MoIU}$ estimators from achieving the usual sub-Gaussian concentration rate. Then, we demonstrate that there exists a pathological distribution where sub-Gaussian rate fails at a much sharper threshold for block size.
	\subsection{Distribution-independent Bound}
	\begin{proposition}[General Impossibility Result]
		\label{prop:impossibility_result_general_main}
		Let $\widehat\theta_\star$ denote either $\widehat\theta_\mathrm{MoIU},\widetilde\theta_\mathrm{MoIU}$ or $\widecheck\theta_\mathrm{MoIU}$. Let $Z\sim\mathcal{N}(0,1)$ and for all $\gamma>0$, we define $\Phi_\gamma\triangleq\P(|Z|\ge\gamma)$. Then, we have:
		\begin{align}
			M\ge\frac{4n\sigma^2(h)}{\Phi_\gamma\gamma^2\sigma_1^2(h)}\implies\P\left(\left|\widehat\theta_\star(h)-\theta(h)\right|\ge\frac{\gamma\sigma_1(h)}{\sqrt{n}}\right) \ge f_n(\gamma),
		\end{align}
		\noindent where $f_n$ is a sequence of functions with $\lim_{n\to\infty}f_n(\gamma) = \frac{\Phi_\gamma}{2}>0$.
	\end{proposition}
	
	\begin{proof} (Sketch)
		We will focus on the proof outline of $\widehat\theta_\mathrm{MoIU}$ since the result is easily extendable to other estimators. Let $\varepsilon>0$ and $U_n$ be the complete U-Statistic over all pairs in $\Lambda_n$. Define the two events $E,\bar G$ as follows:
		\begin{align}
			E=\Big\{\left|U_n-\theta(h)\right|\ge\varepsilon\Big\}, \quad \bar G = \Bigg\{B_{K,\varepsilon}<\frac{K}{2}\Bigg\}.
		\end{align}
		\noindent Here, $B_{K,\varepsilon}=\sum_{k=1}^K\mathds{1}_{|\widehat U_{M,\J_k}(h)-U_n|\ge\varepsilon/2}$. We observe that $E\cap\bar G$ implies $|\widehat\theta_\mathrm{MoIU}(h)-\theta(h)|\ge\varepsilon/2$. Essentially, when the complete U-Statistic $U_n$ deviates substantially from $\theta(h)$ and more than half of the incomplete U-Statistics are close to $U_n$, then the median estimator deviates substantially from $\theta(h)$. Therefore, we can lower bound the desired deviation probability via $\P(E\cap\bar G)$. By Bonferroni, we have $\P(E\cap\bar G^c)\ge\P(E)-\P(\bar G^c)$.
		
		\noindent Let $\varepsilon=\frac{2\gamma\sigma_1(h)}{\sqrt{n}}$ and $Z_n=n^\frac{1}{2}\left(\frac{U_n-\theta(h)}{2\sigma_1(h)}\right)$. By CLT of U-Statistics (cf. Theorem \ref{thm:clt_ustats}) we have:
		\begin{align*}
			\P(E) &= \P(|Z_n|\ge\gamma) \xrightarrow{n\to\infty} \P(|Z|\ge\gamma)\triangleq\Phi_\gamma.
		\end{align*}
		\noindent On the other hand, for sufficiently large $M$ (i.e., $M\ge\frac{4n\sigma^2(h)}{\Phi_\gamma\gamma^2\sigma_1^2(h)}$), we can show that $\P(\bar G^c)\le\frac{\Phi_\gamma}{2}$. Putting it all together, we have:
		\begin{align*}
			\P\left(\left|\widehat\theta_\mathrm{MoIU}(h)-\theta(h)\right|\ge\frac{\gamma\sigma_1(h)}{\sqrt{n}}\right) &\ge \P(E) - \P(\bar G^c) = \P(|Z_n|\ge\gamma)-\frac{\Phi_\gamma}{2}\xrightarrow{n\to\infty}\frac{\Phi_\gamma}{2},
		\end{align*}
		\noindent as desired.
	\end{proof}
	
	\noindent In essence, although increasing the block size makes each incomplete U-statistic $\widehat U_{M,\J_k}(h)$ a more precise estimator of $\theta(h)$, it simultaneously forces these sub-estimators to concentrate too tightly around the complete U-statistic (since $\E[\widehat U_{M,\J_k}(h)|S_n]=U_n,\forall k\in[K]$). Consequently, the tail probability of the median estimator is governed by that of $U_n$, which is heavy-tailed by assumption.
	
	\subsection{Catoni-type Lower Bound}
	While Proposition \ref{prop:impossibility_result_general_main} is independent of both the kernel $h$ and the data distribution, the threshold on block size $M$ is extremely vacuous even for moderate values of $\gamma$. In the next result, we present an anti-concentration lower bound in the style of Proposition \ref{prop:impossibility_result_general_main} with a much sharper threshold on $M$ that does not rely on $\gamma$.
	\begin{proposition}
		\label{prop:catoni-type-lowerbound-main}
		Let $\widehat\theta_\star$ denote either $\widehat\theta_\mathrm{MoIU},\widetilde\theta_\mathrm{MoIU}$ or $\widecheck\theta_\mathrm{MoIU}$. Define kernel $h(x,y)=\frac{1}{2}(x+y)$. Then, for all $\gamma\ge \frac{1}{2}$, there exists a distribution $\mathcal{P}_{\sigma,\gamma}$ (from which $X_1,\dots,X_n\sim_\mathrm{i.i.d.}\mathcal{P}^{\otimes n}_{\sigma,\gamma}$) with $\E_{\mathcal{P}_{\sigma,\gamma}}[X]=0$ and $\E_{\mathcal{P}_{\sigma,\gamma}}[X^2]=\sigma^2$ such that:
		\begin{align}
			M\ge 8en\implies \P\left(\big|\widehat\theta_\star(h)\big|\ge\frac{\gamma\sigma}{\sqrt{n}}\right) \ge\frac{1}{8e\gamma^2}.
		\end{align}
	\end{proposition}
	\noindent Note that the above result is tight and matches the Chebyshev's inequality rate up to constant. The pathological distribution $\mathcal{P}_{\sigma,\gamma}$ is constructed as a probability mass over an alphabet $\{-n\xi,0,n\xi\}$ where $\xi$ is dependent on the deviation level $\gamma$. The proof follows the same trajectory that relies on the fast concentration of each incomplete U-Statistic to $U_n$, which equals the sample mean since the kernel $h(x,y)$ is simply the average of its arguments. As a result, the concentration rate of $\widehat\theta_\mathrm{MoIU}$ is governed by the sample mean $\widehat\mu_n\triangleq\frac{1}{n}\sum_{j=1}^nX_j$, which is proven to scale sharply as $1/\gamma^2$ in \citet{article:catoni2012}. For the specific construction of $\mathcal{P}_{\sigma,\gamma}$, we refer readers to Appendix \ref{sec:catoni-lowerbound}.
	
	\section{Related Works}
	\textbf{U-Statistics and Incomplete U-Statistics}: U-Statistics were introduced in \citet{article:hoeffding1948} as a class of unbiased statistics for estimating expectations of multi-argument symmetric kernels. In later works by \citet{article:ginezinn1984} and \citet{article:arconesgine1993}, concentration and asymptotic properties of U-Statistics were formally established, setting the foundation for subsequent works in multiple disciplines of statistical learning theory. The applications of U-Statistics typically span across learning problems that involve interactions between data such as ranking or bipartite ranking \citep{Presnellranking, clemencon2005ranking, article:clemencon2008, clemencon2021ranking, clemencon2025ranking}, clustering \citep{article:clemencon2011clustering, article:clemencon2014clustering, liliuclustering, Bello02012024}, similarity or pairwise learning \citep{leiregularizedpairwise,leisharperpairwise} and contrastive learning \citep{article:hieu2025, article:hieu2026}. In \citet{blom1976}, the definition of incomplete U-Statistics was first formalized and studied as a computationally efficient alternative to the complete counterpart in \citet{article:hoeffding1948}. For an order-$k$ kernel $h:\mathcal{Z}^k\to\R$, an incomplete U-Statistics randomly selects (via selection with/without replacement, Bernoulli sampling, etc) only a small subset of size-$k$ tuples instead of evaluating over the full collection of $\binom{n}{k}$ tuples. This practical advantage later motivates a wide variety of theoretical results including  Gaussian approximation bounds \citep{hui2019randomizedhighdim,song2019berryesseen, dennis2026berryesseen}, finite-sample concentration rates \citep{ maurer2022exponentialfinitesamplebounds,ke2026concentrationinequalitiesincompleteustatistics} and SGD-based generalization bounds \citep{clemencon2015sgd,clemencon2016scalinguperm}. We note that most of the aforementioned works focus on the regime with light-tailed (sub-Gaussian) or almost-surely bounded kernels. 
	
	\textbf{Robust Estimation}: Robust estimation begins with the M-estimation framework \citep{huber1964}, which was later consolidated in \citet{bookhuber1981} and the influence-function approach of \citet{Hampel2005}. Subsequently, the trimmed mean was formalized by \citet{turkey1962}, with later asymptotic results by \citet{stiglertrimmedmeans1973} and optimality results by \citet{lugosi2021trimmedmeans}. Later, a closely related concept, Winsorized mean, was introduced in \citet{dixon1968winsor}. In the generalized-linear-model setting, the Cantoni-type robust quasi-likelihood estimator was introduced by \citet{catoni2001glm} and extended in \citet{catoni2004}. For heavy-tailed data, the median-of-means principle was originated from \citet{nemirovsky1983problem} and was sharpened into sub-Gaussian estimators by \citet{devroye2016mom} and \citet{lugosi2019geometricmedian}. 
	
	\noindent Some of the above technique has been adapted for estimation of U-Statistics in settings with heavy-tailed kernels. For instance, in \citet{stanislav2020matrix-valued-u}, a Huberization approach was utilized for robust estimation of matrix-valued U-Statistics. In \citet{pmlr-v97-clemencon19a}, the most related work to ours, adapted median-of-means technique with block randomization to U-Statistics with univariate pairwise kernels. In a work closely related to ours by \citet{joly2015robust-ustats}, a median-of-mean type estimator was proposed for robust estimation of U-Statistics. However, in their work, the formulated estimator was structurally different from both our work and \citet{pmlr-v97-clemencon19a}. Specifically, given a list of $V$ partitions (non-overlapping blocks) $\mathcal{B}_1,\dots,\mathcal{B}_V\subseteq[n]$ and a multi-argument kernel $h:\mathcal{Z}^m\to\R$, their work define the sub-estimator restricted to a list of $m$ blocks $\mathcal{B}_{i_1},\dots,\mathcal{B}_{i_m}$ (where $i_1,\dots,i_m\in[V]$) as follows:
	\begin{align}
		U_{\mathcal{B}_{i_1},\dots,\mathcal{B}_{i_m}}(h)\triangleq\frac{1}{\prod_{k=1}^m|\mathcal{B}_{i_k}|}\sum_{j_1,\dots,j_m\in \mathcal{B}_{i_1}\times\dots\times\mathcal{B}_{i_m}} h\left(X_{j_1},\dots,X_{j_m}\right).
	\end{align}
	\noindent Then, their final estimator is defined as the median of all possible $\binom{V}{m}$ sub-estimators induced by selecting $m$ blocks without replacement from the parition $\mathcal{B}_1,\dots,\mathcal{B}_V$. In particular
	\begin{align}
		\bar U_\mathcal{B}(h)\triangleq\mathrm{med}\left(\left\{U_{\mathcal{B}_{i_1},\dots,\mathcal{B}_{i_m}}(h): i_1,\dots,i_m\in C_{V,m}\right\}\right).
	\end{align}
	\noindent While their work stays close to the complete U-Statistics regime and considerably focused on degenerate U-Statistics, which is different from the settings considered in our work, the above formulation is particularly interesting and worth mentioning.
	
	\section{Conclusion}
	In this work, we derived an $O(n^{-1/2})$ finite-sample concentration bound (cf. Theorem \ref{thm:univariate-moiu}) for the median-of-\textit{incomplete}-U-Statistics estimator formulated in \citet{pmlr-v97-clemencon19a}. Using a convex combination formulation of the data-functional $\widehat W_n^\varepsilon$, we are able to simplify the analysis and achieve $\P(\widehat W_n^\varepsilon-\widehat q_\varepsilon\ge t)\lesssim \exp(-\frac{nt^2}{M})$ (cf. Proposition \ref{prop:ugly-vstats-bound}), improving upon the bound obtained via a straightforward application of McDiarmid's inequality. Furthermore, we demonstrated that our arguments can be readily extended to the case where each block's incomplete U-Statistic is independently computed using pairs sampled \textit{without replacement} (cf. Theorem \ref{thm:univariate-moiu-swor}). Moreover, we prove that the same $O(n^{-1/2})$ concentration rate can be achieved in the sampling scheme where all $K\times M$ pairs are selected \textit{without replacement} at once from $\Lambda_n$ (cf. Theorem \ref{thm:univariate-moiu-swor-blockwise}). Then, these selected pairs are subsequently partitioned into $K$ blocks of size $M$ to compute incomplete U-Statistics as sub-estimators. Unlike the previous results, this sampling scheme suffers from a nuanced block-wise dependence that prohibits Hoeffding-type bounds, making it necessary to employ a more delicate Serfling-type argument (cf. Lemma \ref{lem:serfling-swor-blockwise}). Furthermore, we extend the concentration of $\widehat W_n^\varepsilon$ to a Bernstein-type bound in Proposition \ref{prop:ugly-vstats-bernstein-bound} by leveraging small variances of $\phi_k$. Using this result, we are able to improve the dependency on $\tau$ from an order of $O(\tau^{-3/2})$ to $O(\tau^{-1})$, making the trade-off between efficiency and accuracy sharper. Finally, we proved an anti-concentration inequality demonstrating that $M\le O(n)$ is an intrinsic condition for all median of incomplete U-Statistics estimators considered in this work. Essentially, when $M\ge \Omega(n)$, the tail of MoIU (in all sampling regimes) is governed by the heavy-tailed complete U-Statistics (cf. Proposition \ref{prop:impossibility_result_general_main}).
	
	\noindent We highlight that there is a structural asymmetry between the two parameters that govern our estimators, i.e., the number of blocks $K$ and the block size $M$. The block size $M$ cannot be taken arbitrarily large: Since each sub-estimator averages over $M$ pairs drawn from the same pool of data points, increasing $M$ will make the reuse of data more frequent, degrading the concentration of the data-functional $\widehat W_n^\varepsilon$, as reflected in the rate $\exp(-\frac{nt^2}{M})$ (cf. Proposition \ref{prop:ugly-vstats-bound}). On the other hand, the number of blocks $K$ enters only through the amplification term $\exp(-K(\frac{1}{2}-\tau)^2)$ and can be taken as large as desired at no cost to the rate (with the exception of $\widecheck\theta_\mathrm{MoIU}$): Adding blocks sharpens the median's confidence but does not tighten the underlying per-block's error.
	
	\noindent While we have explored different configurations of selection with or without replacement for block sampling in this work, it would be interesting to investigate other sampling schemes. For instance, in \citet{clemencon2016scalinguperm}, Bernoulli sampling is mentioned as a possible variant of incomplete U-Statistics where each pair in $\Lambda_n$ has a fixed probability of being chosen for computing the final statistics. In other words, the number of chosen pairs is random due to the configured selection probability. Hence, as a future research direction, it would be interesting to see how the concentration rate of median-of-incomplete-U-Statistics is influenced by this extra layer of randomness. 
	
	\newpage
	\bibliography{main}
	
	\newpage
	\appendix
	\title{\huge\textbf{\underline{Supplementary Material}}\\[0.5em]
		\huge\textbf{On Finite-sample Concentration of Median of Incomplete U-Statistics}}
	\etocdepthtag.toc{appendix}     
	
	\begingroup
	\etocsettagdepth{mainbody}{none}        
	\etocsettagdepth{appendix}{subsection}  
	\section*{Contents}                     
	\etocsettocstyle{}{}                    
	\tableofcontents
	\endgroup
	
	\newpage
	\section{Table of Notations}
	\begin{table}[ht!]
		\centering
		\caption{Summary of Common Notations}
		\label{tab:notation}
		\begin{tabular}{ll}
			\toprule
			\textbf{Ntn.} & \textbf{Description} \\
			\midrule
			$\mathcal{Z}$ & The data space \\
			$\mathcal{P}$ & The data distribution over $\mathcal{Z}$ \\
			$[n]$ & The indices set $\{1,2,\dots,n\}$ \\
			$C_{n, k}$ & The set of $k$-combinations chosen from $[n]$ \\
			$P_{n, k}$ & The set of $k$-permutations chosen from $[n]$ \\
			$\Pi_n$ & The set of bijections $\pi:[n]\to[n]$ (i.e., permutations of $[n]$) \\
			$\Lambda_n$ & The set of all pairs of distinct indices in $[n]$, i.e., $C_{n,2}$ \\
			$\eta(\cdot)$ & For any ${\bf J}\in\Lambda_n^M$, $\eta({\bf J})\subseteq[n]$ is the set of unique indices in $\bf J$ \\
			$\binom{\mathcal{R}}{k}$ & The set of $k$-combinations chosen from a set $\mathcal{R}$ \\
			$\sim_\mathrm{wr}$ & Random selection with replacement from some set \\
			$\sim_\mathrm{wor}$ & Random selection without replacement from some set \\
			$\sim_\star$ & Random selection either with or without replacement from some set \\
			\midrule
			$\mathrm{med}(\cdot)$ & The median operator \\
			$\theta(h)$ & $\theta(h)\triangleq\E[h(X_1,X_2)]$ where $h:\mathcal{Z}\times\mathcal{Z}\to\R$ is the pairwise kernel of concern \\
			$\lambda_\varepsilon(\cdot)$ & $\lambda_\varepsilon(z)\triangleq\mathds{1}_{|z-\theta(h)|\ge\varepsilon}$, i.e., the indicator of $\varepsilon$-deviation from $\theta(h)$ \\
			$\sigma^2(h)$ & $\sigma^2(h)\triangleq\mathrm{Var}(h(X_1,X_2))$, i.e., variance of the full kernel \\
			$\sigma_1^2(h)$ & $\sigma_1^2(h)\triangleq\mathrm{Var}(\E[h(X_1,X_2)|X_1])$, i.e., variance of the first Hoeffding-projection \\
			$\sigma_2^2(h)$ & $\sigma_2^2(h)\triangleq\sigma^2(h)-2\sigma_1^2(h)$, i.e., the variance of the first Hoeffding-projection's residual \\
			\bottomrule
		\end{tabular}
	\end{table}
	
	\section{Preliminaries}
	\subsection{U-Statistics Fundamentals}
	In this sub-section, we provide a brief overview on fundamental properties of U-Statistics. Specifically, we focus on describing the sub-Gaussianity of arbitrary order U-Statistics with bounded kernels, which will be necessary for proving our main results. Experienced readers are welcome to skip immediately to Section \ref{sec:proof-univariate} for our main proofs.
	\begin{definition}
		\label{eq:ustats_definition}
		Given a distribution $\mathcal{P}$ over a measurable space $\mathcal{Z}$ and $S=\{{X_j\}}_{j=1}^n\sim_\mathrm{i.i.d.}\mathcal{P}^{\otimes n}$. Let $h:\mathcal{Z}^k\to\R$ be a symmetric kernel, i.e., for any $\pi\in\Pi_k$, $h(z_1,\dots,z_k)=h\left(z_{\pi(1)},\dots,z_{\pi(k)}\right)$. The U-Statistic of order $k$ with kernel $h$, defined as:
		\begin{align}
			\label{eq:ustats}
			U_n^k(h) \triangleq \frac{1}{\binom{n}{k}}\sum_{j_1,\dots, j_k\in C_{n,k}}h\left(X_{j_1},\dots,X_{j_k}\right),
		\end{align}
		\noindent is an unbiased estimator for the parameter $\theta\triangleq\E_{\mathcal{P}^{\otimes k}}\left[h(X_1,\dots, X_k)\right]$.
	\end{definition}
	\begin{remark}
		In all of the results that follow, we will always assume that the considered kernel is symmetric and non-degenerate, i.e., $\sigma_1^2(h)\triangleq\mathrm{Var}(\E[h(X_1,\dots,X_n)|X_1])>0$.
	\end{remark}
	\begin{remark}[Decoupling Formula, cf.~\citet{book:pena1998}]
		An important property of the one-sample U-Statistic defined in Eqn.~\eqref{eq:ustats} is that it can be ``decoupled" into an average of i.i.d. tuples means over the set of permutations $\Pi_n$. Specifically, for $n_k\triangleq\lfloor n/k\rfloor$:
		\begin{align}
			\label{eq:decoupled_formula}
			\boxed{U_n^k(h) = \frac{1}{n!}\sum_{\pi\in\Pi_n} \left[\frac{1}{n_k}\sum_{j=1}^{n_k}h\left(X_{\pi(jk-k+1)},\dots,X_{\pi(jk)}\right)\right]}.
		\end{align}
		\noindent As we will see in later sections, the above is very useful for concentration properties of $U_n^k$.
	\end{remark}
	
	\begin{definition}[Sub-Gaussianity]
		A centered random variable $X$ is called \textbf{sub-Gaussian with variance proxy $v^2$} if its moment-generating function is bounded as:
		\begin{align}
			M_X(\lambda) \triangleq \E\left[e^{\lambda X}\right] \le \exp\left(\frac{\lambda^2v^2}{2}\right).
		\end{align}
		\noindent We then denote that $X\in\mathcal{SG}(v^2)$.
	\end{definition}
	
	\begin{lemma}[Decoupling Sub-Gaussianity of U-Statistics]
		\label{lem:gine-decouppling}
		Let $\varphi:\R\to\R$ be a \textbf{convex} function. Then, for any $\lambda\in\R$, we have:
		\begin{align}
			\E\left[\varphi\left(\lambda U_n^k(h)\right)\right] &\le \E\left[\varphi\left(\frac{\lambda}{n_k}\sum_{j=1}^{n_k}h\left(Z_{j_1},\dots,Z_{j_k}\right)\right)\right],
		\end{align}
		\noindent where we denote $n_k\triangleq\lfloor n/k\rfloor$ and the expectation on the right-hand-side is taken over $n_k\times k$ samples $\{Z_{j_\ell}\}_{j,\ell=1}^{n_k,k}$ drawn i.i.d. from $\mathcal{P}$.
	\end{lemma}
	
	\begin{proof}
		For all $j\in[n_k]$, we denote $j_1,\dots,j_k$ as $j_\ell\triangleq (j-1)k+\ell$ for all $\ell\in[k]$. Using the decoupled U-Statistics formula, for all $\lambda\in\R$, we have:
		\begin{align*}
			\E\left[\varphi\left(\lambda U_n^k(h)\right)\right] &= \E\left[
			\varphi\left(
			\frac{\lambda}{n!}\sum_{\pi\in\Pi_n} \left[\frac{1}{n_k}\sum_{j=1}^{n_k}h\left(X_{\pi(jk-k+1)},\dots,X_{\pi(jk)}\right)\right]
			\right)
			\right] \\
			&\le \frac{1}{n!}\sum_{\pi\in\Pi_n}\E\left[\varphi\left(\frac{\lambda}{n_k}\sum_{j=1}^{n_k}h\left(X_{\pi(jk-k+1)},\dots,X_{\pi(jk)}\right)\right)\right] \qquad\text{(Jensen's Inequality)} \\
			&=\E\left[\varphi\left(\frac{\lambda}{n_k}\sum_{j=1}^{n_k}h\left(Z_{j_1},\dots,Z_{j_k}\right)\right)\right],
		\end{align*}
		\noindent as desired.
	\end{proof}
	
	\begin{remark}
		In other words, we can derive a Hoeffding-type concentration bound for $U_n^k$ by relying on the sub-Gaussianity of the i.i.d. average $\frac{1}{n_k}\sum_{j=1}^{n_k}h\left(Z_{j_1},\dots,Z_{j_k}\right)$, which we can easily analyze with a rich collection of tools for sum of independent random variables. Finally, we state the Hoeffding's lemma, which we will use to establish the classic Hoeffding-type MGF bound for U-Statistics with bounded kernels.
	\end{remark}
	
	\begin{lemma}[Hoeffding's Lemma]
		\label{lem:hoeffding-lemma}
		Let $X$ be a centered random variable such that $a\le X\le b$ with probability one. Then, for all $\lambda\in\R$, we have:
		\begin{align}
			M_X(\lambda)\le\exp\left(\frac{\lambda^2(b-a)^2}{8}\right).
		\end{align}
	\end{lemma}
	
	\begin{lemma}[Sub-Gaussianity of U-Statistics with Bounded Kernels]
		\label{lem:mgf-ustats}
		Let $U_n^k(h)$ be an order $k$ U-Statistics where $0\le h(X_1,\dots,X_k)\le\mathcal{B}$ $\mathcal{P}^{\otimes k}$-almost-surely. Then, for all $\lambda\in\R$:
		\begin{align}
			\E\left[\exp\left(\lambda\left[U_n^k(h) - \theta\right]\right)\right] \le \exp\left(\frac{\lambda^2\mathcal{B}^2}{8n_k}\right).
		\end{align}
	\end{lemma}
	
	\begin{proof}
		Applying Lemma \ref{lem:gine-decouppling} with $\varphi(x)=e^{\lambda x}$, we have:
		\begin{align*}
			\E\left[\exp\left(\lambda\left[U_n^k(h) - \theta\right]\right)\right] &\le \E\left[\exp\left(\frac{\lambda}{n_k}\sum_{j=1}^{n_k}\left[h(Z_{j_1},\dots,Z_{j_k})-\theta\right]\right)\right] \\
			&= \prod_{j=1}^{n_k}\E\left[\exp\left(\frac{\lambda}{n_k}\left[h(Z_{j_1},\dots,Z_{j_k})-\theta\right]\right)\right] \quad\text{(By Independence)} \\
			&\le \prod_{j=1}^{n_k}\exp\left(\frac{\lambda^2\mathcal{B}^2}{8n_k^2}\right)\\
			&=\exp\left(\frac{\lambda^2\mathcal{B}^2}{8n_k}\right)\qquad\text{(Lemma \ref{lem:hoeffding-lemma})},
		\end{align*}
		\noindent as desired.
	\end{proof}
	
	\subsection{Variance Calculations}
	\begin{proposition}[Variance of U-Statistics]
		\label{prop:variance-of-ustats}
		For an order $k$ U-Statistic $U_n^k(h)$ where the kernel has finite second moment, we have:
		\begin{align}
			\binom{n}{k}\mathrm{Var}(U_n^k(h)) = \sum_{q=1}^k\binom{k}q{\binom{n-k}{k-q}}\sigma_q^2(h),
		\end{align}
		\noindent where $\sigma_q^2(h)\triangleq\mathrm{Var}\big(\E[h(X_1,\dots,X_k)|X_1,\dots,X_q]\big)$. If $h:\mathcal{Z}^k\to\R^d$, we have:
		\begin{align}
			\E\|U_n^k(h)-\theta(h)\|_2^2&=\frac{1}{\binom{n}{k}} \sum_{q=1}^k\binom{k}{q}\binom{n-k}{k-q}\varsigma_q^2(h),
		\end{align}
		\noindent where $\varsigma_q^2(h)\triangleq\mathrm{tr}\left(\mathrm{Cov}(\E[h(X_1,\dots,X_k)|X_1,\dots,X_q])\right)$.
	\end{proposition}
	
	\begin{proof}
		Let $h_*(x_1,\dots,x_k)\triangleq h(x_1,\dots,x_k)-\theta(h)$ be the centered kernel. Let us consider two set of indices ${\bf I}=\{i_1,\dots,i_k\}$ and ${\bf J}=\{j_1,\dots,j_k\}$ where $|{\bf I}\cap{\bf J}|=q$. Then, we have:
		\begin{align*}
			\E[h_*(X_{i_1},\dots,X_{i_k})h_*(X_{j_1},\dots,X_{j_k})] = \sigma_q^2(h).
		\end{align*}
		\noindent The number of distinct choices of such two indices sets is $\binom{n}{k}\binom{k}{q}\binom{n-k}{k-q}$. Hence, we have:
		\begin{align*}
			\mathrm{Var}(U_n^k(h)) &= \frac{1}{\binom{n}{k}^2}\sum_{{\bf I}\in C_{n,k}}\sum_{{\bf J}\in C_{n,k}}\E[h_*(X_{i_1},\dots,X_{i_k})h_*(X_{j_1},\dots,X_{j_k})] \\
			&= \frac{1}{\binom{n}{k}^2}\sum_{q=1}^k\binom{n}{k}\binom{k}{q}\binom{n-k}{k-q}\sigma_q^2(h) \\
			&= \frac{1}{\binom{n}{k}}\sum_{q=1}^k\binom{k}{q}\binom{n-k}{k-q}\sigma_q^2(h),
		\end{align*}
		\noindent as desired. For multivariate case, we have $\E\|U_n^k(h)-\theta(h)\|_2^2 = \mathrm{tr}(\mathrm{Cov}(U_n^k(h)))$. Hence, we have:
		\begin{align*}
			\mathrm{Cov}(U_n^k(h)) &= \frac{1}{\binom{n}{k}^2}\sum_{{\bf I}\in C_{n,k}}\sum_{{\bf J}\in C_{n,k}}\mathrm{Cov}\big(h_*(X_{i_1},\dots,X_{i_k}), h_*(X_{j_1},\dots,X_{j_k})\big) \\
			&= \frac{1}{\binom{n}{k}^2}\sum_{q=1}^k\binom{n}{k}\binom{k}{q}\binom{n-k}{k-q}\mathrm{Cov}(\E[h(X_1,\dots, X_k)|X_1,\dots, X_q]) \\
			&= \frac{1}{\binom{n}{k}}\sum_{q=1}^k\binom{k}{q}\binom{n-k}{k-q}\mathrm{Cov}(\E[h(X_1,\dots, X_k)|X_1,\dots, X_q]).
		\end{align*}
		\noindent Therefore, we have:
		\begin{align*}
			\E\|U_n^k(h)-\theta(h)\|_2^2 &= \mathrm{tr}\left[\mathrm{Cov}(U_n^k(h))\right] = \frac{1}{\binom{n}{k}}\sum_{q=1}^k\binom{k}{q}\binom{n-k}{k-q}\varsigma_q^2(h),
		\end{align*}
		\noindent as desired.
	\end{proof}
	
	\begin{remark}
		We note that $\sigma_q,\varsigma_q$ (for $1\le q\le k$) in the above result is used to denote the $q$-th Hoeffding projections, not to override the notations $\sigma_1,\sigma_2,\varsigma_1,\varsigma_2$ in Table \ref{tab:notation} and in the main text.
	\end{remark}
	
	\begin{corollary}
		\label{cor:variance-bounds-mm}
		For all $k\in[K]$, let $\bar U_{M,\I_k}(h)$ and $\widehat U_{M,\J_k}(h)$ be defined as in Eqn.~\eqref{eq:moru} and Eqn.~\eqref{eq:moiu}, respectively. Then, we have:
		\begin{align}
			\mathrm{Var}\left[\bar U_{M,\I_k}(h)\right] &= \frac{4\sigma_1^2(h)}{M} + \frac{2\sigma_2^2(h)}{M(M-1)}, \\
			\mathrm{Var}\left[\widehat U_{M,\J_k}(h)\right] &=\left(1-\frac{1}{M}\right)\left(\frac{4\sigma_1^2(h)}{n} + \frac{2\sigma_2^2(h)}{n(n-1)}\right) + \frac{\sigma^2(h)}{M}.
		\end{align}
		\noindent Here, $\sigma^2(h)\triangleq\mathrm{Var}(h(X_1,X_2))=2\sigma_1^2(h)+\sigma_2^2(h)$.
	\end{corollary}
	
	\begin{proof}
		The variance bound of $\bar U_{M,\I_k}(h)$ is a direct application of Proposition \ref{prop:variance-of-ustats}. For the second variance bound, let us denote $\widehat U_{M,\J_k}(h)$ as $\widehat U_M$ for short and let $U_n\triangleq\frac{1}{\binom{n}{2}}\sum_{(i,j)\in\Lambda_n}h(X_i,X_j)$, i.e., the full order-$2$ U-Statistic. By total variance, we have:
		\begin{align*}
			\mathrm{Var}(\widehat U_M) &= \mathrm{Var}(\E[\widehat U_M|S_n]) + \E[\mathrm{Var}(\widehat U_M|S_n)] 
			= \mathrm{Var}(U_n)+ \E[\mathrm{Var}(\widehat U_M|S_n)].
		\end{align*}
		\noindent Note that in the last inequality, we used $\E[\widehat U_M|S_n]=U_n$. To compute the second term, we have:
		\begin{align*}
			\mathrm{Var}(\widehat U_M|S_n) &= \frac{\widehat V_n^2}{M},\quad\text{where}\quad \widehat V_n^2=\frac{1}{\binom{n}{2}}\sum_{(i,j)\in\Lambda_n}h^2(X_i,X_j) - U_n^2.
		\end{align*}
		\noindent Therefore, we have:
		\begin{align*}
			\mathrm{Var}(\widehat U_M) &= \mathrm{Var}(U_n) + \frac{1}{M}\E\left[\widehat V_n^2\right] \\
			&= \mathrm{Var}(U_n) +\frac{1}{M}\left[\frac{1}{\binom{n}{2}}\sum_{(i,j)\in\Lambda_n}\E[h^2(X_i,X_j)] - \E[U_n^2]\right] \\
			&= \mathrm{Var}(U_n)+\frac{1}{M}\left[\frac{1}{\binom{n}{2}}\sum_{(i,j)\in\Lambda_n}(\sigma^2(h) + \theta^2(h)) - (\mathrm{Var}(U_n) + \theta^2(h))\right] \\
			&= \mathrm{Var}(U_n)+\frac{1}{M}\left[\sigma^2(h) - \mathrm{Var}(U_n)\right] \\
			&= \left(1-\frac{1}{M}\right)\mathrm{Var}(U_n) + \frac{\sigma^2(h)}{M} \\
			&= \left(1-\frac{1}{M}\right)\left(\frac{4\sigma_1^2(h)}{n} + \frac{2\sigma_2^2(h)}{n(n-1)}\right) + \frac{\sigma^2(h)}{M}.
		\end{align*}
		\noindent The final equality comes from applying Proposition \ref{prop:variance-of-ustats} for $\mathrm{Var}(U_n)$. Now, we officially have all the fundamental tools needed to prove the main results.
	\end{proof}
	
	\section{Proof of Proposition \ref{prop:ugly-vstats-bound}}
	\label{sec:proof-univariate}
	For reader's convenience, the estimator of concern is a median-of-\textit{incomplete}-U-Statistics $\widehat\theta_\mathrm{MoIU}(h)$ (with blocks selected \textit{with replacement}) defined as follows:
	\begin{align*}
		\widehat\theta_\mathrm{MoIU}(h) &\triangleq \mathrm{med}\left(\widehat U_{M,\J_1}(h), \dots,\widehat U_{M,\J_K}(h)\right), \\
		\forall k\in[K]: \widehat U_{M,\J_k}(h) &\triangleq \frac{1}{M}\sum_{(i,j)\in{\bf J}_k} h\left(X_{i}, X_{j}\right),\quad{\bf J}_k\sim_\mathrm{wr}\Lambda_n.
	\end{align*}
	\noindent Here, for each $k\in[K]$, the sequence ${\bf J}_k=\big\{\left(j_{k, 1}, i_{k, 1}\right), \dots, \left(j_{k, M}, i_{k, M}\right)\big\}$ is drawn with replacement from the set of all pairs $\Lambda_n\triangleq\big\{(i,j):1\le i<j\le n\big\}$. Before presenting the main proof, we would like to re-cite the McDiarmid's approach presented in \citet{pmlr-v97-clemencon19a}. Essentially, instead of being decomposed into a convex combination of U-Statistics, the data-functional $\widehat W_n^\varepsilon$ is treated as $\widehat W_n^\varepsilon=f(X_1,\dots,X_n)$, i.e., a function of $n$ data points. Let $S=\{X_1,\dots, X_i, \dots, X_n\}$ and $S_i=\{X_1,\dots,X_i',\dots,X_n\}$, which is the exact copy of $S$ except for a difference at the $i$-th position. Let $\Lambda_n,\Lambda_n'$ be the set of pairs that can be formed from $S,S_i$, respectively. Then, $\Lambda_n$ and $\Lambda_n'$ have exactly $n-1$ different pairs. Hence, the probability that a single randomly chosen pair contains $X_i'$ is $\frac{n-1}{\binom{n}{2}}=\frac{2}{n}$, meaning the probability that \textbf{at least one} out of $M$ selected (with replacement) pairs contains $X_i'$ is $1-\left(1-\frac{2}{n}\right)^M$. Therefore, we have:
	\begin{align*}
		|f(S)-f(S_i)| &\le 1 - \left(1-\frac{2}{n}\right)^M\le \frac{2M}{n}\qquad\text{(Bernoulli Inequality)}.
	\end{align*}
	\noindent Now, we introduce our first supporting Lemma on convex combination of statistics.
	\begin{lemma}[Tail Probability of Convex Combination]
		\label{lem:pitcan}
		Let $X=\sum_{k=1}^U\alpha_kX_k$ for $\alpha_1,\dots,\alpha_U\in\R_+$ such that $\sum_{k=1}^U\alpha_k=1$. Then, for all $t>0$, we have:
		\begin{align}
			\P(X\ge t)\le e^{-\lambda t}\sum_{k=1}^U\alpha_kM_{X_k}(\lambda),\quad\forall\lambda\in\R_+.
		\end{align}
	\end{lemma}
	
	\begin{proof}
		Applying Chernoff's bound, for $\lambda\in\R_+$, we have:
		\begin{align*}
			\P(X\ge t) &\le e^{-\lambda t}M_{X}(\lambda) = e^{-\lambda t}\E\left[\exp\left(\lambda\sum_{k=1}^U\alpha_kX_k\right)\right]  \\
			&\le e^{-\lambda t}\E\left[\sum_{k=1}^U\alpha_k\exp\left(\lambda X_k\right)\right] \qquad\text{(Jensen's Inequality)} \\
			&= e^{-\lambda t}\sum_{k=1}^U\alpha_k\E\left[\exp\left(\lambda X_k\right)\right].
		\end{align*}
		\noindent Hence, we get the desired result. Now, we are ready to tackle the hardest part of the proof.
	\end{proof}
	
	\begin{proposition}[cf. Proposition \ref{prop:ugly-vstats-bound}, Restated]
		\label{prop:ugly-vstats-bound-appendix}
		Let $\widehat W_n^\varepsilon$ be the V-Statistic defined in Eqn.~\eqref{eq:vstats-ugly} and $\widehat q_\varepsilon=\E[\widehat W_n^\varepsilon]$. Suppose that $n\ge 2M$. Then, for all $t>0$, we have:
		\begin{align}
			\boxed{\P(\widehat W_n^\varepsilon-\widehat q_\varepsilon\ge t)\le \exp\left(-\frac{nt^2}{2M}\right)}.
		\end{align}
	\end{proposition}
	\begin{proof}
		Note that for every sequence of pairs ${\bf J}\in\Lambda_n^M$, there are at least $2$ (one pair repeats $M$ times) and at most $2M$ (all pairs are disjoint) unique indices. Therefore, we can write the set of pairs sequences $\Lambda_n^M$ as follows:
		\begin{align}\label{eq:pairs-set-decomposition}
			\Lambda_n^{M} \triangleq \bigcup_{k=2}^{2M}\left[\bigcup_{\mathcal{J}_k\in C_{n,k}}\Theta_{\mathcal{J}_k}\right],\quad\Theta_{\mathcal{J}_k}\triangleq\Big\{{\bf J}\in\Lambda_n^M : \eta({\bf J})=\mathcal{J}_k\Big\}.
		\end{align}
		\noindent Here, $\eta({\bf J})\subseteq[n]$ is the set of all unique indices in the sequence ${\bf J}\in\Lambda_n^M$. Hence, $\Theta_{\mathcal{J}_k}$ can be interpreted as the set of all length-$M$ sequences of pairs whose set of distinct indices is \underline{exactly $\mathcal{J}_k$}. As a result, we can write $\widehat W_n^\varepsilon$ as follows:
		\begin{align*}
			\widehat W_n^\varepsilon &\triangleq \frac{1}{|\Lambda_n^M|}\sum_{{\bf J}\in\Lambda_n^M}\lambda_\varepsilon\left(\widehat U_{M, \J}(h)\right)\qquad\qquad\left(\widehat U_{M,\J}(h) \triangleq \frac{1}{M}\sum_{(i,j)\in{\bf J}} h\left(X_{i}, X_{j}\right)\right) \\
			&= \frac{1}{|\Lambda_n^M|}\sum_{k=2}^{2M}\sum_{\mathcal{J}_k\in C_{n,k}}\left[\sum_{{\bf J}\in\Theta_{\mathcal{J}_k}}\lambda_\varepsilon\left(\widehat U_{M,\J}(h)\right)\right] \qquad (\text{cf. Eqn~\eqref{eq:pairs-set-decomposition}}) \\
			&= \sum_{k=2}^{2M}\frac{G_k\binom{n}{k}}{|\Lambda_n^M|}\left\{\frac{1}{\binom{n}{k}}\sum_{\mathcal{J}_k\in C_{n,k}}\left[\frac{1}{G_k}\sum_{{\bf J}\in\Theta_{\mathcal{J}_k}}\lambda_\varepsilon\left(\widehat U_{M,\J}(h)\right)\right]\right\}\\
			&= \sum_{k=2}^{2M}\alpha_{n,k}U_{n,k}(\phi_k).
		\end{align*}
		\noindent Here, $G_k\triangleq|\Theta_{\mathcal{J}_k}|$ where $\mathcal{J}_k\subseteq[n]$ is an indices subset of size $k$. Trivially, $\sum_{k=2}^{2M}\alpha_{n,k}=1$, making $\widehat W_n^\varepsilon$ a convex combination of order-$k$ U-Statistics $U_{n,k}(\phi_k)$ of kernels $\phi_k$ defined as follows:
		\begin{align}
			\label{eq:phi_k}
			\phi_k\big((X_j)_{j\in\mathcal{J}_k}\big) \triangleq \frac{1}{G_k}\sum_{{\bf J}\in\Theta_{\mathcal{J}_k}} \lambda_\varepsilon\left(\widehat U_{M,\J}(h)\right).
		\end{align}
		\noindent Note that $|\phi_k|\le 1$, i.e., bounded kernel, for all order $k$, meaning we are free to invoke Hoeffding-type treatment for each $U_{n,k}(\phi_k)$. Note that $\widehat q_\varepsilon\triangleq \E[\widehat W_n^\varepsilon] = \sum_{k=2}^{2M}\alpha_{n,k}\E U_{n,k}(\phi_k)$. Therefore, for all $t,\lambda>0$, we have:
		\begin{align*}
			\P(\widehat W_n^\varepsilon-\widehat q_\varepsilon\ge t) &= \P\left(\sum_{k=2}^{2M}\alpha_{n,k}U_{n,k}(\phi_k)-\widehat q_\varepsilon\ge t\right) \\
			&= \P\left(\sum_{k=2}^{2M}\alpha_{n,k}\Big(U_{n,k}(\phi_k)-\E U_{n,k}(\phi_k)\Big)\ge t\right) \\
			&\le e^{-\lambda t}\sum_{k=2}^{2M}\alpha_{n,k}\E\left[\exp\left(\lambda\Big[U_{n,k}(\phi_k)-\E U_{n,k}(\phi_k)\Big]\right)\right]\quad\text{(Lemma \ref{lem:pitcan})}.
		\end{align*}
		\noindent Now, since $U_{n,k}(\phi_k)$ are order-$k$ U-Statistics with kernels $\phi_k$ bounded by $1$, Lemma \ref{lem:mgf-ustats} yields:
		\begin{align*}
			\E\left[\exp\left(\lambda\Big[U_{n,k}(\phi_k)-\E U_{n,k}(\phi_k)\Big]\right)\right] \le \exp\left(\frac{\lambda^2}{8\lfloor n/k\rfloor}\right).
		\end{align*}
		\noindent As a result, we have:
		\begin{align*}
			\P(\widehat W_n^\varepsilon-\widehat q_\varepsilon\ge t) &\le e^{-\lambda t}\sum_{k=2}^{2M}\alpha_{n,k}\exp\left(\frac{\lambda^2}{8\lfloor n/k\rfloor}\right) \\
			&\le \exp\left(-\lambda t + \frac{\lambda^2}{8\lfloor n/2M\rfloor}\right) \quad (k\le 2M) \\
			&\le \exp\left(-\lambda t + \frac{\lambda^2}{2 n/M}\right) \ \qquad (\lfloor x\rfloor\le \frac{x}{2}\text{ for }x\ge 1) \\
			&= \exp\left(-\lambda t + \frac{\lambda^2M}{2n}\right).
		\end{align*}
		\noindent Let $f(\lambda)=-\lambda t+\frac{\lambda^2M}{2n}$. We have $f'(\lambda)=-t+\frac{\lambda M}{n}$, which has root $\lambda_\star=\frac{nt}{M}$. Hence, $f(\lambda)$ attains its minimum at $f(\lambda_\star)=-\frac{nt^2}{2M}$. As a result, we have the desired rate $\P(\widehat W_n^\varepsilon-\widehat q_\varepsilon\ge t)\le \exp(-\frac{nt^2}{2M})$. Now, we are ready to prove the concentration rate of $\widehat\theta_\mathrm{MoIU}$ in the univariate case.
	\end{proof}
	
	\section{Proof of Theorem \ref{thm:univariate-moiu}}
	\begin{proof}\ (Theorem \ref{thm:univariate-moiu})
		For all $k\in[K]$, let $Z_k^\varepsilon=\lambda_\varepsilon\big(\widehat U_{M,\J_k}(h)\big)$ and $\widehat W_n^\varepsilon\triangleq\E[Z_k^\varepsilon|S_n]$, i.e., $\widehat q_\varepsilon=\E[\widehat W_n^\varepsilon]$. Using Chebyshev's inequality and the variance bound in Corollary \ref{cor:variance-bounds-mm}, we have:
		\begin{align*}
			\widehat q_\varepsilon &\le \frac{1}{\varepsilon^2}\left[\left(1-\frac{1}{M}\right)\left(\frac{4\sigma_1^2(h)}{n} + \frac{2\sigma_2^2(h)}{n(n-1)}\right) + \frac{\sigma^2(h)}{M}\right] \\
			&\le \frac{1}{\varepsilon^2}\left[\frac{4\sigma_1^2(h)}{n} + \frac{2\sigma_2^2(h)}{n(n-1)} + \frac{\sigma^2(h)}{M}\right].
		\end{align*}
		\noindent Employing the same strategy for analyzing $\bar\theta_\mathrm{MoRU}$, we have:
		\begin{align*}
			\P\Big(|\widehat\theta_\mathrm{MoIU}(h)-\theta(h)|\ge\varepsilon\Big) &\le \P\left(\frac{1}{K}\sum_{k=1}^K Z_k^\varepsilon\ge\frac{1}{2}\right)=\P\left(\frac{1}{K}\sum_{k=1}^K Z_k^\varepsilon-\widehat q_\varepsilon\ge\frac{1}{2}-\widehat q_\varepsilon\right)\\
			&\le \E_{S_n}\left[\P\left(\frac{1}{K}\sum_{k=1}^K Z_k^\varepsilon-\widehat W_n^\varepsilon\ge\frac{1}{2} - \tau\Big|S_n\right)\right] + \P(\widehat W_n^\varepsilon-\widehat q_\varepsilon\ge \tau-\widehat q_\varepsilon). 
		\end{align*}
		\noindent Here, $\tau\in(0,\frac{1}{2})$ is to be determined later. Since $Z_1^\varepsilon,\dots,Z_K^\varepsilon$ are independent given $S_n$, applying Hoeffding's inequality yields:
		\begin{align*}
			\P\left(\frac{1}{K}\sum_{k=1}^K Z_k^\varepsilon-\widehat W_n^\varepsilon\ge\frac{1}{2} - \tau\Big|S_n\right)&\le \exp\left(-2K\left(\frac{1}{2}-\tau\right)^2\right).
		\end{align*}
		\noindent Furthermore, applying Proposition \ref{prop:ugly-vstats-bound} and the Chebyshev's bound for $\widehat q_\varepsilon$ yields:
		\begin{align*}
			\P\left(\widehat W_n^\varepsilon-\widehat q_\varepsilon\ge \tau-\widehat q_\varepsilon\right) &\le \exp\left(-\frac{n}{2M}(\tau-\widehat q_\varepsilon)^2\right) \\
			&\le \exp\left(-\frac{n}{2M}\left[\tau-\frac{1}{\varepsilon^2}\left(\frac{4\sigma_1^2(h)}{n} + \frac{2\sigma_2^2(h)}{n(n-1)} + \frac{\sigma^2(h)}{M}\right)\right]^2\right).
		\end{align*}
		\noindent As a result, we have:
		\begin{align*}
			&\P(|\widehat\theta_\mathrm{MoIU}(h)-\theta(h)|\ge\varepsilon)\le  \\ &\exp\left(-2K\left(\frac{1}{2}-\tau\right)^2\right) + 
			\exp\left(-\frac{n}{2M}\left[\tau-\frac{1}{\varepsilon^2}\left(\frac{4\sigma_1^2(h)}{n} + \frac{2\sigma_2^2(h)}{n(n-1)} + \frac{\sigma^2(h)}{M}\right)\right]^2\right).
		\end{align*}
		\noindent Choosing $K\ge \frac{\log(2/\delta)}{2(\frac{1}{2}-\tau)^2}$ yields $\exp\left(-2K\left(\frac{1}{2}-\tau\right)^2\right)\le \delta/2$. Let $0<\gamma<\tau$ be such that $n\ge \frac{2\log(2/\delta)}{(\tau-\gamma)^2}$ and the following holds:
		\begin{align*}
			\gamma = \frac{1}{\varepsilon^2}\left(\frac{4\sigma_1^2(h)}{n} + \frac{2\sigma_2^2(h)}{n(n-1)} + \frac{\sigma^2(h)}{M}\right).
		\end{align*}
		\noindent Then, in order for $\exp\left(-\frac{n}{2M}(\tau-\gamma)^2\right)\le\delta/2$ to hold, we need $M\le \frac{n(\tau-\gamma)^2}{2\log(2/\delta)}$. Let $M=\lfloor \frac{n(\tau-\gamma)^2}{2\log(2/\delta)}\rfloor$, i.e., we have $M\ge\frac{n(\tau-\gamma)^2}{4\log(2/\delta)}$ (since $n\ge 2\log(2/\delta)/(\tau-\gamma)^2$, i.e., $M\ge1$). Then:
		\begin{align*}
			\varepsilon^2&=\frac{1}{\gamma}\left(\frac{4\sigma_1^2(h)}{n} + \frac{2\sigma_2^2(h)}{n(n-1)} + \frac{\sigma^2(h)}{M}\right) \\
			&\le \frac{1}{\gamma}\left(\frac{4\sigma_1^2(h)}{n} + \frac{2\sigma_2^2(h)}{n(n-1)} + \frac{4\sigma^2(h)\log(2/\delta)}{n(\tau-\gamma)^2}\right).
		\end{align*}
		\noindent Setting $\gamma=\frac{\tau}{3}$ yields:
		\begin{align*}
			\varepsilon^2\le \frac{12\sigma_1^2(h)}{n\tau} + \frac{6\sigma_2^2(h)}{\tau n(n-1)} + \frac{27\sigma^2(h)\log(2/\delta)}{n\tau^3}.
		\end{align*}
		\noindent As a result, for $n\ge\frac{9\log(2/\delta)}{2\tau^2}$, $K\ge \frac{\log(2/\delta)}{2(\frac{1}{2}-\tau)^2}$ and $M=\lfloor \frac{2n\tau^2}{9\log(2/\delta)}\rfloor$, we have:
		\begin{align*}
			|\widehat\theta_\mathrm{MoIU}(h)-\theta(h)|\le \sqrt{\frac{12\sigma_1^2(h)}{n\tau} + \frac{6\sigma_2^2(h)}{\tau n(n-1)} + \frac{27\sigma^2(h)\log(2/\delta)}{n\tau^3}},
		\end{align*}
		\noindent with probability of at least $1-\delta$, as desired.
	\end{proof}
	
	\begin{remark}
		The choice $\gamma=\frac{\tau}{3}$ is optimal for getting the tightest constant for $\varepsilon$. If we set $\gamma=c\tau$ where $c\in(0,1)$. Then, $\frac{1}{\gamma(\tau-\gamma)^2}=\frac{1}{\tau^3c(1-c)^2}$, which achieves the largest denominator when $c=\frac{1}{3}$.
	\end{remark}
	
	\section{Proof of Theorem \ref{thm:univariate-moiu-swor}}
	\label{sec:about-wor-moiu}
	For reader's convenience, the estimator of concern is a median-of-\textit{incomplete}-U-Statistics $\widetilde\theta_\mathrm{MoIU}(h)$ (with blocks selected \textit{without replacement}) defined as follows:
	\begin{align*}
		\widetilde\theta_\mathrm{MoIU}(h) &\triangleq \mathrm{med}\left(\widetilde U_{M,\J_1}(h), \dots,\widetilde U_{M,\J_K}(h)\right), \\
		\forall k\in[K]: \widetilde U_{M,\J_k} &\triangleq \frac{1}{M}\sum_{(i,j)\in{\bf J}_k} h\left(X_{i}, X_{j}\right),\quad{\bf J}_k\sim_\mathrm{wor}\Lambda_n.
	\end{align*}
	\noindent First, we will begin with calculating the variance of $\widetilde U_{M,\J_k}(h)$.
	\begin{lemma}
		\label{lem:variance-bounds-miou-swor}
		For all $k\in[K]$, let $\widetilde U_{M,\J_k}(h)$ be defined as in Eqn.~\eqref{eq:univariate-moiu-swor} and $\widehat U_{M,\J_k}(h)$ be defined as in Eqn.~\eqref{eq:moiu}. Then, we have:
		\begin{align}
			\mathrm{Var}\left[\widetilde U_{M,\J_k}(h)\right] &\le \mathrm{Var}\left[\widehat U_{M,\J_k}(h)\right] = \left(1-\frac{1}{M}\right)\left(\frac{4\sigma_1^2(h)}{n} + \frac{2\sigma_2^2(h)}{n(n-1)}\right) + \frac{\sigma^2(h)}{M}.
		\end{align}
	\end{lemma}
	\begin{proof}
		For notational brevity, denote $\widetilde U_{M,\J_k}(h)$ as $\widetilde U_M$ and $\widehat U_{M,\J_k}(h)$ as $\widehat U_M$ (for any $k\in[K]$). By total variance (just like the calculation of $\mathrm{Var}[\widehat U_M]$ in Corollary \ref{cor:variance-bounds-mm}), we have:
		\begin{align*}
			\mathrm{Var}[\widetilde U_M] &= \mathrm{Var}\left(\E[\widetilde U_M|S_n]\right) + \E\left[\mathrm{Var}(\widetilde U_M|S_n)\right] = \mathrm{Var}(U_n)+\E\left[\mathrm{Var}(\widetilde U_M|S_n)\right].
		\end{align*}
		\noindent Here, as usual, we denote $U_n$ as the full U-Statistics over the pairs in $S_n$. We have:
		\begin{align*}
			\mathrm{Var}(\widetilde U_M|S_n) &= \frac{\binom{n}{2}-M}{\binom{n}{2}-1}\times\frac{\widehat V_n^2}{M},\quad\text{where}\quad \widehat V_n^2=\frac{1}{\binom{n}{2}}\sum_{(i,j)\in\Lambda_n}h^2(X_i,X_j)-U_n^2.
		\end{align*}
		\noindent Here, $\frac{\binom{n}{2}-M}{\binom{n}{2}-1}$ is the correction factor due to selection without replacement. Then, by the computation in Corollary \ref{cor:variance-bounds-mm}, we have:
		\begin{align*}
			\mathrm{Var}[\widetilde U_M] &= \mathrm{Var}(U_n) + \frac{\binom{n}{2}-M}{\binom{n}{2}-1}\times\frac{\E[\widehat V_n^2]}{M} \le \mathrm{Var}(U_n)+\frac{\E[\widehat V_n^2]}{M}=\mathrm{Var}[\widehat U_M],
		\end{align*}
		\noindent as desired.
	\end{proof}
	
	\begin{proof} (Theorem \ref{thm:univariate-moiu-swor})
		\label{rem:about-wor-moiu}
		We note that if $\widetilde U_{M,\J_k}(h),\forall k\in[K]$ are computed by averaging over pairs selected without replacement from $\Lambda_n$, Theorem \ref{thm:univariate-moiu} also holds. Specifically, let ${\bf J}_1,\dots,{\bf J}_K\subseteq\Lambda_n$ be independent indices sets, each containing $M$ distinct\footnote{Two pairs $(i,j)$ and $(u, v)$ in $\Lambda_n$ are \textbf{distinct} if they share \textbf{at most} one index.} pairs selected without replacement from $\Lambda_n$, then $\widetilde U_{M,\J_k}(h)\triangleq\frac{1}{M}\sum_{(i,j)\in{\bf J}_k}h(X_i,X_j)$. In this case, all Bernoulli indicators $Z_k^\varepsilon,\forall k\in[K]$ are still independent of each other given $S_n$. Hence, the application of regular Hoeffding-type inequalities is not a major issue for $Z_1^\varepsilon,\dots,Z_K^\varepsilon$. However, now the conditional expectation $\widetilde W_n^\varepsilon$ (given $S_n$) is a U-Statistic of order $M$ over the set $\Lambda_n$. Specifically:
		\begin{align}
			\label{eq:tilde-w-n-eps}
			\widetilde W_n^\varepsilon=\frac{1}{|\Gamma_{n}^M|}\sum_{{\bf J}\in \Gamma_n^M}\lambda_\varepsilon(\widetilde U_{M,\J}(h)),\quad \widetilde U_{M,\J}(h)\triangleq\frac{1}{M}\sum_{(i,j)\in{\bf J}}h(X_i,X_j).
		\end{align}
		\noindent Where $\Gamma_n^M$ is defined as follows:
		\begin{align}
			\label{eq:Gamma_nM}
			\Gamma_{n}^M \triangleq \Big\{{\bf J}=(i_m,j_m)_{m=1}^M\in\Lambda_n^M: \forall u\ne v\in[M], (i_u, j_u)\ne (i_v,j_v)\Big\}.
		\end{align}
		\noindent Fortunately, we can still write $\widetilde W_n^\varepsilon$ as a convex combination of U-Statistics. Note that for every sequence of pairs in $\Gamma_n^M$, there are at least $K_{\min}=\min\{k\ge 2: \binom{k}{2}\ge M\}$ unique indices and at most $2M$ unique indices from $[n]$. Therefore, $\widetilde W_n^\varepsilon$ can be written as a convex combination of U-Statistics with orders $k$ where $K_{\min}\le k\le 2M$. Specifically, we can segment $\Gamma_n^M$ as follows: 
		\begin{align}
			\label{eq:Gamma_nM_segmentation}
			\Gamma_n^{M} \triangleq \bigcup_{k=K_{\min}}^{2M}\left[\bigcup_{\mathcal{J}_k\in C_{n,k}}\Xi_{\mathcal{J}_k}\right],\quad\Xi_{\mathcal{J}_k}\triangleq\Big\{{\bf J}\in\Gamma_n^M : \eta({\bf J})=\mathcal{J}_k\Big\}.
		\end{align}
		\noindent Then, the arguments from Proposition \ref{prop:ugly-vstats-bound} can carry over without issues. Then, using the fact that $\mathrm{Var}\left[\widehat U_{M,\J_k}\right]\ge\mathrm{Var}\left[\widetilde U_{M,\J_k}\right]$ for all $k\in[K]$ (cf. Lemma \ref{lem:variance-bounds-miou-swor}), we can bound the deviation probability for $\widetilde W_n^\varepsilon-\E[\widetilde W_n^\varepsilon]$ with the same rate and constants as $\widehat W_n^\varepsilon-\E[\widehat W_n^\varepsilon]$ in Theorem \ref{thm:univariate-moiu}. We will replicate this argument again in the next Section.
	\end{proof}
	
	\section{Proof of Theorem \ref{thm:univariate-moiu-swor-blockwise}}
	For reader's convenience, the estimator of concern is a median-of-\textit{incomplete}-U-Statistics $\widecheck\theta_\mathrm{MoIU}(h)$ (with blocks selected \textit{without replacement} both within and across blocks) defined as follows:
	\begin{align*}
		\widecheck\theta_\mathrm{MoIU}(h) &\triangleq \mathrm{med}\left(\widecheck U_{M,\J_1}(h), \dots,\widecheck U_{M,\J_K}(h)\right), \\
		\forall k\in[K]: \widecheck U_{M,\J_k} &\triangleq \frac{1}{M}\sum_{m=1}^Mh\left(X_{i_m^k}, X_{j_m^k}\right),\quad{\bf J}=\big\{(i_m^k, j_m^k)\big\}_{m,k=1}^{M,K}\sim_\mathrm{wor}\Lambda_n.
	\end{align*}
	\noindent For this proof, we will need to re-visit some tools from Azuma or Serfling-type inequalities. First, we state the following fundamental result.
	
	\begin{proposition}[Azuma Inequality \citep{Azuma1967-hu}]
		\label{prop:azuma-inequality}
		Suppose that $M_1,\dots,M_N$ is a martingale satisfying $|M_k-M_{k-1}|\le c_k$ with probability one. Then, for all $\varepsilon>0$, we have
		\begin{align}
			\P(M_N-M_0\ge\varepsilon)\le \exp\left(-\frac{\varepsilon^2}{2\sum_{k=1}^Nc_k^2}\right).
		\end{align}
	\end{proposition}
	
	\begin{lemma}[Lemma \ref{lem:serfling-swor-blockwise}, Restated]
		\label{lem:serfling-swor-blockwise-appendix}
		Let $\mathcal{P}_N=\{\z_1,\dots,\z_N\}\subseteq\mathcal{Z}$ be a finite population. Let $K, M$ be integers such that $KM^2\le N$ and $h:\mathcal{Z}^M\to[0,1]$ is a \textit{symmetric} kernel in its $M$ arguments. Let $\{Z_{k, m}\}_{k,m=1}^{K,M}\sim_\mathrm{wor}\mathcal{P}_N$ and define $W_K, U_M$ as follows:
		\begin{align}
			W_K\triangleq\frac{1}{K}\sum_{k=1}^Kh\left(Z_{k,1},\dots, Z_{k,M}\right), \quad U_M\triangleq\frac{1}{\binom{N}{M}}\sum_{j_1,\dots,j_M\in C_{N,M}}h\left(\z_{j_1},\dots,\z_{j_M}\right).
		\end{align}
		\noindent Then, for any $t>0$, we have:
		\begin{align}
			\boxed{\P(W_K-U_M\ge t) \le \exp\left(-\frac{Kt^2}{8}\right)}.
		\end{align}
	\end{lemma}
	
	\begin{proof}
		We will adapt a Serfling-type argument \citep{Serfling1974-lz} for concentration inequality in selection without replacement. Let $\mathcal{B}_1,\dots,\mathcal{B}_K\subseteq\mathcal{P}_N$ where $\mathcal{B}_k=\{Z_{k,m}\}_{m=1}^M,\forall k\in[K]$. We will construct a Doob martingale by revealing each block one by one. Specifically, we define:
		\begin{align}
			\forall k\in[K]: M_k &= \E[W_K|\mathcal{B}_1,\dots,\mathcal{B}_k],\qquad M_0=U_M.
		\end{align}
		\noindent For all $k\in[K]$, let $h(Z_{k, 1}, \dots, Z_{k,M})=h(\mathcal{B}_k)$ for notational brevity and let $\mathcal{R}_k$ be the pool of available elements after revealing $\mathcal{B}_1,\dots,\mathcal{B}_k$, i.e., $\mathcal{R}_k:=\mathcal{P}_N\setminus \bigcup_{\ell=1}^k \mathcal{B}_\ell$. Hence, the size of $\mathcal{R}_k$ is $N_k=N-Mk$. We can write $M_k$ as follows:
		\begin{align}
			M_k &= \frac{1}{K}\left[\sum_{\ell=1}^k h(\mathcal{B}_\ell) + (K-k)U_M(\mathcal{R}_k)\right],\\ U_M(\mathcal{R}_k)&\triangleq\frac{1}{\binom{N_k}{M}}\sum_{j_1,\dots,j_M\in\binom{\mathcal{R}_k}{M}}h\left(\z_{j_1},\dots,\z_{j_M}\right).
		\end{align}
		\noindent Recall that the notation $\binom{\mathcal{R}_k}{M}$ specifies the set of $M$ pairs selected without replacement from $\mathcal{R}_k$ (cf. Table \ref{tab:notation}). Here, the first term is the average from knowing the first $k$ blocks and the second term is the expectation of the future blocks depending on the remaining elements in $\mathcal{R}_k$. From the above, we can bound the martingale increment as follows:
		\begin{align*}
			|M_k-M_{k-1}| &= \frac{1}{K}\left|h(\mathcal{B}_k)+(K-k)U_M(\mathcal{R}_k)-(K-k+1)U_M(\mathcal{R}_{k-1})\right| \\
			&\le \frac{1}{K}\left|h(\mathcal{B}_k)-U_M(\mathcal{R}_{k-1})\right| + \frac{K-k}{K}\left|U_M(\mathcal{R}_k)-U_M(\mathcal{R}_{k-1})\right| \\
			&\le \frac{1}{K}+ \frac{K-k}{K}\left|U_M(\mathcal{R}_k)-U_M(\mathcal{R}_{k-1})\right|.
		\end{align*}
		\noindent Now, let $\gamma=\binom{N_{k}}{M}/\binom{N_{k-1}}{M}$, we can write $U_M(\mathcal{R}_{k-1})$ as follows:
		\begin{align*}
			U_M(\mathcal{R}_{k-1}) &= \frac{1}{\binom{N_{k-1}}{M}}\sum_{j_1,\dots,j_M\in\binom{\mathcal{R}_{k-1}}{M}}h\left(\z_{j_1},\dots,\z_{j_M}\right) \\
			&= \frac{1}{\binom{N_{k-1}}{M}}\left[\sum_{j_1,\dots,j_M\in\binom{\mathcal{R}_{k}}{M}}h\left(\z_{j_1},\dots,\z_{j_M}\right)+ \sum_{i_1,\dots,i_M\in\binom{\mathcal{R}_{k-1}}{M}\setminus\binom{\mathcal{R}_k}{M}}h\left(\z_{i_1},\dots,\z_{i_M}\right)\right] \\
			&= \gamma U_M(\mathcal{R}_{k}) + (1-\gamma)U_M^c(\mathcal{R}_{k-1}).
		\end{align*}
		\noindent Here, we used $U_M^c(\mathcal{R}_{k-1})$ to denote the average of the kernel $h$ over the possible $M$-tuples in $\binom{\mathcal{R}_{k-1}}{M}\setminus\binom{\mathcal{R}_k}{M}$. As a result, we can continue bounding $|M_k-M_{k-1}|$ as follows:
		\begin{align*}
			|M_k-M_{k-1}| &\le \frac{1}{K}+\frac{K-k}{K}(1-\gamma)|U_M(\mathcal{R}_k)-U_M^c(\mathcal{R}_{k-1})|\\
			&\le\frac{1}{K}+\frac{K-k}{K}(1-\gamma).
		\end{align*}
		\noindent Now, we have:
		\begin{align*}
			\gamma &= \frac{\binom{N_{k-1}-M}{M}}{\binom{N_{k-1}}{M}} = \prod_{\ell=0}^{M-1}\frac{N_{k-1}-M-\ell}{N_{k-1}-\ell} = \prod_{\ell=0}^{M-1}\left(1-\frac{M}{N_{k-1}-\ell}\right).
		\end{align*}
		\noindent Using the inequality that $1-\prod_{\ell=1}^m(1-a_\ell)\le \sum_{\ell=1}^m a_\ell$ with $a_\ell\in[0,1],\forall \ell\in[m]$ (i.e., the Weierstrass product inequality), we have:
		\begin{align*}
			|M_k-M_{k-1}| &\le \frac{1}{K}+\frac{K-k}{K}\sum_{\ell=0}^{M-1}\frac{M}{N_{k-1}-\ell} \le \frac{1}{K}+\frac{1}{K}\cdot\frac{M^2(K-k)}{N_{k-1}-M+1} \\
			&\le \frac{1}{K}+\frac{1}{K}\cdot\frac{M^2K-kM^2}{N-kM+1} \quad\ \ (N_{k-1}=N-kM+M) \\
			&\le \frac{1}{K}+\frac{1}{K}\cdot\frac{N-kM^2}{N-kM+1} \qquad (KM^2\le N\text{ by assumption})\\
			&\le \frac{2}{K}.
		\end{align*}
		\noindent Hence, applying Azuma's inequality (cf. Proposition \ref{prop:azuma-inequality}) yields:
		\begin{align*}
			\P(W_K-U_M\ge t) &= \P(M_K-M_0\ge t) \le \exp\left(-\frac{t^2}{2\sum_{k=1}^K(2/K)^2}\right) = \exp\left(-\frac{Kt^2}{8}\right),
		\end{align*}
		\noindent as desired. Now, we are ready to prove the main result.
	\end{proof}
	
	\begin{proof} (Theorem \ref{thm:univariate-moiu-swor-blockwise})
		Suppose we know a-priori that $KM^2\le \binom{n}{2}=\frac{1}{2}n(n-1)$. For all $k\in[K]$, define $Z_k^\varepsilon=\lambda_\varepsilon(\widecheck{U}_{M,\J_k}(h))$. We proceed to break the failure probability as follows:
		\begin{align*}
			&\P\Big(|\widecheck{\theta}_\mathrm{MoIU}(h) - \theta(h)|\ge\varepsilon\Big) \le \\
			&\E_{S_n}\left[\P\left(\frac{1}{K}\sum_{k=1}^K Z_k^\varepsilon-\widecheck W_n^\varepsilon\ge\frac{1}{2} - \tau\Big|S_n\right)\right] + \P(\widecheck W_n^\varepsilon-\widecheck q_\varepsilon\ge \tau-\widecheck q_\varepsilon).
		\end{align*}
		\noindent Here, we have $\widecheck q_\varepsilon=\E[\widecheck W_n^\varepsilon]$ and $\widecheck W_n^\varepsilon\triangleq\E[\frac{1}{K}\sum_{k=1}^KZ_k^\varepsilon|S_n]$, which is an order-$M$ U-Statistic over the collection of unique pairs $\Lambda_n$. Let $\Gamma_n^M\subset\Lambda_n^M$ be the set of sequences with $M$ unique pairs selected from $\Lambda_n$ without replacement (cf. also Eqn.~\eqref{eq:Gamma_nM}):
		\begin{align}
			\Gamma_{n}^M \triangleq \Big\{{\bf J}=(i_m,j_m)_{m=1}^M\in\Lambda_n^M: \forall u\ne v\in[M], (i_u, j_u)\ne (i_v,j_v)\Big\}.
		\end{align}
		\noindent Then, $\widecheck{W}_n^\varepsilon$ is defined as follows:
		\begin{align}
			\widecheck W_n^\varepsilon=\frac{1}{|\Gamma_{n}^M|}\sum_{{\bf J}\in \Gamma_n^M}\lambda_\varepsilon\left(\widetilde U_{M,\J}(h)\right),\quad \widetilde U_{M,\J}(h)\triangleq\frac{1}{M}\sum_{(i,j)\in{\bf J}}h(X_i,X_j).
		\end{align}
		\noindent Hence, we note that the definition of $\widecheck{W}_n^\varepsilon$ is identical to that of $\widetilde W_n^\varepsilon$ (cf. Eqn~\eqref{eq:tilde-w-n-eps}). This is because even though $Z_1^\varepsilon,\dots,Z_K^\varepsilon$ are dependent due to selection without replacement, their expectations (given $S_n$) are the same without conditioning on each other. As before, we can segment $\Gamma_n^M$ as:
		\begin{align*}
			\Gamma_n^{M} \triangleq \bigcup_{k=K_{\min}}^{2M}\left[\bigcup_{\mathcal{J}_k\in C_{n,k}}\Xi_{\mathcal{J}_k}\right],\quad\Xi_{\mathcal{J}_k}\triangleq\Big\{{\bf J}\in\Gamma_n^M : \eta({\bf J})=\mathcal{J}_k\Big\}.
		\end{align*}
		\noindent Then, just like Proposition \ref{prop:ugly-vstats-bound}, we can break $\widecheck W_n^\varepsilon$ into a convex combination of U-Statistics with order of at most $2M$. Specifically:
		\begin{align*}
			\widecheck W_n^\varepsilon &= \sum_{k=K_{\min}}^{2M}\beta_{n,k}U_{n,k}(\varphi_k),
		\end{align*}
		\noindent where $\beta_{n,k}\triangleq\frac{F_k\binom{n}{k}}{|\Gamma_n^M|}$ and $F_k\triangleq|\Xi_{\mathcal{J}_k}|$ for a subset $\mathcal{J}_k\subseteq[n]$ of $k$ indices. The order-$k$ kernel $\varphi_k$ is defined for each set $\mathcal{J}_k$ as follows:
		\begin{align*}
			\varphi_k\big((X_j)_{j\in\mathcal{J}_k}\big) \triangleq \frac{1}{F_k}\sum_{{\bf J}\in\Xi_{\mathcal{J}_k}} \lambda_\varepsilon\left(\widetilde  U_{M,\J}(h)\right).
		\end{align*}
		\noindent Then, by Chernoff's bound for convex combination, for all $\lambda>0$, we have:
		\begin{align*}
			\P(\widecheck W_n^\varepsilon-\widecheck q_\varepsilon\ge t) &= \P\left(\sum_{k=K_{\min}}^{2M}\beta_{n,k}U_{n,k}(\varphi_k)-\widecheck q_\varepsilon\ge t\right) \\
			&= \P\left(\sum_{k=K_{\min}}^{2M}\beta_{n,k}\Big(U_{n,k}(\varphi_k)-\E U_{n,k}(\varphi_k)\Big)\ge t\right) \\
			&\le e^{-\lambda t}\sum_{k=K_{\min}}^{2M}\beta_{n,k}\E\left[\exp\left(\lambda\Big[U_{n,k}(\varphi_k)-\E U_{n,k}(\varphi_k)\Big]\right)\right]\quad\text{(Lemma \ref{lem:pitcan})}\\
			&\le e^{-\lambda t}\sum_{k=K_{\min}}^{2M}\beta_{n,k}\exp\left(\frac{\lambda^2}{8\lfloor n/k\rfloor}\right) \qquad \text{(Hoeffding's Lemma)} \\
			&\le \exp\left(-\lambda t + \frac{\lambda^2}{8\lfloor n/2M \rfloor}\right) \qquad (\text{Since }k\le 2M) \\
			&\le \exp\left(-\lambda t + \frac{\lambda^2M}{2n}\right).
		\end{align*}
		\noindent Let $\lambda=\frac{nt}{M}$ yields $\P(\widecheck W_n^\varepsilon-\widecheck q_\varepsilon)\le \exp(-\frac{nt^2}{2M})$. Now, let us analyze the (dependent) indicators $Z_k^\varepsilon,\forall k\in[K]$. From Lemma \ref{lem:serfling-swor-blockwise} and the initial assumption that $KM^2\le\binom{n}{2}$, we have:
		\begin{align*}
			\P\left(\frac{1}{K}\sum_{k=1}^KZ_k^\varepsilon-\widecheck W_n^\varepsilon\ge\frac{1}{2}-\tau\Big|S_n\right)&\le \exp\left(-\frac{K}{8}\left(\frac{1}{2}-\tau\right)^2\right).
		\end{align*}
		\noindent Combining the bounds and the Chebyshev inequality on $\widecheck q_\varepsilon$, we have:
		\begin{align*}
			&\P\Big(|\widecheck{\theta}_\mathrm{MoIU}(h) - \theta(h)|\ge\varepsilon\Big) \le \exp\left(-\frac{K}{8}\left(\frac{1}{2}-\tau\right)^2\right) + \exp\left(-\frac{n}{M}\left(\tau-\widecheck q_\varepsilon\right)^2\right) \\
			&\exp\left(-\frac{K}{8}\left(\frac{1}{2}-\tau\right)^2\right) + \exp\left(-\frac{n}{2M}\left[\tau-\frac{1}{\varepsilon^2}\left(\frac{4\sigma_1^2(h)}{n} + \frac{2\sigma_2^2(h)}{n(n-1)} + \frac{\sigma^2(h)}{M}\right)\right]^2\right).
		\end{align*}
		\noindent Note that in the above, we use the following variance bound for the Chebyshev inequality on $\widecheck q_\varepsilon$:
		\begin{align*}
			\mathrm{Var}(\widetilde U_{M,\J_k})\le \mathrm{Var}(\widehat U_{M,\J_k}) \le \frac{4\sigma_1^2(h)}{n} + \frac{2\sigma_2^2(h)}{n(n-1)} + \frac{\sigma^2(h)}{M} \qquad\text{(Corollary \ref{cor:variance-bounds-mm})}.
		\end{align*}
		\noindent In order for $\exp\left(-\frac{K}{8}\left(\frac{1}{2}-\tau\right)^2\right)\le \delta/2$ to hold, we need $K\ge\frac{8}{(\frac{1}{2}-\tau)^2}\log(2/\delta)$. Furthermore, in order for $\exp\left(-\frac{n}{2M}\left[\tau-\frac{1}{\varepsilon^2}\left(\frac{4\sigma_1^2(h)}{n} + \frac{2\sigma_2^2(h)}{n(n-1)} + \frac{\sigma^2(h)}{M}\right)\right]^2\right)\le \delta/2$ to hold, using the same approach as in Theorem \ref{thm:univariate-moiu}, we have:
		\begin{align*}
			\varepsilon\le\sqrt{\frac{12\sigma_1^2(h)}{n\tau} + \frac{6\sigma_2^2(h)}{\tau n(n-1)} + \frac{27\sigma^2(h)\log(2/\delta)}{n\tau^3}},
		\end{align*}
		\noindent for $M\le\frac{2n\tau^2}{9\log(2/\delta)}$. Let $K=\lceil\frac{8}{(\frac{1}{2}-\tau)^2}\log(2/\delta)\rceil\le \frac{9\log(2/\delta)}{(\frac{1}{2}-\tau)^2}$ and $M=\lfloor\frac{2n\tau^2}{9\log(2/\delta)}\rfloor$, we have:
		\begin{align*}
			KM^2 &\le \frac{9\log(2/\delta)}{(\frac{1}{2}-\tau)^2}\times\frac{4n^2\tau^4}{81\log^2(2/\delta)} = n^2\times\frac{4\tau^4}{9(\frac{1}{2}-\tau)^2\log(2/\delta)} = n^2g(\tau,\delta).
		\end{align*}
		\noindent Hence, for $\delta,\tau>0$ such that $g(\tau,\delta)\le\frac{1}{4}$, we have:
		\begin{align*}
			KM^2\le n^2g(\tau,\delta) \le \frac{n^2}{4}\le \frac{1}{2}n(n-1)=\binom{n}{2}\qquad\forall n\ge 2,
		\end{align*}
		\noindent as desired.
	\end{proof}
	
	\begin{remark}
		To clear any doubts on the definition of $\widecheck{W}_n^\varepsilon$, we note that $\E[\frac{1}{K}\sum_{k=1}^K Z_k^\varepsilon|S_n]$ is precisely the order-$M$ U-Statistic defined by $\widecheck W_n^\varepsilon$, not a U-Statistic with order $KM$. This is because the expectations of all $Z_k^\varepsilon,\forall k\in[K]$ are the same conditionally given $S_n$. This can be seen most intuitively via Giné decoupling. Specifically, let $g:\mathcal{Z}^{2M}\to[0,1]$ be a symmetric kernel where each argument is a pair. Then, let $\Lambda_n=\{P_1,\dots,P_N\}$ where $N=\binom{n}{2}$ and for any $P_\ell=(i,j)\in\Lambda_n$, denote $Z_{P_\ell}=(X_i,X_j)$ as the corresponding data pair for notational brevity. We have: 
		\begin{align*}
			\E[W_K|S_n] &= \frac{1}{\binom{N}{K\times M}}\sum_{(j_{k,m})_{k,m=1}^{K,M}\in C_{N,KM}}\frac{1}{K}\sum_{k=1}^Kg\left(\big(Z_{P_{j_{k,m}}}\big)_{m=1}^M\right) \qquad \Bigg(N=\binom{n}{2}\Bigg) \\
			&= \frac{1}{N!}\sum_{\pi\in\Pi_N}\frac{1}{K}\sum_{k=1}^Kg\left(\big(Z_{P_{\pi(kM-M+m)}}\big)_{m=1}^M\right) \\
			&= \frac{1}{K}\sum_{k=1}^K\left[\frac{1}{N!}\sum_{\pi\in\Pi_N}g\left(\big(Z_{P_{\pi(kM-M+m)}}\big)_{m=1}^M\right)\right] \\
			&= \frac{1}{K}\sum_{k=1}^K\left[\frac{1}{\binom{N}{M}}\sum_{j_1,\dots,j_M\in C_{N,M}}g\left(\big(Z_{P_{j_m}}\big)_{m=1}^M\right)\right] \\
			&= \frac{1}{\binom{N}{M}}\sum_{j_1,\dots,j_M\in C_{N,M}}g\left(\big(Z_{P_{j_m}}\big)_{m=1}^M\right).
		\end{align*}
		\noindent To be clear, the kernel $g((X_{i_m^k}, X_{j_m^k})_{m=1}^M)$ is equivalent to $\lambda_\varepsilon(\frac{1}{M}\sum_{m=1}^M h(X_{i_m^k}, X_{j_m^k}))$ in the proof.
	\end{remark}
	
	\begin{remark}
		We used the variance of $\widetilde U_{M,\J_k}(h)$ for the Chebyshev bound of $\widecheck q_\varepsilon$ because $\widecheck U_{M,\J_k}$ and $\widetilde U_{M,\J_k}$ has the same variance (both bounded by the variance of $\widehat U_{M,\J_k}$) marginally over $S_n$. Again, the dependence across blocks do not affect the variance of each sub-estimator $\widecheck U_{M,\J_k}$ unless they are conditioned on one another.
	\end{remark}
	\section{Proof of Theorem \ref{thm:multivariate-moiu}}
	For convenience, we re-iterate that the estimator considered is $\widehat\Theta_\mathrm{MoIU}(h)$ (for the multivariate kernel $h:\mathcal{Z}\times\mathcal{Z}\to\R^d$) defined as the following geometric median of incomplete U-Statistics:
	\begin{align*}
		\widehat\Theta_\mathrm{MoIU}(h) &\triangleq \arg\min_{\z\in\R^d}\sum_{k=1}^K\left\|\z-\widehat U_{M,\J_k}(h)\right\|_2, \\
		\forall k\in[K]: \widehat U_{M,\J_k}(h) &\triangleq \frac{1}{M}\sum_{(i,j)\in{\bf J}_k}h\left(X_{i}, X_{j}\right),\quad{\bf J}_k\sim_\mathrm{wr}\Lambda_n.
	\end{align*}
	\noindent First, we will derive the Chebyshev bound analogous to the univariate case.
	
	\begin{lemma}[Chebyshev Bound - Multivariate Case]
		\label{lem:chebyshev-bound-multivariate}
		For all $k\in[K]$, let $\widehat U_{M,\J_k}$ be defined as in Eqn.~\eqref{eq:moiu} where $h:\mathcal{Z}\times\mathcal{Z}\to\R^d$ is a multivariate kernel. Then, we have:
		\begin{align}
			\E\|\widehat U_{M,\J_k}(h) - \Theta(h)\|_2^2 = \left(1-\frac{1}{M}\right)\left(\frac{4\varsigma_1^2(h)}{n}+\frac{2\varsigma_2^2(h)}{n(n-1)}\right) + \frac{\varsigma^2(h)}{M},
		\end{align}
		\noindent where $\varsigma_1^2(h)\triangleq\mathrm{tr}(\mathrm{Cov}(\E[h(X_1,X_2)|X_1]))$, $\varsigma^2(h)\triangleq\mathrm{tr}(\mathrm{Cov}(h(X_1,X_2)))$ and $\varsigma_2^2(h)\triangleq\varsigma^2(h)-2\varsigma_1^2(h)$.
	\end{lemma}
	
	\begin{proof}
		Let $U_n\triangleq\frac{1}{\binom{n}{2}}\sum_{(i,j)\in\Lambda_n}h(X_i,X_j)$ and write $\widehat U_{M,\J_k}$ as $\widehat U_M$ for short. First, note the identity $\E\|\widehat U_M-\Theta(h)\|_2^2=\mathrm{tr}\left[\mathrm{Cov}(\widehat U_M)\right]$. By the law of total covariance, we have:
		\begin{align*}
			\mathrm{Cov}(\widehat U_M) &= \mathrm{Cov}(\E[\widehat U_M|S_n]) + \E[\mathrm{Cov}(\widehat U_M|S_n)] = \mathrm{Cov}(U_n) + \E[\mathrm{Cov}(\widehat U_M|S_n)].
		\end{align*}
		\noindent To compute the conditional covariance, we have:
		\begin{align*}
			\mathrm{Cov}(\widehat U_M|S_n) &= \frac{1}{M}(\widehat V_n-U_nU_n^\top),\quad\text{where}\quad\widehat V_n=\frac{1}{\binom{n}{2}}\sum_{(i,j)\in\Lambda_n}h(X_i,X_j)h(X_i,X_j)^\top.
		\end{align*}
		\noindent Let $\Sigma_h\triangleq\mathrm{Cov}(h(X_1,X_2))$, we have:
		\begin{align*}
			\mathrm{Cov}(\widehat U_M) &= \mathrm{Cov}(U_n)+\E\left[\frac{1}{M}(\widehat V_n-U_nU_n^\top)\right] \\
			&= \mathrm{Cov}(U_n) + \frac{1}{M}\E[\widehat V_n] - \frac{1}{M}\E[U_nU_n^\top] \\
			&= \mathrm{Cov}(U_n) + \frac{1}{M}\E[h(X_1,X_2)h(X_1,X_2)^\top] - \frac{1}{M}\E[U_nU_n^\top] \\
			&= \mathrm{Cov}(U_n) + \frac{1}{M}(\Sigma_h+\Theta(h)\Theta(h)^\top) - \frac{1}{M}(\mathrm{Cov}(U_n)+\Theta(h)\Theta(h)^\top) \\
			&= \left(1-\frac{1}{M}\right)\mathrm{Cov}(U_n) + \frac{1}{M}\Sigma_h.
		\end{align*}
		\noindent Then, we have:
		\begin{align*}
			\E\|\widehat U_M-\Theta(h)\|_2^2 &= \mathrm{tr}\left[\mathrm{Cov}(\widehat U_M)\right] = \left(1-\frac{1}{M}\right)\mathrm{tr}\left[\mathrm{Cov}(U_n)\right] + \frac{1}{M}\mathrm{tr}(\Sigma_h) \\
			&= \left(1-\frac{1}{M}\right)\left(\frac{4\varsigma_1^2(h)}{n}+\frac{2\varsigma_2^2(h)}{n(n-1)}\right) + \frac{\varsigma^2(h)}{M},
		\end{align*}
		\noindent as desired. Now, to apply the median-of-means trick for geometric median, we present the following key Lemma (with proof for Hilbert spaces).
	\end{proof}
	
	\begin{lemma}[cf. Lemma 2.1 in \citet{article:sminsker2018}]
		\label{lem:key-geometric-lemma}
		Let $(\mathcal{Z},\|\cdot\|)$ be a norm space and $\z_1,\dots, \z_n\in\mathcal{Z}$ with $\z_*$ be their geometric median, i.e., $\z_*\triangleq\arg\min_{\z\in\mathcal{Z}}\sum_{j=1}^n\|\z_j-\z\|$. Fix $\alpha\in(0,\frac{1}{2})$ and assume that $\z\in\mathcal{Z}$ is a point such that $\|\z-\z_*\|> rC_\alpha$ where $C_\alpha$ depends on $\alpha$ and $r>0$. Then there exists $\mathcal{J}\subseteq[n]$ with $|\mathcal{J}|\ge\alpha n$ such that $\forall j\in\mathcal{J}, \|\z_j-\z\|>r$. Specifically:
		\begin{enumerate}[label=(\roman*)]
			\item If $(\mathcal{Z},\|\cdot\|)$ is induced by a Hilbert space then the above holds with $C_\alpha=\frac{1-\alpha}{\sqrt{1-2\alpha}}$.
			\item If $(\mathcal{Z},\|\cdot\|)$ is a general Banach space then the above holds with $C_\alpha=\frac{2(1-\alpha)}{1-2\alpha}$.
		\end{enumerate}
	\end{lemma}
	
	\begin{proof}
		Without loss of generality, assume that $\z_j\ne\z_*,\forall j\in[n]$. Let $\u=\z-\z_*$ and $t=\|\u\|$. Define the set $\mathcal{J}_*\triangleq\big\{j\in[n]:\|\z_j-\z\|\le r\big\}$. Suppose that $|\mathcal{J}_*|>(1-\alpha)n$, contradicting the conclusion that $|\mathcal{J}|\ge\alpha n$.
		
		\noindent Since $\z_*$ minimizes $F(\z)=\sum_{j=1}^n\|\z_j-\z\|$, we have $\sum_{j=1}^n\v_j=0$ where $\v_j=\nabla\|\z_j-\z_*\|=\frac{\z_j-\z_*}{\|\z_j-\z_*\|}$. Then, we have:
		\begin{align*}
			0= \bigg<\u,\sum_{j=1}^n\v_j\bigg> &= \sum_{j\in\mathcal{J}_*}\big<\u,\v_j\big>+\sum_{i\notin\mathcal{J}_*}\big<\u,\v_i\big>.
		\end{align*}
		\noindent\textbf{Lower-bound for $\big<\u,\v_j\big>$ for $j\in\mathcal{J}_*$}: We have $\|\z_j-\z\|=\|(\z_j-\z_*)-\u\|\le r$. Therefore, we have $t^2+\|\z_j-\z_*\|^2-2\big<\z_j-\z_*,\u\big>\le r^2$. Hence, we have:
		\begin{align*}
			\big<\u,\v_j\big> &\ge \frac{\|\z_j-\z_*\|^2+t^2-r^2}{2\|\z_j-\z_*\|} \ge \frac{2\|\z_j-\z_*\|\sqrt{t^2-r^2}}{2\|\z_j-\z_*\|} = \sqrt{t^2-r^2}.
		\end{align*}
		\noindent The second inequality comes from AM-GM and the fact that $t>rC_\alpha>r$.
		
		\noindent\textbf{Lower-bound for $\big<\u,\v_i\big>$ for $i\notin\mathcal{J}_*$}: Trivially, by Cauchy-Schwarz, we have
		\begin{align*}
			\big<\u,\v_i\big>\ge -|\big<\u,\v_i\big>| \ge -\|\u\|\cdot\|\v_i\| \ge -t.
		\end{align*}
		\noindent Combining both cases lower bounds, we  have:
		\begin{align*}
			0 &= \sum_{j\in\mathcal{J}_*}\big<\u,\v_j\big>+\sum_{i\notin\mathcal{J}_*}\big<\u,\v_i\big> \ge |\mathcal{J}_*|\sqrt{t^2-r^2} - (n - |\mathcal{J}_*|)t.
		\end{align*}
		\noindent As a result, using the fact that $t>rC_\alpha$, we have:
		\begin{align*}
			|\mathcal{J}_*| &\le \frac{n}{1+\sqrt{1-\frac{r^2}{t^2}}} \le \frac{n}{1+\sqrt{1-\frac{r^2}{r^2C_\alpha^2}}} = \frac{n}{1+\sqrt{1-\frac{1-2\alpha}{(1-\alpha)^2}}} = (1-\alpha)n.
		\end{align*}
		\noindent Hence, this contradicts the initial assumption that $|\mathcal{J}_*|>(1-\alpha)n$. For interested readers, we refer to the proof for general Banach spaces in \citet[Appendix]{article:sminsker2018}. Now, we are ready to present the proof of Theorem \ref{thm:multivariate-moiu}.
	\end{proof}
	
	\begin{remark}
		The reason why we assume $\z_j\ne\z_*,\forall j\in[n]$ is because we need to compute $\v_j$ as gradients of $\|\z_*-\z_j\|$ for all $j\in\mathcal{J}_*$. If we relax the assumption that $\z_j\ne\z_*,\forall j\in[n]$, we can define $\v_j$ as sub-gradients instead of gradients for full rigor. However, that would be unnecessary because $\z_*$ cannot belong to $\mathcal{J}_*$ anyway due to the initial assumption that $\|\z-\z_*\|>rC_\alpha>r$.
	\end{remark}
	
	\begin{proof} (Theorem \ref{thm:multivariate-moiu})
		Let $\alpha\in(0,\frac{1}{2})$ and $C_\alpha=\frac{1-\alpha}{\sqrt{1-2\alpha}}$, from Lemma \ref{lem:key-geometric-lemma}, we have the following event inclusion:
		\begin{align}
			\label{eq:event-inclusion-gmedian}
			\left\{
			\|\widehat\Theta_\mathrm{MoIU}(h)-\Theta(h)\|_2\ge\varepsilon
			\right\} \subseteq\left\{\frac{1}{K}\sum_{k=1}^K\mathds{1}_{\|\widehat U_{M,\J_k}(h)-\Theta(h)\|_2\ge C_\alpha^{-1}\varepsilon}\ge \alpha\right\}.
		\end{align}
		\noindent For all $k\in[K]$, let $Z_k^{\varepsilon,\alpha}\triangleq\lambda_{\varepsilon,\alpha}\big(\widehat U_{M,\J_k}(h)\big)$ where $\lambda_{\varepsilon,\alpha}(\z)\triangleq\mathds{1}_{\|\z-\Theta(h)\|_2\ge C_\alpha^{-1}\varepsilon}$. Let $\widehat W_n^{\varepsilon,\alpha}\triangleq\E[Z_k^{\varepsilon,\alpha}|S_n]$ and $\widehat q_{\varepsilon,\alpha}\triangleq\E[\widehat W_n^{\varepsilon,\alpha}]$. First, note that by Chebyshev's inequality, we have:
		\begin{align*}
			\widehat q_{\varepsilon,\alpha} &= \P(\|\widehat U_{M,\J_k}(h)-\Theta(h)\|_2\ge C_\alpha^{-1}\varepsilon) \le \frac{C_\alpha^2\E\|\widehat U_{M,\J_k}(h)-\Theta(h)\|_2^2}{\varepsilon^2} \\
			&\le \frac{C_\alpha^2}{\varepsilon^2}\left[\left(1-\frac{1}{M}\right)\left(\frac{4\varsigma_1^2(h)}{n}+\frac{2\varsigma_2^2(h)}{n(n-1)}\right) + \frac{\varsigma^2(h)}{M}\right] \quad\text{(Lemma \ref{lem:chebyshev-bound-multivariate})}\\
			&\le \frac{C_\alpha^2}{\varepsilon^2}\left[\frac{4\varsigma_1^2(h)}{n}+\frac{2\varsigma_2^2(h)}{n(n-1)} + \frac{\varsigma^2(h)}{M}\right].
		\end{align*}
		\noindent Furthermore, due to the event inclusion in Eqn.~\eqref{eq:event-inclusion-gmedian}, we have:
		\begin{align*}
			&\P\Big(\|\widehat\Theta_\mathrm{MoIU}(h)-\Theta(h)\|_2\ge\varepsilon\Big) \\ 
			&\le \P\left(\frac{1}{K}\sum_{k=1}^K Z_k^{\varepsilon,\alpha}\ge \alpha\right) = \P\left(\frac{1}{K}\sum_{k=1}^K Z_k^{\varepsilon,\alpha}-\widehat q_{\varepsilon,\alpha}\ge \alpha-\widehat q_{\varepsilon,\alpha}\right) \\
			&\le \E_{S_n}\left[\P\left(\frac{1}{K}\sum_{k=1}^K Z_k^{\varepsilon,\alpha}-\widehat W_n^{\varepsilon,\alpha}\ge \alpha-\tau\Big|S_n\right)\right] + \P(\widehat W_n^{\varepsilon,\alpha}-\widehat q_{\varepsilon,\alpha}\ge\tau-\widehat q_{\varepsilon,\alpha}),
		\end{align*}
		\noindent where $\tau\in(0,\alpha)$ is a free variable to be determined. Since $Z_k^{\varepsilon,\alpha},\forall k\in[K]$ are independent given $S_n$, applying Hoeffding's inequality yields:
		\begin{align*}
			\P\left(\frac{1}{K}\sum_{k=1}^K Z_k^{\varepsilon,\alpha}-\widehat W_n^{\varepsilon,\alpha}\ge \alpha-\tau\Big|S_n\right) &\le \exp\left(-2K(\alpha-\tau)^2\right).
		\end{align*}
		\noindent To handle the second failure probability, invoking Proposition \ref{prop:ugly-vstats-bound}\footnote{Note that the structure of the V-Statistics $\widehat W_n^{\varepsilon,\alpha}$ is identical to that of $\widehat W_n^\varepsilon$ in the univariate case. Hence, the application Proposition \ref{prop:ugly-vstats-bound} is valid.} yields:
		\begin{align*}
			\P(\widehat W_n^{\varepsilon,\alpha}-\widehat q_{\varepsilon,\alpha}\ge \tau-\widehat q_{\varepsilon,\alpha}) &\le \exp\left(-\frac{n}{2M}(\tau-\widehat q_{\varepsilon,\alpha})^2\right) \\
			&\le\exp\left(-\frac{n}{2M}\left[\tau-\frac{C_\alpha^2}{\varepsilon^2}\left(\frac{4\varsigma_1^2(h)}{n}+\frac{2\varsigma_2^2(h)}{n(n-1)} + \frac{\varsigma^2(h)}{M}\right)\right]^2\right).
		\end{align*}
		\noindent Combining the failure probabilities, we have:
		\begin{align*}
			&\P\Big(\|\widehat\Theta_\mathrm{MoIU}(h)-\Theta(h)\|_2\ge\varepsilon\Big) \le \\ 
			&\exp\left(-2K(\alpha-\tau)^2\right) + \exp\left(-\frac{n}{2M}\left[\tau-\frac{C_\alpha^2}{\varepsilon^2}\left(\frac{4\varsigma_1^2(h)}{n}+\frac{2\varsigma_2^2(h)}{n(n-1)} + \frac{\varsigma^2(h)}{M}\right)\right]^2\right).
		\end{align*}
		Choosing $K\ge \frac{\log(2/\delta)}{2(\alpha-\tau)^2}$ yields $\exp\left(-2K\left(\alpha-\tau\right)^2\right)\le \delta/2$. Let $0<\gamma<\tau$ be such that $n\ge \frac{2\log(2/\delta)}{(\tau-\gamma)^2}$ and the following holds:
		\begin{align*}
			\gamma = \frac{C_\alpha^2}{\varepsilon^2}\left(\frac{4\varsigma_1^2(h)}{n}+\frac{2\varsigma_2^2(h)}{n(n-1)} + \frac{\varsigma^2(h)}{M}\right).
		\end{align*}
		\noindent Then, in order for $\exp\left(-\frac{n}{2M}(\tau-\gamma)^2\right)\le\delta/2$ to hold, we need $M\le \frac{n(\tau-\gamma)^2}{2\log(2/\delta)}$. Let $M=\lfloor \frac{n(\tau-\gamma)^2}{2\log(2/\delta)}\rfloor$, i.e., we have $M\ge\frac{n(\tau-\gamma)^2}{4\log(2/\delta)}$ (since $n\ge 2\log(2/\delta)/(\tau-\gamma)^2$, i.e., $M\ge1$). Then:
		\begin{align*}
			\varepsilon^2&=\frac{C_\alpha^2}{\gamma}\left(\frac{4\varsigma_1^2(h)}{n} + \frac{2\varsigma_2^2(h)}{n(n-1)} + \frac{\varsigma^2(h)}{M}\right) \\
			&\le \frac{C_\alpha^2}{\gamma}\left(\frac{4\varsigma_1^2(h)}{n} + \frac{2\varsigma_2^2(h)}{n(n-1)} + \frac{4\varsigma^2(h)\log(2/\delta)}{n(\tau-\gamma)^2}\right).
		\end{align*}
		\noindent Setting $\gamma=\frac{\tau}{3}$ yields:
		\begin{align*}
			\varepsilon^2\le C_\alpha^2\left(\frac{12\varsigma_1^2(h)}{n\tau} + \frac{6\varsigma_2^2(h)}{\tau n(n-1)} + \frac{27\varsigma^2(h)\log(2/\delta)}{n\tau^3}\right).
		\end{align*}
		\noindent As a result, for $\tau<\alpha\in(0,\frac{1}{2})$, $n\ge\frac{9\log(2/\delta)}{2\tau^2}$, $K\ge \frac{\log(2/\delta)}{2(\alpha-\tau)^2}$ and $M=\lfloor \frac{2n\tau^2}{9\log(2/\delta)}\rfloor$, we have:
		\begin{align*}
			\|\widehat\Theta_\mathrm{MoIU}(h)-\Theta(h)\|_2\le C_\alpha\sqrt{\frac{12\varsigma_1^2(h)}{n\tau} + \frac{6\varsigma_2^2(h)}{\tau n(n-1)} + \frac{27\varsigma^2(h)\log(2/\delta)}{n\tau^3}},
		\end{align*}
		\noindent with probability of at least $1-\delta$, as desired.
	\end{proof}
	
	\section{Sharper Constants via Bernstein}
	\label{sec:bernstein-sharpening}
	\subsection{Bernstein Concentration of $\widehat W_n^\varepsilon$}
	\begin{lemma}[Bernstein's Lemma]
		\label{lem:bernstein-mgf-bound}
		Let $X$ be a random variable with $\E[X]=\mu$, $\mathrm{Var}(X)=\sigma^2$ and $|X-\mu|\le 1$ almost surely. Then, for all $\lambda\in\R_+$:
		\begin{align}
			\E\left[e^{\lambda(X-\mu)}\right] \le \exp\left(\sigma^2(e^\lambda-\lambda-1)\right).
		\end{align}
	\end{lemma}
	
	\begin{proof}
		Without loss of generality, assume that $\mu=0$. Let $G=\sum_{k=2}^\infty \frac{\lambda^{k-2}\E[X^k]}{k!\sigma^2}$. Expand the MGF using Taylor's series, we have:
		\begin{align*}
			\E[e^{\lambda X}]&= \E\left[1+\lambda X + \sum_{k=2}^\infty\frac{\lambda^kX^k}{k!}\right] = 1 + \lambda^2\sigma^2\E\left[\sum_{k=2}^\infty\frac{\lambda^{k-2}X^k}{k!\sigma^2}\right] \\
			&= 1 + \lambda^2\sigma^2\sum_{k=2}^\infty\frac{\lambda^{k-2}\E[X^k]}{k!\sigma^2} = 1 + \lambda^2\sigma^2 G.
		\end{align*}
		\noindent For $k\ge 2$, we have:
		\begin{align*}
			\E[X^k] = \E[X^2X^{k-2}] \le \E[X^2|X|^{k-2}] \le \E[X^2] = \sigma^2 \qquad (|X|\le 1 \text{ a.s}).
		\end{align*}
		\noindent Hence, we have:
		\begin{align*}
			G &\le \sum_{k=2}^\infty \frac{\lambda^{k-2}\sigma^2}{k!\sigma^2} = \frac{1}{\lambda^2}\left[\sum_{k=0}^\infty \frac{\lambda^{k}}{k!} - \lambda-1\right] = \frac{e^\lambda-\lambda-1}{\lambda^2}.
		\end{align*}
		\noindent As a result, we have:
		\begin{align*}
			\E[e^{\lambda X}] &= 1 + \lambda^2\sigma^2G = 1 + \sigma^2(e^\lambda-\lambda-1) \le \exp\left(\sigma^2(e^\lambda-\lambda-1)\right),
		\end{align*}
		\noindent as desired.
	\end{proof}
	
	\begin{lemma}[Bernstein's Lemma for U-Statistics]
		\label{lem:mgf-ustats-bernstein}
		Let $U_n^k(h)$ be an order $k$ U-Statistic where $\E[U_n^k(h)]=\theta$ and $0\le h(X_1,\dots,X_k)\le 1$ $\mathcal{P}^{\otimes k}$-almost-surely. Then, for all $\lambda\in\R_+$:
		\begin{align}
			\E\left[\exp\left(\lambda\left[U_n^k(h)-\theta\right]\right)\right]\le\exp\left(\frac{\lambda^2\sigma_h^2}{2n_k\left(1-\frac{\lambda}{3n_k}\right)}\right).\qquad\forall \lambda< 3n_k.
		\end{align}
		\noindent Here, $\sigma_h^2\triangleq\mathrm{Var}(h(X_1,\dots,X_k))$.
	\end{lemma}
	
	\begin{proof}
		For all $j\in[n_k]$, we denote $j_1,\dots,j_k$ as $j_\ell\triangleq (j-1)k+\ell$ for all $\ell\in[k]$. Using Giné decoupling as in the Hoeffding-type result (Lemma \ref{lem:mgf-ustats}), we have:
		\begin{align*}
			\E\left[\exp\left(\lambda\left[U_n^k(h)-\theta\right]\right)\right] &\le \prod_{j=1}^{n_k}\E\exp\left(\frac{\lambda}{n_k}\left[h(Z_{j_1},\dots,Z_{j_k})-\theta\right]\right) \\
			&\le \prod_{j=1}^{n_k}\exp\left(\sigma_h^2\left(e^{\lambda/n_k}-\frac{\lambda}{n_k}-1\right)\right) \qquad \text{(Lemma \ref{lem:bernstein-mgf-bound})}\\
			&= \exp\left(n_k\sigma_h^2\left(e^{\lambda/n_k}-\frac{\lambda}{n_k}-1\right)\right) \\
			&= \exp\left(\frac{\lambda^2\sigma_h^2}{2n_k}\cdot\frac{2n_k^2}{\lambda^2}\left(e^{\lambda/n_k}-\frac{\lambda}{n_k}-1\right)\right) \\
			&\le \exp\left(\frac{\lambda^2\sigma_h^2}{2n_k}\cdot\frac{1}{1-\frac{\lambda}{3n_k}}\right)\qquad (\lambda<3n_k).
		\end{align*}
		\noindent Here, the last inequality stems from the elementary inequality ${e^t-t-1}\le\frac{t^2/2}{1-t/3}$ for all $t<3$. Applying this inequality with $t=\frac{\lambda}{n_k}$ yields the above desired result.
	\end{proof}
	
	\begin{lemma}[Variance Bound of $\phi_k$]
		\label{lem:variance-bound-phi-k}
		Let $\phi_k:\mathcal{Z}^k\to\{0,1\}$ be the order $k$ (where $2\le k\le 2M$) kernel defined as follows:
		\begin{align}
			\phi_k(\z_1,\dots,\z_k) \triangleq \frac{1}{G_k}\sum_{\J\in\Theta_k}\lambda_\varepsilon(\widehat U_{M,\J}(h)),\qquad \widehat U_{M,\J}(h) = \frac{1}{M}\sum_{(i,j)\in\J}h(\z_i,\z_j).
		\end{align}
		\noindent Here, $G_k=|\Theta_k|$ where $\Theta_k$ is defined as follows:
		\begin{align}
			\Theta_k \triangleq \Big\{\J\in\Lambda_k^M: \eta(\J)=[k]\Big\},\qquad \Lambda_k\triangleq \Big\{(i,j): 1\le i < j\le k\Big\}.
		\end{align}
		\noindent Then, let $X_1,\dots,X_k\sim\mathcal{P}^{\otimes k}$, we have:
		\begin{align}
			\mathrm{Var}(\phi_k(X_1,\dots,X_k)) &\le \frac{1}{\varepsilon^2}\left(\frac{4\sigma_1^2(h)}{k} + \frac{2\sigma_2^2(h)}{k(k-1)} + \frac{\sigma^2(h)}{M}\right).
		\end{align}
	\end{lemma}
	
	\begin{proof}
		We will denote $\mathrm{Var}(\phi_k(X_1,\dots,X_k))=\sigma^2(\phi_k)$ for notational brevity. Let $\mathcal{U}_k$ be the uniform distribution over $\Theta_k$ and $S_k=\{X_1,\dots,X_k\}$. Then, $\phi_k(X_1,\dots,X_k)=\E_{\J\sim\mathcal{U}_k}\left[\lambda_\varepsilon(\widehat U_{M,\J}(h))\big|S_k\right]$. By the law of total variance, we have:
		\begin{align*}
			\underset{\substack{\J\sim\mathcal{U}_k\\ S_k\sim\mathcal{P}^{\otimes k}}}{\mathrm{Var}}\left[\lambda_\varepsilon(\widehat U_{M,\J}(h))\right] &= \underset{S_k\sim\mathcal{P}^{\otimes k}}{\mathrm{Var}}\left(\underset{\J\sim\mathcal{U}_k}{\E}\left[\lambda_\varepsilon(\widehat U_{M,\J}(h))\big|S_k\right]\right) + \underset{S_k\sim\mathcal{P}^{\otimes k}}{\mathbb{E}}\left[\underset{\J\sim\mathcal{U}_k}{\mathrm{Var}}\left(\lambda_\varepsilon(\widehat U_{M,\J}(h))\big|S_k\right)\right] \\
			&= \underset{S_k\sim\mathcal{P}^{\otimes k}}{\mathrm{Var}}\left(\phi_k(X_1,\dots,X_k)\right) + \underset{S_k\sim\mathcal{P}^{\otimes k}}{\mathbb{E}}\left[\underset{\J\sim\mathcal{U}_k}{\mathrm{Var}}\left(\lambda_\varepsilon(\widehat U_{M,\J}(h))\big|S_k\right)\right] \\
			&= \sigma^2(\phi_k) + \underset{S_k\sim\mathcal{P}^{\otimes k}}{\mathbb{E}}\left[\underset{\J\sim\mathcal{U}_k}{\mathrm{Var}}\left(\lambda_\varepsilon(\widehat U_{M,\J}(h))\big|S_k\right)\right]\\
			&\ge \sigma^2(\phi_k).
		\end{align*}
		\noindent Then, we can bound $\sigma^2(\phi_k)$ by bounding the variance of $\lambda_\varepsilon(\widehat U_{M,\J}(h))$. For brevity, let $U=\widehat U_{M,\J}(h)$, we have $\lambda_\varepsilon(U)\sim\mathrm{Bern}(q_{k,\varepsilon})$ where $q_{k,\varepsilon}$ is the following probability:
		\begin{align*}
			q_{k,\varepsilon} &= \underset{\substack{\J\sim\mathcal{U}_k\\ S_k\sim\mathcal{P}^{\otimes k}}}{\P}\left(\left|\widehat U_{M,\J}(h)-\theta(h)\right|\ge\varepsilon\right).
		\end{align*}
		\noindent By Chebyshev's inequality, we have:
		\begin{align*}
			q_{k,\varepsilon} &\le \frac{1}{\varepsilon^2}\underset{\substack{\J\sim\mathcal{U}_k\\ S_k\sim\mathcal{P}^{\otimes k}}}{\mathrm{Var}}\left[\widehat U_{M,\J}(h)\right] = \frac{\mathrm{Var}(U)}{\varepsilon^2}.
		\end{align*}
		\noindent By total variance, we have:
		\begin{align*}
			\mathrm{Var}(U) &= \underset{\J\sim\mathcal{U}_k}{\E}\left[\underset{S_k\sim\mathcal{P}^{\otimes k}}{\mathrm{Var}}\left(U|\J\right)\right] + \underset{\J\sim\mathcal{U}_k}{\mathrm{Var}}\left(\underset{S_k\sim\mathcal{P}^{\otimes k}}{\E}\left[U|\J\right]\right).
		\end{align*}
		\noindent Observe that:
		\begin{align*}
			\underset{S_k\sim\mathcal{P}^{\otimes k}}{\E}\left[U|\J\right] &= \frac{1}{M}\sum_{(i,j)\in\J}\E[h(X_i, X_j)|\J] = \theta(h).
		\end{align*}
		\noindent Therefore, $\underset{\J\sim\mathcal{U}_k}{\mathrm{Var}}\left(\underset{S_k\sim\mathcal{P}^{\otimes k}}{\E}\left[U|\J\right]\right)=0$ and we can bound the variance of $U$ via the conditional variance $\mathrm{Var}(U|\J)$ given a fixed length-$M$ sequence $\J\in\Theta_k$.
		
		\noindent Let $\J=\{\J_1,\dots,\J_M\}$ and $\J_*$ be the set of unique pairs in $\J$. For all $p\in\J_*$, we define $c_p$ as the multiplicity of $p$ in $\J$. Furthermore, for all $i\in[k]$, we define $d_i$ as the number of pairs in $\J$ that contains $i$. Formally:
		\begin{align}
			\forall p\in\J_*: c_p \triangleq \sum_{m=1}^M\mathds{1}_{\J_m=p},\quad \forall i\in[k]: d_i \triangleq\sum_{m=1}^M\mathds{1}_{i\in\J_m}.
		\end{align}
		\noindent Finally, for all $p=(i,j)\in\J_*$, we write $h_p=h(X_i,X_j)$ for brevity. As a result, we can write $U=\frac{1}{M}\sum_{p\in\J_*}c_ph_p$. Then, we compute $\mathrm{Var}(U|\J)$ as follows:
		\begin{align*}
			\mathrm{Var}(U|\J) &= \frac{1}{M^2}\mathrm{Var}\left[\sum_{p\in\J_*}c_ph_p\right] = \frac{1}{M^2}\left[\sigma^2(h)\sum_{p\in\J_*}c_p^2 + \sum_{p\ne q\in\J_*}c_pc_q\mathrm{Cov}(h_p, h_q)\right].
		\end{align*}
		\noindent For all $p\ne q\in\J_*$ (i.e., sharing at most one index), if $p\cap q\ne\emptyset$ then $\mathrm{Cov}(h_p,h_q)=\sigma_1^2(h)$. Otherwise, $\mathrm{Cov}(h_p,h_q)=0$ due to independence. Hence:
		\begin{align*}
			\mathrm{Var}(U|\J) &= \frac{\sigma^2(h)}{M^2}\sum_{p\in\J_*}c_p^2 + \frac{\sigma_1^2(h)}{M^2}\sum_{p\ne q\in\J_*,|p\cap q|=1}c_pc_q.
		\end{align*}
		\noindent Now, note that $d_i=\sum_{i\in p}c_p$. Therefore, we have the following identity:
		\begin{align*}
			d_i^2 &= \sum_{p\in\J_*,i\in p}\sum_{q\in\J_*,i\in q}c_pc_q = \sum_{p\in\J_*, i\in p}c_p^2 +\sum_{p\ne q\in\J_*, i\in p,q} c_pc_q.
		\end{align*}
		\noindent Therefore, we have:
		\begin{align*}
			\sum_{p\ne q\in\J_*, |p\cap q|=1}c_pc_q &= \sum_{i=1}^k\sum_{q\ne q\in\J_*, i\in p,q}c_pc_q = \sum_{i=1}^k\left(d_i^2 - \sum_{p\in\J_*,i\in p}c_p^2\right) \\
			&= \sum_{i=1}^kd_i^2 - \sum_{i=1}^k\sum_{p\in\J_*,i\in p}c_p^2 = \sum_{i=1}^kd_i^2 - \sum_{p\in\J_*}\sum_{i=1}^k\mathds{1}_{i\in p}c_p^2 \\
			&= \sum_{i=1}^kd_i^2 - 2\sum_{p\in\J^*}c_p^2.
		\end{align*}
		\noindent The last equality comes from the fact that every unique pair is counted twice when swept through the list of indices $i\in[k]$. Plug the above identity to the computation of $\mathrm{Var}(U|\J)$:
		\begin{align*}
			\mathrm{Var}(U|\J) &= \frac{\sigma^2(h)}{M^2}\sum_{p\in\J_*}c_p^2 + \frac{\sigma_1^2(h)}{M^2}\left[\sum_{i=1}^kd_i^2-2\sum_{p\in\J_*}c_p^2\right] = \frac{\sigma_1^2(h)}{M^2}\sum_{i=1}^kd_i^2 + \frac{\sigma^2(h)-2\sigma_1^2(h)}{M^2}\sum_{p\in\J_*}c_p^2 \\
			&= \frac{\sigma_1^2(h)}{M^2}\sum_{i=1}^kd_i^2+\frac{\sigma_2^2(h)}{M^2}\sum_{p\in\J_*}c_p^2.
		\end{align*}
		\noindent \textbf{(i) Bounding $\E[d_i^2]$}: We have
		\begin{align*}
		d_i^2 &= \sum_{m=1}^M\mathds{1}_{i\in\J_m} + 2\sum_{1\le p < q\le M} \mathds{1}_{i\in \J_q\cap\J_p} =  d_i + 2\sum_{1\le p < q\le M} \mathds{1}_{i\in \J_q\cap\J_p}.
		\end{align*}
		\noindent Therefore, we have:
		\begin{align*}
			\E[d_i^2] &= \E[d_i] + 2\binom{M}{2}\P_{\J\sim\mathcal{U}_k}(i\in\J_1\cap\J_2) \\
			&= \frac{2M}{k}+ 2\sum_{1\le p<q\le M}\P_{\J\sim\mathcal{U}_k}(i\in\J_q\cap\J_p) \\
			&= \frac{2M}{k}+ 2\binom{M}{2}\P_{\J\sim\mathcal{U}_k}(i\in\J_1\cap\J_2) \qquad\text{(By Symmetry)}.
		\end{align*}
		\noindent Let $\bar\J=\{\bar \J_1,\dots,\bar \J_M\}\sim\mathcal{\bar U}_k$ be i.i.d. pairs drawn from $\mathcal{\bar U}_k$, the uniform distribution over the set $\Lambda_k$ of $\binom{k}{2}$ pairs of indices in $[k]$. Let $\mathcal{E}=\Big\{\bigcup_{m=1}^M\bar\J_m = [k]\Big\}$, i.e., the event that $\bar\J_1,\dots,\bar{\J}_M$ covers all indices in $[k]$. We have $\P_{\J\sim\mathcal{U}_k}(i\in\J_1\cap\J_2)=\P_{\bar\J\sim\mathcal{\bar U}^{\otimes M}_k}(i\in\bar\J_1\cap\bar\J_2|\mathcal{E})$. By Bayes rule:
		\begin{align*}
			\P(i\in\bar\J_1\cap\bar\J_2|\mathcal{E}) &= \P(i\in\bar\J_1\cap\bar \J_2)\cdot \frac{\P(\mathcal{E}|i\in\bar\J_1\cap\bar \J_2)}{\P(\mathcal{E})} \le \P(i\in\bar\J_1\cap\bar\J_2).
		\end{align*}
		\noindent All probabilities above are taken over the i.i.d. draw over $\mathcal{\bar U}_k^{\otimes M}$. We have $\P(\mathcal{E}|i\in\bar\J_1\cap\bar\J_2)\le\P(\mathcal{E})$ (due to Lemma \ref{lem:stochastic-dominance-lemma}) and $\P(i\in\bar\J_1\cap\bar\J_2)=\P(i\in\bar\J_1)\P(i\in\bar\J_2)$ due to independence of $\bar\J_1,\bar\J_2$. For all $i\in[k]$, we have $\P(i\in\bar\J_1)=\frac{k-1}{\binom{k}{2}}=\frac{2}{k}$. Hence:
		\begin{align*}
			\E[d_i^2] &\le \frac{2M}{k} + 2\binom{M}{2}\P(i\in\bar\J_1)\P(i\in\bar\J_2) = \frac{2M}{k} + \frac{4M(M-1)}{k^2} \\
			&\le \frac{2M}{k} + \frac{4M^2}{k^2}.
		\end{align*}
		\noindent\textbf{(ii) Bounding $\E[c_p^2]$}: We have
		\begin{align*}
			\sum_{p\in\J_*}c_p^2 &= \sum_{p\in\J_*}\left(\sum_{m=1}^M\mathds{1}_{\J_m=p}\right)^2 = \sum_{p\in\J_*}\sum_{m=1}^M\mathds{1}_{\J_m=p} + 2\sum_{1\le u<v\le M}\sum_{p\in\J_*}\mathds{1}_{\J_u=\J_v}\cdot\mathds{1}_{\J_u=p} \\
			&= \sum_{p\in\J_*}c_p + 2\sum_{1\le u<v\le M}\mathds{1}_{\J_u=\J_v} = M + 2\sum_{1\le u<v\le M}\mathds{1}_{\J_u=\J_v}.
		\end{align*}
		\noindent Taking the expectation over $\mathcal{U}_k$, we have:
		\begin{align*}
			\E\left[\sum_{p\in\J_*}c_p^2\right] &= M + 2\binom{M}{2}\P_{\J\sim\mathcal{U}_k}(\J_1=\J_2).
		\end{align*}
		\noindent Let $\bar\J=\{\J_1,\dots,\bar\J_M\}\sim\mathcal{\bar U}_k$ be i.i.d. draws from $\mathcal{\bar U}_k$, the uniform distribution over $\Lambda_k$ and $\mathcal{E}$ be the event that $\bar\J$ covers all indices in $[k]$. Again, we have $\P_{\J\sim\mathcal{U}_k}(\J_1=\J_2)=\P_{\bar\J\sim\mathcal{\bar U}_k^{\otimes M}}(\bar\J_1=\bar\J_2|\mathcal{E})$. By Bayes rule, we have:
		\begin{align*}
			\P(\bar\J_1=\bar\J_2|\mathcal{E}) = \P(\bar\J_1=\bar\J_2) \cdot\frac{\P(\mathcal{E}|\bar\J_1=\bar\J_2)}{\P(\mathcal{E})}\le \P(\bar\J_1=\bar\J_2),
		\end{align*}
		\noindent where $\P(\mathcal{E}|\bar\J_1=\bar\J_2)\le \P(\mathcal{E})$ stems from Lemma \ref{lem:stochastic-dominance-second-lemma}. Hence, we have:
		\begin{align*}
			\E\left[\sum_{p\in\J_*}c_p^2\right] &\le M + 2\binom{M}{2}\P_{\bar\J_1,\bar\J_2\sim\mathcal{\bar U}^{\otimes 2}_k}(\J_1=\J_2) = M + 2\binom{M}{2}\Big/\binom{k}{2}.
		\end{align*}
		
		\noindent \textbf{Final Variance Bound}: Combining the bounds on $\E[d_i^2]$ and $\E[c_p^2],\forall p\in\J_*$, we have
		\begin{align*}
			\E_{\J\sim\mathcal{U}_k}[\mathrm{Var}(U|\J)] &\le \frac{\sigma_2^2(h)}{M^2}\E_{\J\sim\mathcal{U}_k}\left[\sum_{p\in\J_*}c_p^2\right] + \frac{\sigma_1^2(h)}{M^2}\sum_{i=1}^k\E_{\J\sim\mathcal{U}_k}[d_i^2] \\
			&\le \frac{\sigma_2^2(h)}{M^2}\left[M+\frac{2M(M-1)}{k(k-1)}\right] + \frac{\sigma_1^2(h)}{M^2}\left[2M+\frac{4M^2}{k}\right] \\
			&\le \frac{\sigma_2^2(h)}{M} + \frac{2\sigma^2_2(h)}{k(k-1)} + \frac{2\sigma_1^2(h)}{M} + \frac{4\sigma_1^2(h)}{k} \\
			&\le \frac{4\sigma_1^2(h)}{k} + \frac{2\sigma_2^2(h)}{k(k-1)} + \frac{\sigma^2(h)}{M} \qquad (\sigma^2(h)=2\sigma_1^2(h)+\sigma_2^2(h)).
		\end{align*}
		\noindent Finally, recall that $\mathrm{Var}\left[\lambda_\varepsilon(U)\right]=q_{k, \varepsilon}(1-q_{k,\varepsilon})\le q_{k,\varepsilon}$. Therefore, we have:
		\begin{align*}
			\sigma^2(\phi_k) &\le \underset{\substack{\J\sim\mathcal{U}_k\\ S_k\sim\mathcal{P}^{\otimes k}}}{\mathrm{Var}}\left[\lambda_\varepsilon(U)\right] \le q_{k,\varepsilon} \le \frac{1}{\varepsilon^2}\mathrm{Var}(U) = \frac{1}{\varepsilon^2}\E_{\J\sim\mathcal{U}_k}[\mathrm{Var}(U|\J)] \\
			&\le \frac{4\sigma_1^2(h)}{k\varepsilon^2} + \frac{2\sigma_2^2(h)}{k(k-1)\varepsilon^2} + \frac{\sigma^2(h)}{M\varepsilon^2},
		\end{align*}
		\noindent as desired.
	\end{proof}
	
	\begin{lemma}
		\label{lem:stochastic-dominance-lemma}
		Let $\bar\J=\{\bar\J_1,\dots,\bar\J_M\}\sim\mathcal{\bar U}_k$ be i.i.d. pairs drawn from $\mathcal{\bar U}_k$, the uniform distribution over $\Lambda_k=\Big\{(i,j):1\le i<j\le k\Big\}$. Let $\mathcal{E}=\Big\{\bigcup_{m=1}^M\bar\J_m=[k]\Big\}$, i.e., the event that $\J_1,\dots,\J_M$ covers all indices in $[k]$. Then, for all $i\in[k]$, we have:
		\begin{align}
			\P(\mathcal{E}|i\in\bar\J_1\cap\bar\J_2) \le\P(\mathcal{E}).
		\end{align}
		\noindent The above inequality also holds when $\bar\J_1,\dots,\bar\J_M$ are \textbf{selected without replacement} from $\Lambda_k$.
	\end{lemma}
	
	\begin{proof}
		Let $V=|\bar\J_1\cup\bar\J_2|$ (a random variable) be the number of indices covered by the first two draws. Let $g(v)$ be defined as follows:
		\begin{align}
			g(v) \triangleq \P(\mathcal{E}|V=v).
		\end{align}
		\noindent Then, $g$ is \textbf{strictly increasing} because the more indices covered in the first two draws, the higher the probability of covering all $k$ indices. Let $U_i=\big\{i\in\bar\J_1\cap\bar\J_2\big\}$, we have:
		\begin{align}
			\P(\mathcal{E}|U_i) &= \E_V[h(V)|U_i],\qquad\P(\mathcal{E}|U_i^c) =\E_V[h(V)|U_i^c].
		\end{align}
		\noindent We first prove that $\P(\mathcal{E}|U_i^c)\ge\P(\mathcal{E}|U_i)$. Since $h$ is strictly increasing, notice that if $V|U_i^c\succeq_\mathrm{st}V|U_i$ (stochastic dominance), we have $\E[h(V)|U_i^c]\ge\E[h(V)|U_i]$.
		
		\textbf{(i) Under event $U_i$}: We know that $\P(V\le 1|U_i) = 0$ and $\P(V\le v|U_i)=1$ for all $v\ge 3$. Therefore, the interesting case is $\P(V\le 2|U_i)$. Let $\bar\J_1=(i, X), \bar\J_2=(i, Y)$ where $X,Y$ are drawn uniformly from $[k]\setminus\{i\}$. Then, we have:
		\begin{align*}
			\P(V\le 2|U_i)=\P(V=2|U_i) &= \P(X=Y) = \frac{1}{k-1}.
		\end{align*}
		\noindent Therefore, we have: 
		\begin{align}
			\label{eq:PV_given_Ui}
			\P(V\le v|U_i) &= \begin{cases}
				\frac{1}{k-1} &\quad\text{for } v = 2, \\
				1 &\quad\text{for all } v\ge3.
			\end{cases}
		\end{align}
		\noindent\textbf{(ii) Under event $U_i^c$}: We also have $\P(V\le1|U_i^c)=0$ and for $v\ge3$ we have $\P(V\le v|U_i^c)\le 1$ (Since $\P(V\le 3|U_i^c)=1-\P(V=4|U_i^c)<1$, $\P(V\le v|U_i^c)=1,\forall v\ge 4$). 
		
		\noindent To compute $\P(V=2|U_i^c)$, we need to compute $\P(V=2,U_i^c)$ and $\P(U_i^c)$.
		\begin{enumerate}[label=(\roman*)]
			\item To compute $\P(U_i^c)$, there are $\binom{k}{2}^2$ possible ways of choosing $\bar\J_1,\bar\J_2$ among which there are $(k-1)^2$ ways to choose two pairs both containing $i$. Hence, we have $\P(U_i^c)=\frac{\binom{k}{2}^2-(k-1)^2}{\binom{k}{2}^2}$.
			
			\item To compute $\P(V=2,U_i^c)$, this is the probability of choosing $\bar\J_1=\bar\J_2$, both not containing $i$. There are $\binom{k-1}{2}$ such pairs. Hence, $\P(V=2,U_i^c)=\frac{\binom{k-1}{2}}{\binom{k}{2}^2}$.
		\end{enumerate}
		\noindent Therefore, we have:
		\begin{align}
			\P(V=2|U_i^c) &= \frac{\P(V=2,U_i^c)}{\P(U_i^c)} = \frac{\binom{k-1}{2}}{\binom{k}{2}^2-(k-1)^2} = \frac{2}{(k-1)(k+2)}.
		\end{align}
		\noindent As a result, we have:
		\begin{align}
			\label{eq:PV_given_Uic}
			\begin{cases}
				\P(V\le 2|U_i^c) &= \frac{2}{(k-1)(k+2)}, \\
				\P(V\le v|U_i^c) &\le 1 \qquad\qquad\qquad\forall v\ge 3.
			\end{cases}
		\end{align}
		\noindent From Eqn.~\eqref{eq:PV_given_Ui} and Eqn.~\eqref{eq:PV_given_Uic}, we have $\P(V\le v|U_i)\ge\P(V\le v|U_i^c)$ for all $v\ge 2$ ($\frac{1}{k-1}\ge\frac{2}{(k-1)(k+2)}$ for all $k\ge2$). Therefore, we have $V|U_i^c\succeq_\mathrm{st}V|U_i$. Therefore, we have $\E_V[h(V)|U_i^c]\ge\E_V[h(V)|U_i]$ or $\P(\mathcal{E}|U_i^c)\ge\P(\mathcal{E}|U_i)$ as initially desired. Therefore, by Bayes' rule:
		\begin{align*}
			\P(\mathcal{E}) - \P(\mathcal{E}|U_i) &=  \P(\mathcal{E}|U_i^c)\P(U_i^c) + \P(\mathcal{E}|U_i)\P(U_i) - \P(\mathcal{E}|U_i) \\
			&= \P(\mathcal{E}|U_i^c)\P(U_i^c) - \P(\mathcal{E}|U_i)[1-\P(U_i)] \\		
			&= \P(\mathcal{E}|U_i^c)\P(U_i^c) - \P(\mathcal{E}|U_i)\P(U_i^c) \\
			&= [\P(\mathcal{E}|U_i^c) - \P(\mathcal{E}|U_i)]\P(U_i^c) \\
			&\ge 0.
		\end{align*}
		\noindent Therefore, we have $\P(\mathcal{E})\ge\P(\mathcal{E}|U_i)$, as desired.
		
		\noindent\textbf{When $\bar\J_1,\dots,\bar\J_M$ are selected without replacement from $\Lambda_k$}: Proving that $V|U_i^c\succeq_\mathrm{st}V|U_i$ is vastly more trivial in this case. For $v\le 2$, we have $\P(V\le v|U_i^c)=\P(V\le v|U_i)=0$. However, we have $\P(V=3|U_i)=1$ while $\P(V=3|U_i^c)<1$. Therefore, we have $\P(V\le v|U_i)\ge\P(V\le v|U_i^c)$ for all $v\ge 1$, making the desired stochastic dominance hold.
	\end{proof}
	
	\begin{lemma}
		\label{lem:stochastic-dominance-second-lemma}
		Let $\bar\J=\{\bar\J_1,\dots,\bar\J_M\}\sim\mathcal{\bar U}_k$ be i.i.d. pairs drawn from $\mathcal{\bar U}_k$, the uniform distribution over $\Lambda_k=\Big\{(i,j):1\le i<j\le k\Big\}$. Let $\mathcal{E}=\Big\{\bigcup_{m=1}^M\bar\J_m=[k]\Big\}$, i.e., the event that $\J_1,\dots,\J_M$ covers all indices in $[k]$. Then, we have:
		\begin{align}
			\P(\mathcal{E}|\bar\J_1=\bar\J_2) \le\P(\mathcal{E}).
		\end{align}
		\noindent The above inequality also holds when $\bar\J_1,\dots,\bar\J_M$ are \textbf{selected without replacement} from $\Lambda_k$.
	\end{lemma}
	
	\begin{proof}
		Define the random variable $V=|\bar\J_1\cup\bar\J_2|$ and the event $W=\{\bar\J_1=\bar\J_2\}$. We prove the above result using stochastic dominance as in Lemma \ref{lem:stochastic-dominance-lemma}. We have:
		\begin{align*}
			\P(V\le v|W) &= \begin{cases}
				1 &\quad\text{for all } v\ge2,\\
				0 &\quad\text{otherwise}.
			\end{cases}\qquad\P(V\le v|W^c) = \begin{cases}
				1 &\quad\text{for all } v\ge 3, \\
				0&\quad\text{otherwise}.
			\end{cases}
		\end{align*}
		\noindent As a result, we have $\P(V\le v|W)\ge\P(V\le v|W^c)$ for all $v\in\mathbb{N}$. Hence, we have $V|W^c\succeq_\mathrm{st}V|W$. Using the same calculation as in Lemma \ref{lem:stochastic-dominance-lemma}, we have $\P(\mathcal{E})\ge\P(\mathcal{E}|\bar\J_1=\bar\J_2)$, as desired.
	\end{proof}
	
	\begin{remark}[On Stochastic Dominance]
		A random variable $X$ is said to stochastically dominate $Y$ (both supported on $\mathcal{Z}\subseteq\R$), denoted $X\succeq_\mathrm{st}Y$, if:
		\begin{align}
			\forall t\in\mathcal{Z}: F_X(t) \le F_Y(t)\qquad \text{or}\qquad \P(X\ge t)\ge \P(Y\ge t).
		\end{align}
		\noindent In that case, for all $g:\mathcal{Z}\to\R$ increasing, we have $\E[g(X)]\ge\E[g(Y)]$.
	\end{remark}
	
	\begin{proposition}[Bernstein Analogue of Proposition \ref{prop:ugly-vstats-bound} - Restated]
		\label{prop:ugly-vstats-bernstein-bound-appendix}
		Let $\widehat W_n^\varepsilon$ be the V-Statistic defined in Eqn.~\eqref{eq:vstats-ugly} and $\widehat q_\varepsilon=\E[\widehat W_n^\varepsilon]$. Suppose that $n\ge 2M$. Then, for all $t>0$, we have:
		\begin{align}
			\boxed{\P(\widehat W_n^\varepsilon-\widehat q_\varepsilon\ge t)\le \exp\left(-\frac{nt^2}{8M\left(\widehat\sigma_{M,\varepsilon}^2 + \frac{t}{3}\right)}\right)},\quad\text{where }\widehat\sigma_{M,\varepsilon}^2=\frac{2\sigma^2(h)}{M\varepsilon^2}.
		\end{align}
	\end{proposition}
	
	\begin{proof}
		Let $\mathrm{Var}(\phi_k(X_1,\dots,X_k))=\sigma^2(\phi_k)$ for all $2\le k\le 2M$ for brevity. Then, using the same convex combination argument as in Proposition \ref{prop:ugly-vstats-bound}, for all $\lambda<3\lfloor n/2M\rfloor$, we have:
		\begin{align*}
			\P(\widehat W_n^\varepsilon-\widehat q_\varepsilon\ge t) 
			&\le e^{-t\lambda}\sum_{k=2}^{2M}\alpha_{n,k}\E\left[\exp\left(\lambda\Big[U_{n,k}(\phi_k)-\E U_{n,k}(\phi_k)\Big]\right)\right]\quad\text{(Lemma \ref{lem:pitcan})} \\
			&\le e^{-t\lambda}\sum_{k=2}^{2M}\alpha_{n,k}\exp\left(\frac{\lambda^2\sigma^2(\phi_k)}{2n_k\left(1-\frac{\lambda}{3n_k}\right)}\right) \qquad \text{(Lemma \ref{lem:mgf-ustats-bernstein}, $n_k=\lfloor n/k\rfloor$)} \\
			&\le \max_{2\le k\le 2M}e^{-t\lambda}\exp\left(\frac{\lambda^2}{2n_k - \frac{2\lambda}{3}}\left[\frac{4\sigma_1^2(h)}{k\varepsilon^2} + \frac{2\sigma_2^2(h)}{k(k-1)\varepsilon^2} + \frac{\sigma^2(h)}{M\varepsilon^2}\right]\right) \qquad\quad\text{(Lemma \ref{lem:variance-bound-phi-k})}\\
			&= \max_{2\le k\le 2M}e^{-t\lambda}\exp\left(\frac{\lambda^2}{\frac{2kn_k}{M} - \frac{2k\lambda}{3M}}\left[\frac{4\sigma_1^2(h)}{M\varepsilon^2}+\frac{2\sigma_2^2(h)}{M(k-1)\varepsilon^2}+\frac{k\sigma^2(h)}{M^2\varepsilon^2}\right]\right).
		\end{align*}
		\noindent Since $k\ge2, kn_k\ge\frac{n}{2}$ and $\frac{k}{M}\le 2$, we have:
		\begin{align*}
			\P(\widehat W_n^\varepsilon-\widehat q_\varepsilon\ge t) &\le \max_{2\le k\le 2M}e^{-t\lambda}\exp\left(\frac{\lambda^2}{n/M - \frac{4\lambda}{3}}\left[\frac{4\sigma_1^2(h)}{M\varepsilon^2}+\frac{2\sigma_2^2(h)}{M\varepsilon^2}+\frac{k\sigma^2(h)}{M^2\varepsilon^2}\right]\right) \\
			&\le \max_{2\le k\le 2M}e^{-t\lambda}\exp\left(\frac{\lambda^2}{n/M - \frac{4\lambda}{3}}\left[\frac{4\sigma_1^2(h)+2\sigma_2^2(h)}{M\varepsilon^2}+\frac{2\sigma^2(h)}{M\varepsilon^2}\right]\right) \\
			&\le e^{-t\lambda}\exp\left(\frac{\lambda^2}{n/2M - \frac{2\lambda}{3}}\left[\frac{2\sigma_1^2(h) + \sigma_2^2(h)}{M\varepsilon^2}+\frac{\sigma^2(h)}{M\varepsilon^2}\right]\right) \\
			&= e^{-t\lambda}\exp\left(\frac{\lambda^2}{n/2M - \frac{2\lambda}{3}}\cdot\frac{2\sigma^2(h)}{M\varepsilon^2}\right)\qquad(\sigma^2(h)=2\sigma_1^2(h) + \sigma_2^2(h)) \\
			&= e^{-t\lambda}\exp\left(\frac{\lambda^2\widehat\sigma_{M,\varepsilon}^2}{n/2M-\frac{2\lambda}{3}}\right)
		\end{align*}
		\noindent Now, let us optimize $f(\lambda)=-t\lambda+\frac{\lambda^2\widehat\sigma_{M,\varepsilon}^2}{n/2M-\frac{2\lambda}{3}}$ w.r.t. $\lambda<\frac{3n}{4M}$. We can write:
		\begin{align*}
			f(\lambda) &= -t\lambda + \frac{b\lambda^2}{1-c\lambda},\qquad b = \frac{2M\widehat\sigma_{M,\varepsilon}^2}{n} \text{ and } c = \frac{4M}{3n}.
		\end{align*}
		\noindent Differentiate $f(\lambda)$ to find the stationary point $f'(\lambda_*)=0$, we have:
		\begin{align*}
			f'(\lambda_*) &= -t + \frac{b\lambda_*(2-c\lambda_*)}{(1-c\lambda_*)^2} = 0 \implies \frac{c\lambda_*(2-c\lambda_*)}{(1-c\lambda_*)^2} = \frac{ct}{b}.
		\end{align*}
		\noindent Substituting $u=c\lambda_*$ to the above equation, we have:
		\begin{align*}
			\frac{u(2-u)}{(1-u)^2} = \frac{ct}{b} \implies \frac{1}{(1-u)^2}-1=\frac{ct}{b} \implies u = 1-\frac{1}{\sqrt{1+ct/b}}. 
		\end{align*}
		\noindent Therefore, we have:
		\begin{align*}
			\lambda_* = \frac{1}{c}\left(1-\frac{1}{\sqrt{1+ct/b}}\right) = \frac{3n}{4M}\left(1-\frac{1}{\sqrt{1+\frac{2t}{3\widehat\sigma_{M,\varepsilon}^2}}}\right).
		\end{align*}
		\noindent This satisfies the domain condition that $\lambda<\frac{3n}{4M}$. As a result:
		\begin{align*}
			\inf_{0<\lambda<\frac{3n}{4M}}f(\lambda) &= f(\lambda_*) = -\frac{nt^2}{2M\left(\widehat\sigma_{M,\varepsilon}+\sqrt{\widehat\sigma_{M,\varepsilon}^2+\frac{2t}{3}}\right)^2}.
		\end{align*}
		\noindent Finally, plugging the above back to the tail bound, we have:
		\begin{align*}
			\P(\widehat W_n^\varepsilon-\widehat q_\varepsilon\ge t) &\le \exp\left(-\frac{nt^2}{2M\left(\widehat\sigma_{M,\varepsilon}+\sqrt{\widehat\sigma_{M,\varepsilon}^2+\frac{2t}{3}}\right)^2}\right) \le \exp\left(-\frac{nt^2}{8M\left(\widehat\sigma_{M,\varepsilon}^2+\frac{t}{3}\right)}\right).
		\end{align*}
		\noindent The last inequality is due to $\left(\widehat\sigma_{M,\varepsilon}+\sqrt{\widehat\sigma_{M,\varepsilon}^2+\frac{2t}{3}}\right)^2\le 4\widehat\sigma_{M,\varepsilon}^2+\frac{4t}{3}=4(\widehat\sigma_{M,\varepsilon}^2+\frac{t}{3})$.
	\end{proof}
	
	\subsection{Bernstein Concentration of $\widetilde W_n^\varepsilon$}
	\begin{proposition}[Variance Bound of $\varphi_k$]
		Let $\varphi_k:\mathcal{Z}^k\to\{0,1\}$ be the order $k$ ($2\le k\le 2M$) kernel defined as follows:
		\begin{align}
			\varphi_k(\z_1,\dots,\z_k) \triangleq \frac{1}{F_k}\sum_{\J\in\Xi_k}\lambda_\varepsilon(\widetilde U_{M, \J}(h)),\quad\widetilde U_{M,\J}(h)=\frac{1}{M}\sum_{(i,j)\in\J}h(\z_i,\z_j).
		\end{align}
		\noindent Here, $F_k=|\Xi_k|$ where $\Xi_k$ is defined as follows:
		\begin{align}
			\Xi_k\triangleq\Big\{\J\in\Gamma_k^M:\eta(\J)=[k]\Big\},\quad\Gamma_k^M \triangleq \Big\{{\bf J}=(\J_m)_{m=1}^M\in\Lambda_k^M: \J_u\ne\J_v,\forall u< v\le M\Big\}.
		\end{align}
		\noindent Then, let $X_1,\dots,X_k\sim\mathcal{P}^{\otimes k}$, we have:
		\begin{align}
			\mathrm{Var}(\varphi_k(X_1,\dots,X_k)) &\le \frac{\sigma^2(h)}{M\varepsilon^2} + \frac{4\sigma_1^2(h)}{k\varepsilon^2}.
		\end{align}
	\end{proposition}
	
	\begin{proof}
		We will denote $\mathrm{Var}(\varphi_k(X_1,\dots,X_k))=\sigma^2(\varphi_k)$. Let $\mathcal{V}_k$ be the uniform distribution over set of sequences $\Xi_k$. Then,  $\varphi_k(X_1,\dots,X_k)=\E_{\J\sim\mathcal{V}_k}\left[\lambda_\varepsilon(\widetilde U_{M,\J}(h))\big|S_k\right]$ where $S_k=\{X_1,\dots,X_k\}$. Let $U=\widetilde U_{M,\J}(h)$ for brevity and $q_{k,\varepsilon}=\P_{\J\sim\mathcal{V}_k,S_k}(|U-\theta(h)|\ge\varepsilon)$. Following the same arguments as Lemma \ref{lem:variance-bound-phi-k}, we have:
		\begin{align*}
			\sigma^2(\varphi_k) &\le \underset{\substack{\J\sim\mathcal{V}_k\\S_k\sim\mathcal{P}^{\otimes k}}}{\mathrm{Var}}[\lambda_\varepsilon(U)] = q_{k,\varepsilon}(1-q_{k,\varepsilon}) \le q_{k,\varepsilon} \\
			&\le \frac{1}{\varepsilon^2}\mathrm{Var}(U) = \frac{1}{\varepsilon^2}\E_{\J\sim\mathcal{V}_k}[\mathrm{Var}(U|\J)]\qquad\text{(Chebyshev's Inequality)}.
		\end{align*}
		\noindent Following the same calculation of $\mathrm{Var}(U|\J)$, we have:
		\begin{align*}
			\mathrm{Var}(U|\J) &\le \frac{\sigma_1^2(h)}{M^2}\sum_{i=1}^kd_i^2 + \frac{\sigma_2^2(h)}{M^2}\sum_{p\in\J_*}c_p^2,
		\end{align*}
		\noindent where $\J_*$ is the set of unique pairs in $\J$, for $p\in\J_*$, $c_p$ denotes the multiplicity of pair $p$ and $d_i=\sharp\{(j,\ell)\in\J:i\in(j,\ell)\}$. Note that due to pairs selection without replacement, all pairs in $\J$ are unique ($\J_*=\J$) and $c_p=1$ for all $p\in\J_*$. Hence, we have:
		\begin{align*}
			\E_{\J\sim\mathcal{V}_k}[\mathrm{Var}(U|\J)] &\le \frac{\sigma^2_2(h)}{M} + \frac{\sigma_1^2(h)}{M^2}\sum_{i=1}^k\E_{\J\sim\mathcal{V}_k}[d_i^2] \quad \Big(\sum_{p\in\J_*}c_p^2=M\Big).
		\end{align*}
		\noindent Similar to Lemma \ref{lem:variance-bound-phi-k}, we have:
		\begin{align*}
			\E[d_i^2] &= \E[d_i] + 2\binom{M}{2}\P_{\J\sim\mathcal{V}_k}(i\in\J_1\cap\J_2) = \frac{2M}{k} + 2\binom{M}{2}\P_{\J\sim\mathcal{V}_k}(i\in\J_1\cap\J_2).
		\end{align*}
		\noindent Let $\bar\J=\{\bar\J_1,\dots,\bar\J_M\}\sim\mathcal{\bar V}_k$ where $\mathcal{\bar V}_k$ is the distribution of sequences of $M$ pairs selected \textbf{without replacement} from $\Lambda_k$. Let $\mathcal{E}=\Big\{\bigcup_{m=1}^M\bar\J_m=[k]\Big\}$, the event that $\bar\J$ covers all $k$ indices. Then, we have $\P_{\J\sim\mathcal{V}_k}(i\in\J_1\cap\J_2)=\P_{\bar\J\sim\mathcal{\bar V}_k}(i\in\bar\J_1\cap\bar\J_2|\mathcal{E})$. By Bayes rule:
		\begin{align*}
			\P(i\in\bar\J_1\cap\bar\J_2|\mathcal{E}) &= \P(i\in\bar\J_1\cap\bar\J_2)\cdot\frac{\P(\mathcal{E}|i\in\bar\J_1\cap\bar\J_2)}{\P(\mathcal{E})}\le\P(i\in\bar\J_1\cap\bar\J_2).
		\end{align*}
		\noindent All probabilities are taken over $\bar\J\sim\mathcal{\bar V}_k$, i.e., $\bar\J$ selected without replacement from $\Lambda_k$. The inequality $\P(\mathcal{E}|i\in\bar\J_1\cap\bar\J_2)\le\P(\mathcal{E})$ comes from Lemma \ref{lem:stochastic-dominance-lemma} for pairs selection without replacement. 
		
		\noindent To calculate $\P(i\in\bar\J_1\cap\bar\J_2)$, we need to count the number of ways that $\bar\J_1,\bar\J_2$ selected without replacement from $\Lambda_k$ contains a given index $i$. This is equivalent to the number of ways that we can choose $2$ other indices from $k-1$ remaining indices excluding $i$. Hence, we have:
		\begin{align*}
			\P(i\in\bar\J_1\cap\bar\J_2) &= \frac{\binom{k-1}{2}}{\binom{\frac{1}{2}k(k-1)}{2}} = \frac{(k-1)(k-2)}{\frac{1}{2}k(k-1)\left[\frac{1}{2}k(k-1)-1\right]} \\
			&= \frac{4(k-1)(k-2)}{k(k-1)[k(k-1)-2]} \\
			&= \frac{4}{k(k+1)} \le\frac{4}{k^2}.
		\end{align*}
		\noindent As a result, we have:
		\begin{align*}
			\E_{\J\sim\mathcal{V}_k}[\mathrm{Var}(U|\J)] &\le \frac{\sigma^2_2(h)}{M} + \frac{\sigma_1^2(h)}{M^2}\sum_{i=1}^k\E_{\J\sim\mathcal{V}_k}[d_i^2] \\
			&\le \frac{\sigma^2_2(h)}{M} + \frac{\sigma_1^2(h)}{M^2}\left[2M + 2k\binom{M}{2}\P_{\J\sim\mathcal{V}_k}(i\in\J_1\cap\J_2)\right] \\
			&\le  \frac{\sigma^2_2(h)}{M} + \frac{\sigma_1^2(h)}{M^2}\left[2M + \frac{4M(M-1)}{k}\right] \\
			&\le \frac{\sigma^2_2(h)}{M} + \frac{2\sigma_1^2(h)}{M} + \frac{4\sigma_1^2(h)}{k} \\
			&\le \frac{\sigma^2(h)}{M} + \frac{4\sigma_1^2(h)}{k}.
		\end{align*}
		\noindent Therefore, we have:
		\begin{align*}
			\sigma^2(\varphi_k) &\le\frac{1}{\varepsilon^2}\E_{\J\sim\mathcal{V}_k}[\mathrm{Var}(U|\J)] \le \frac{\sigma^2(h)}{M\varepsilon^2} + \frac{4\sigma_1^2(h)}{k\varepsilon^2},
		\end{align*}
		\noindent as desired. As a result, we have a similar bound to Proposition \ref{prop:ugly-vstats-bernstein-bound-appendix} for $\widetilde W_n^\varepsilon$ below.
	\end{proof}
	
	\begin{proposition}[Bernstein Concentration of $\widetilde W_n^\varepsilon$]
		\label{prop:bernstein-bound-for-swor}
		Let $\widetilde W_n^\varepsilon$ be the V-Statistic defined in Eqn.~\eqref{eq:tilde-w-n-eps} and $\widetilde q_\varepsilon=\E[\widetilde W_n^\varepsilon]$. Suppose that $n\ge 2M$. Then, for all $t>0$, we have:
		\begin{align}
			\boxed{\P(\widetilde W_n^\varepsilon-\widetilde q_\varepsilon\ge t)\le \exp\left(-\frac{nt^2}{8M\left(\widetilde\sigma_{M,\varepsilon}^2 + \frac{t}{3}\right)}\right)},\quad\widetilde\sigma_{M,\varepsilon}^2 = \frac{\sigma^2(h)+2\sigma_1^2(h)}{M\varepsilon^2}.
		\end{align}
	\end{proposition}
	
	\begin{proof}
		Let $\mathrm{Var}(\varphi_k(X_1,\dots,X_k))=\sigma^2(\varphi_k)$ for $K_{\min}\le k\le 2M$ where $K_{\min}=\min\Big\{k:\binom{k}{2}\ge M\Big\}$. Using the same decomposition as in the proof of Theorem \ref{thm:univariate-moiu-swor-blockwise}, for all $\lambda<3\lfloor n/2M\rfloor$ we have:
			\begin{align*}
			\P(\widetilde W_n^\varepsilon-\widetilde q_\varepsilon\ge t) 
			&\le e^{-t\lambda}\sum_{k=K_{\min}}^{2M}\beta_{n,k}\E\left[\exp\left(\lambda\Big[U_{n,k}(\varphi_k)-\E U_{n,k}(\varphi_k)\Big]\right)\right]\quad\text{(Lemma \ref{lem:pitcan})} \\
			&\le e^{-t\lambda}\sum_{k=K_{\min}}^{2M}\beta_{n,k}\exp\left(\frac{\lambda^2\sigma^2(\varphi_k)}{2n_k\left(1-\frac{\lambda}{3n_k}\right)}\right) \qquad \text{(Lemma \ref{lem:mgf-ustats-bernstein}, $n_k=\lfloor n/k\rfloor$)} \\
			&\le \max_{K_{\min}\le k\le 2M}e^{-t\lambda}\exp\left(\frac{\lambda^2}{2n_k - \frac{2\lambda}{3}}\left[\frac{4\sigma_1^2(h)}{k\varepsilon^2} +  \frac{\sigma^2(h)}{M\varepsilon^2}\right]\right) \qquad\text{(Lemma \ref{lem:variance-bound-phi-k})}\\
			&= \max_{K_{\min}\le k\le 2M}e^{-t\lambda}\exp\left(\frac{\lambda^2}{\frac{2kn_k}{M} - \frac{2k\lambda}{3M}}\left[\frac{4\sigma_1^2(h)}{M\varepsilon^2}+\frac{k\sigma^2(h)}{M^2\varepsilon^2}\right]\right).
		\end{align*}
		\noindent Since $k\ge K_{\min}\ge 2$, $kn_k\ge\frac{n}{2}$ and $\frac{k}{M}\le 2$, we have:
		\begin{align*}
			\P(\widehat W_n^\varepsilon-\widehat q_\varepsilon\ge t) &\le \max_{K_{\min}\le k\le 2M}e^{-t\lambda}\exp\left(\frac{\lambda^2}{n/M - \frac{4\lambda}{3}}\left[\frac{4\sigma_1^2(h)}{M\varepsilon^2}+\frac{k\sigma^2(h)}{M^2\varepsilon^2}\right]\right) \\
			&\le \max_{K_{\min}\le k\le 2M}e^{-t\lambda}\exp\left(\frac{\lambda^2}{n/M - \frac{4\lambda}{3}}\left[\frac{4\sigma_1^2(h)}{M\varepsilon^2}+\frac{2\sigma^2(h)}{M\varepsilon^2}\right]\right) \\
			&\le e^{-t\lambda}\exp\left(\frac{\lambda^2}{n/2M - \frac{2\lambda}{3}}\left[\frac{2\sigma_1^2(h)}{M\varepsilon^2}+\frac{\sigma^2(h)}{M\varepsilon^2}\right]\right) \\
			&= e^{-t\lambda}\exp\left(\frac{\lambda^2\widetilde\sigma_{M,\varepsilon}^2}{n/2M-\frac{2\lambda}{3}}\right)
		\end{align*}
		\noindent Following the same optimization step in Proposition \ref{prop:ugly-vstats-bernstein-bound-appendix}, we obtain the desired inequality.
	\end{proof}
	
	\begin{remark}
				\noindent Here, note that $\widetilde\sigma_{M,\varepsilon}^2=\frac{1}{M\varepsilon^2}[\sigma^2(h)+2\sigma_1^2(h)]\le\frac{2\sigma^2(h)}{M\varepsilon^2}$\footnote{$\sigma^2(h)=\sigma_2^2(h)+2\sigma_1^2(h)\ge2\sigma_1^2(h)$.}, which is the Bernstein's variance term $\widehat\sigma_{M,\varepsilon}^2$ of the concentration of $\widehat W_n^\varepsilon$. Therefore, the concentration rate that we will derive in the next section for $\widehat\theta_\mathrm{MoIU}$ will also apply to $\widetilde\theta_\mathrm{MoIU}$ (with exactly the same constants in the bound and in the configurations of $n,M,K$). 
	\end{remark}
	
	\subsection{Bernstein Version of Theorem \ref{thm:univariate-moiu}}
	
	\begin{proof} (Theorem \ref{thm:univariate-moiu-bernstein})
		Let $\widehat\sigma_{M,\varepsilon}^2$ be defined as in Proposition \ref{prop:ugly-vstats-bernstein-bound-appendix}. Following the same arguments as Theorem \ref{thm:univariate-moiu}, for any $\tau\in(0,\frac{1}{2})$, we have:
		\begin{align*}
			&\P\Big(|\widehat\theta_\mathrm{MoIU}(h)-\theta(h)|\ge\varepsilon\Big) \\ 
			&\le \exp\left(-2K\left(\frac{1}{2}-\tau\right)^2\right) + \P(\widehat W_n^\varepsilon - \widehat q_\varepsilon\ge\tau-\widehat q_\varepsilon) \\
			&\le\exp\left(-2K\left(\frac{1}{2}-\tau\right)^2\right) + \exp\left(-\frac{n(\tau-\widehat q_\varepsilon)^2}{8M\left(\widehat\sigma_{M,\varepsilon}^2+\frac{\tau-\widehat q_\varepsilon}{3}\right)}\right) \qquad \text{(Proposition \ref{prop:ugly-vstats-bernstein-bound-appendix})}.
		\end{align*}
		\noindent By Chebyshev's inequality, we have:
		\begin{align*}
			\tau-\widehat q_\varepsilon &\ge \tau - \frac{1}{\varepsilon^2}\left[\frac{4\sigma_1^2(h)}{n}+\frac{2\sigma_2^2(h)}{n(n-1)}+\frac{\sigma^2(h)}{M}\right] \triangleq \omega_{n,\tau,\varepsilon}(h).
		\end{align*}
		\noindent Since $\exp\left[-\frac{nt^2}{8M(\widehat\sigma_{M,\varepsilon}^2+\frac{t}{3})}\right]$ is a decreasing\footnote{Let $g(t)=\frac{t^2}{\widehat\sigma_{M,\varepsilon}^2+\frac{t}{3}}$, then $g'(t)=\frac{t(2\widehat\sigma_{M,\varepsilon}^2+t/3)}{(\widehat\sigma_{M,\varepsilon}^2+t/3)^2}>0$ for all $t>0$. Hence, $g(t)$ is increasing and $-g(t)$ is decreasing.} function for $t>0$, we have:
		\begin{align*}
			\P(|\widehat\theta_\mathrm{MoIU}(h)-\theta(h)|\ge\varepsilon) 
			&\le\exp\left(-2K\left(\frac{1}{2}-\tau\right)^2\right) + \exp\left(-\frac{n(\tau-\widehat q_\varepsilon)^2}{8M\left(\widehat\sigma_{M,\varepsilon}^2+\frac{\tau-\widehat q_\varepsilon}{3}\right)}\right) \\
			&\le\exp\left(-2K\left(\frac{1}{2}-\tau\right)^2\right) + \exp\left(-\frac{n\omega_{n,\tau,\varepsilon}^2(h)}{8M\left(\widehat\sigma_{M,\varepsilon}^2+\frac{1}{3}\omega_{n,\tau,\varepsilon}(h)\right)}\right).
		\end{align*}
		\noindent Let $0\le \gamma< \tau$ be a free variable such that:
		\begin{align*}
			\varepsilon^2 &= \frac{1}{\gamma}\left[\frac{4\sigma_1^2(h)}{n}+\frac{2\sigma_2^2(h)}{n(n-1)}+\frac{\sigma^2(h)}{M}\right].
		\end{align*}
		\noindent In other words, we have $\omega_{n,\tau,\varepsilon}(h)=\tau-\gamma$. Suppose we know a-priori that:
		\begin{align}
			\label{eq:apriori-assumption}
			\frac{\omega_{n,\tau,\varepsilon}(h)}{3} \ge\frac{2\sigma^2(h)}{M\varepsilon^2}\ge\widehat\sigma_{M,\varepsilon}^2 \quad\text{or}\quad \varepsilon^2\ge\frac{6\sigma^2(h)}{M(\tau-\gamma)}\tag{*}.
		\end{align}
		\noindent For instance, setting $\gamma=\frac{\tau}{7}$ satisfies the condition in Eqn.~\eqref{eq:apriori-assumption}. Here, recall that $\widehat\sigma_{M,\varepsilon}^2=\frac{2\sigma^2(h)}{M\varepsilon^2}$ comes from Proposition \ref{prop:ugly-vstats-bernstein-bound-appendix}. As a result, we have:
		\begin{align*}
			\P(|\widehat\theta_\mathrm{MoIU}(h)-\theta(h)|\ge\varepsilon) 
			&\le\exp\left(-2K\left(\frac{1}{2}-\tau\right)^2\right) + \exp\left(-\frac{n\omega_{n,\tau,\varepsilon}^2(h)}{8M\left(\widehat\sigma_{M,\varepsilon}^2+\frac{1}{3}\omega_{n,\tau,\varepsilon}(h)\right)}\right) \\
			&\le \exp\left(-2K\left(\frac{1}{2}-\tau\right)^2\right) + \exp\left(-\frac{n\omega_{n,\tau,\varepsilon}^2(h)}{\frac{16}{3}M\omega_{n,\tau,\varepsilon}(h)}\right) \\
			&= \exp\left(-2K\left(\frac{1}{2}-\tau\right)^2\right) + \exp\left(-\frac{3n\omega_{n,\tau,\varepsilon}(h)}{16M}\right) \\
			&= \exp\left(-2K\left(\frac{1}{2}-\tau\right)^2\right) + \exp\left(-\frac{3n(\tau-\gamma)}{16M}\right).
		\end{align*}
		\noindent Similar to Theorem \ref{thm:univariate-moiu}, for $\delta\in(0,1)$, choosing $K\ge \frac{\log(2/\delta)}{2(\frac{1}{2}-\tau)^2}$ yields $\exp\left(-2K\left(\frac{1}{2}-\tau\right)^2\right)\le \delta/2$. In order for $\exp\left(-\frac{3n(\tau-\gamma)}{16M}\right)\le \delta/2$ to hold, we need $M\le\frac{3n(\tau-\gamma)}{16\log(2/\delta)}$. Let $M=\lfloor \frac{3n(\tau-\gamma)}{16\log(2/\delta)}\rfloor$, suppose that $n\ge\frac{16\log(2/\delta)}{3(\tau-\gamma)}$, we have $M\ge\frac{3n(\tau-\gamma)}{32\log(2/\delta)}$ since $M\ge 1$. Hence, we have:
		\begin{align*}
			\varepsilon^2 &= \frac{1}{\gamma}\left[\frac{4\sigma_1^2(h)}{n}+\frac{2\sigma_2^2(h)}{n(n-1)}+\frac{\sigma^2(h)}{M}\right] \\
			&\le \frac{1}{\gamma}\left[\frac{4\sigma_1^2(h)}{n}+\frac{2\sigma_2^2(h)}{n(n-1)}+\frac{32\sigma^2(h)\log(2/\delta)}{3n(\tau-\gamma)}\right].
		\end{align*}
		\noindent Setting $\gamma=\frac{\tau}{7}$ yields:
		\begin{align*}
			\varepsilon^2 
			&\le \frac{28\sigma_1^2(h)}{\tau n}+\frac{14\sigma_2^2(h)}{\tau n(n-1)}+\frac{32\sigma^2(h)\log(2/\delta)}{\frac{18}{49}\tau^2n} \\
			&\le  \frac{28\sigma_1^2(h)}{\tau n}+\frac{14\sigma_2^2(h)}{\tau n(n-1)}+\frac{88\sigma^2(h)\log(2/\delta)}{\tau^2n}.
		\end{align*}
		\noindent As a result, for $n\ge\frac{56\log(2/\delta)}{9\tau}$, $K\ge \frac{\log(2/\delta)}{2(\frac{1}{2}-\tau)^2}$ and $M=\lfloor \frac{9n\tau}{56\log(2/\delta)}\rfloor$, we have:
		\begin{align*}
			|\widehat\theta_\mathrm{MoIU}(h)-\theta(h)|\le \sqrt{\frac{28\sigma_1^2(h)}{\tau n}+\frac{14\sigma_2^2(h)}{\tau n(n-1)}+\frac{88\sigma^2(h)\log(2/\delta)}{\tau^2n}},
		\end{align*}
		\noindent with probability of at least $1-\delta$, as desired.
	\end{proof}
	\begin{remark}
		Extending the above sharpened result (from $O(\tau^{-3/2})$ to $O(\tau^{-1})$) is straightforward for both $\widetilde\theta_\mathrm{MoIU}$ and $\widecheck\theta_\mathrm{MoIU}$ by applying Proposition \ref{prop:bernstein-bound-for-swor} for the concentration of $\widetilde W_n^\varepsilon$. Hence, we will omit the procedural extension in this work.
	\end{remark}
	
	\section{Impossibility Results}
	\subsection{Asymptotic Lower Bound}
	\begin{theorem}[Asymptotic Normality of U-Statistics]
		\label{thm:clt_ustats}
		Let $S_n=\{X_1,\dots,X_n\}$ be i.i.d. random variables and $h:\mathcal{Z}\times\mathcal{Z}\to\R$ be a symmetric pairwise kernel with finite second moment. Let $U_n\triangleq\frac{1}{\binom{n}{2}}\sum_{(i,j)\in\Lambda_n}h(X_i,X_j)$, i.e., the complete order-$2$ U-Statistic and $\theta(h)=\E U_n$. Then:
		\begin{align}
			\sqrt{n}(U_n-\theta(h)) \xrightarrow{d} \mathcal{N}(0,4\sigma_1^2(h)).
		\end{align}
		\noindent Here, $\sigma_1^2(h)=\mathrm{Var}_{X_1}\left(\E_{X_2}[h(X_1,X_2)|X_1]\right)>0$, i.e., the first-order Hoeffding projection variance.
	\end{theorem}
	\begin{proof} (Sketch)
		Since this is a classic result, we provide a rough proof sketch without going into rigorous details. The main idea is to show that the asymptotic distribution of $U_n$ is determined by its first-order Hoeffding projection. Let $h^\dagger(\z)\triangleq\E_X[h(\z, X)] - \theta(h)$. We define $U_n^\dagger$ as follows:
		\begin{align}
			U_n^\dagger \triangleq \frac{2}{n}\sum_{j=1}^n h^\dagger(X_j).
		\end{align}
		\noindent Then, we can write $U_n$ as the sum of $U^\dagger_n$ and a residual term $R_n$. Specifically:
		\begin{align}
			U_n - \theta(h) &=  U_n^\dagger + (R_n-\theta(h)),\qquad R_n = U_n - U^\dagger_n.
		\end{align}
		\noindent First, for all $j\in[n]$, we have $\E[h^\dagger(X_j)]=0$ and $\mathrm{Var}(h^\dagger(X_j))=\sigma_1^2(h)$. Therefore, by the classic Central Limit Theorem, we have:
		\begin{align*}
			\sqrt{n}U_n^\dagger = \frac{2}{\sqrt{n}}\sum_{j=1}^nh^\dagger (X_j)\xrightarrow{d}\mathcal{N}(0,4\sigma_1^2(h)).
		\end{align*}
		\noindent Then, we can prove that $\P(\sqrt{n}|R_n-\theta(h)|\ge t)\to 0$ as $n\to\infty$ for all $t>0$, i.e., $\sqrt{n}(R_n-\theta(h))\xrightarrow{p}0$. This can be shown via Chebyshev's inequality and the fact that $\mathrm{Var}(R_n)\le O(n^{-2})$. As a result, by Slutsky's Lemma, we have:
		\begin{align*}
			\sqrt{n}(U_n-\theta(h)) &= \sqrt{n}U_n^\dagger + \sqrt{n}(R_n-\theta(h)) \xrightarrow{d} \mathcal{N}(0,4\sigma_1^2(h)),
		\end{align*}
		\noindent as desired.
	\end{proof}
	\begin{proposition}[Failure Result for $M\in\Omega(Kn)$]
		Let $Z\sim\mathcal{N}(0,1)$ and for all $\gamma>0$, we define $\Phi_\gamma\triangleq\P(|Z|\ge\gamma)$. Then, we have:
		\begin{align}
			M\ge Kn\cdot\frac{2\sigma^2(h)}{\Phi_\gamma\gamma^2\sigma_1^2(h)}\implies\P\left(\left|\widehat\theta_\mathrm{MoIU}(h)-\theta(h)\right|\ge\frac{\gamma\sigma_1(h)}{\sqrt{n}}\right) \ge f_n(\gamma),
		\end{align}
		\noindent where $f_n$ is a sequence of functions with $\lim_{n\to\infty}f_n(\gamma) = \frac{\Phi_\gamma}{2}>0$.
	\end{proposition}
	\begin{proof}
		Let $\varepsilon>0$ and $U_n\triangleq\frac{1}{\binom{n}{2}}\sum_{(i,j)\in\Lambda_n}h(X_i,X_j)$, we define two events $E, G$ as follows:
		\begin{align}
			E=\Big\{\left|U_n-\theta(h)\right|\ge\varepsilon\Big\}, \quad G = \bigcap_{k=1}^K\Big\{\left|\widehat U_{M,\J_k}(h)-U_n\right|\le\frac{\varepsilon}{2}\Big\}.
		\end{align}
		\noindent First, we note that under $E\cap G$\footnote{The complete U-Statistic $U_n$ deviates away substantially from $\theta(h)$ but all incomplete U-Statistics are close to $U_n$. This makes the median deviates substantially from $\theta(h)$.} implies $|\widehat\theta_\mathrm{MoIU}(h)-\theta(h)|\ge\frac{\varepsilon}{2}$. Specifically:
		\begin{enumerate}[label=(\roman*)]
			\item \textbf{When $U_n\ge\theta(h)+\varepsilon$}: Then, under $G$ we have $\widehat U_{M,\J_k}(h)\ge U_n - \frac{\varepsilon}{2}\ge\theta(h)+\frac{\varepsilon}{2}$ for all $k\in[K]$.
			\item \textbf{When $U_n\le\theta(h)-\varepsilon$}: Then, under $G$ we have $\widehat U_{M,\J_k}(h)\le U_n + \frac{\varepsilon}{2}\le \theta(h)-\frac{\varepsilon}{2}$ for all $k\in[K]$.
		\end{enumerate}
		\noindent As a result, we have:
		\begin{align*}
			\P\left(\left|\widehat\theta_\mathrm{MoIU}(h)-\theta(h)\right|\ge\frac{\varepsilon}{2}\right) &\ge \P(E\cap G) \ge \P(E) - \P(G^c)\qquad\text{(By Bonferroni)}.
		\end{align*}
		\noindent To control $\P(G^c)$, we have:
		\begin{align*}
			\P(G^c) &= \P\left(\bigcup_{k=1}^K\left\{\left|\widehat U_{M,\J_k}(h)-U_n\right|\ge\frac{\varepsilon}{2}\right\}\right) = \E_{S_n}\left[ \P\left(\bigcup_{k=1}^K\left\{\left|\widehat U_{M,\J_k}(h)-U_n\right|\ge\frac{\varepsilon}{2}\right\}\Big|S_n\right)\right] \\
			&\le \sum_{k=1}^K\E_{S_n}\left[\P\left(\left|\widehat U_{M,\J_k}(h)-U_n\right|\ge\frac{\varepsilon}{2}\Big|S_n\right)\right] \quad \text{(Union Bound)}.
		\end{align*}
		\noindent Conditionally given $S_n$, we have $\E\left[\widehat U_{M,\J_k}(h)|S_n\right]=U_n$. Hence, by Chebyshev's inequality:
		\begin{align*}
			\P\left(\left|\widehat U_{M,\J_k}(h)-U_n\right|\ge\frac{\varepsilon}{2}\Big|S_n\right) &\le \frac{4\mathrm{Var}(\widehat U_{M,\J_k}(h)|S_n)}{\varepsilon^2} = \frac{4\widehat V_n^2}{M\varepsilon^2}.
		\end{align*}
		\noindent Here, $\widehat V_n^2\triangleq\frac{1}{\binom{n}{2}}\sum_{(i,j)\in\Lambda_n}h^2(X_i,X_j) - U_n^2$ following the same calculation as in Corollary \ref{cor:variance-bounds-mm}. We know that $\E[\widehat V_n^2]\le\sigma^2(h)$. Hence, we have:
		\begin{align*}
			\P(G^c) &\le \frac{4K\E[\widehat V_n^2]}{M\varepsilon^2} \le \frac{4K\sigma^2(h)}{M\varepsilon^2}.
		\end{align*}
		\noindent Let $\varepsilon=\frac{2\gamma\sigma_1(h)}{\sqrt{n}}$ and $Z_n=n^\frac{1}{2}\left(\frac{U_n-\theta(h)}{2\sigma_1(h)}\right)$, we have:
		\begin{align*}
			\P\left(\left|\widehat\theta_\mathrm{MoIU}(h)-\theta(h)\right|\ge\frac{\gamma\sigma_1(h)}{\sqrt{n}}\right) &\ge \P(E) - \P(G^c) \\
			&\ge \P(|U_n-\theta(h)|\ge\varepsilon) - \frac{4K\sigma^2(h)}{M\varepsilon^2} \\
			&= \P\left(n^\frac{1}{2}\left|\frac{U_n-\theta(h)}{2\sigma_1(h)}\right|\ge\gamma\right)- \frac{Kn\sigma^2(h)}{M\gamma^2\sigma_1^2(h)} \\
			&= \P(|Z_n|\ge\gamma) - \frac{Kn\sigma^2(h)}{M\gamma^2\sigma_1^2(h)}.
		\end{align*}
		\noindent By the initial assumption that $M\ge Kn\cdot\frac{2\sigma^2(h)}{\Phi_\gamma\gamma^2\sigma_1^2(h)}$, we have:
		\begin{align*}
			\P\left(\left|\widehat\theta_\mathrm{MoIU}(h)-\theta(h)\right|\ge\frac{\gamma\sigma_1(h)}{\sqrt{n}}\right) &\ge \P(|Z_n|\ge\gamma) - \frac{\Phi_\gamma}{2}.
		\end{align*}
		\noindent Let $f_n(\gamma)=\P(|Z_n|\ge\gamma)-\frac{\Phi_\gamma}{2}$. By Theorem \ref{thm:clt_ustats}, we have $Z_n=n^\frac{1}{2}\left(\frac{U_n-\theta(h)}{2\sigma_1(h)}\right)\xrightarrow{d}Z\sim\mathcal{N}(0,1)$. Therefore, we have $\P(|Z_n|\ge\gamma)\to\P(|Z|\ge\gamma)=\Phi_\gamma$ as $n\to\infty$ and:
		\begin{align*}
			\lim_{n\to\infty}f_n(\gamma) = \lim_{n\to\infty}\P(|Z_n|\ge\gamma)-\frac{\Phi_\gamma}{2} = \frac{\Phi_\gamma}{2}>0,
		\end{align*}
		\noindent as desired.
	\end{proof}
	\begin{remark}
		We note that the lower bound $M\ge\Omega(Kn)$ is an artefact of how we design the two failure events $E$ and $G$. In order for $|\widehat\theta_\mathrm{MoIU}(h)-\theta(h)|\ge\frac{\varepsilon}{2}$ to hold, we only need the event $E$ and the event that more than half of the incomplete U-Statistics are too close to $U_n$. As a result, we can avoid the union bound step that is used to bound $\P(G^c)$ above.
	\end{remark}
	
	\begin{proposition}[Sharpened Result for $M\in\Omega(n)$]
		\label{prop:impossibility_result_general}
		Let $Z\sim\mathcal{N}(0,1)$ and for all $\gamma>0$, we define $\Phi_\gamma\triangleq\P(|Z|\ge\gamma)$. Then, we have:
		\begin{align}
			M\ge\frac{4n\sigma^2(h)}{\Phi_\gamma\gamma^2\sigma_1^2(h)}\implies\P\left(\left|\widehat\theta_\mathrm{MoIU}(h)-\theta(h)\right|\ge\frac{\gamma\sigma_1(h)}{\sqrt{n}}\right) \ge f_n(\gamma),
		\end{align}
		\noindent where $f_n$ is a sequence of functions with $\lim_{n\to\infty}f_n(\gamma) = \frac{\Phi_\gamma}{2}>0$.
	\end{proposition}
	\begin{proof}
		Let $\varepsilon>0$ and $U_n\triangleq\frac{1}{\binom{n}{2}}\sum_{(i,j)\in\Lambda_n}h(X_i,X_j)$. Also, define $B_{K,\varepsilon}\triangleq\sum_{k=1}^K\mathds{1}_{|\widehat U_{M,\J_k}(h)-U_n|\ge\varepsilon/2}$, i.e., the number of incomplete U-Statistics that are far from $U_n$. We define two events $E,\bar G$ as:
		\begin{align}
			E=\Big\{\left|U_n-\theta(h)\right|\ge\varepsilon\Big\}, \quad \bar G = \Bigg\{B_{K,\varepsilon}<\frac{K}{2}\Bigg\}.
		\end{align}
		\noindent Under $\bar G$, there exists $\mathcal{J}\subseteq[K]$ with $|\mathcal{J}|\ge\frac{K}{2}$ such that:
		\begin{align*}
			\forall \ell\in\mathcal{J}: \quad \left|\widehat U_{M,\J_\ell}(h) - U_n\right|\le\frac{\varepsilon}{2}.
		\end{align*}
		\noindent Then, we note that $E\cap\bar G$ implies $|\widehat\theta_\mathrm{MoIU}(h)-\theta(h)|\ge\frac{\varepsilon}{2}$. Specifically:
		\begin{enumerate}[label=(\roman*)]
			\item \textbf{When $U_n\ge\theta(h)+\varepsilon$}: Then, under $G$ we have $\widehat U_{M,\J_\ell}(h)\ge U_n - \frac{\varepsilon}{2}\ge\theta(h)+\frac{\varepsilon}{2}$ for all $\ell\in\mathcal{J}$.
			\item \textbf{When $U_n\le\theta(h)-\varepsilon$}: Then, under $G$ we have $\widehat U_{M,\J_\ell}(h)\le U_n + \frac{\varepsilon}{2}\le \theta(h)-\frac{\varepsilon}{2}$ for all $\ell\in\mathcal{J}$.
		\end{enumerate}
		\noindent As a result, by Bonferroni inequality:
		\begin{align*}
			\P\left(\left|\widehat\theta_\mathrm{MoIU}(h)-\theta(h)\right|\ge\frac{\varepsilon}{2}\right) &\ge \P(E\cap \bar G) \ge \P(E) - \P(\bar G^c).
		\end{align*}
		\noindent Furthermore:
		\begin{align*}
			\P(\bar G^c) &= \P\Big(B_{K,\varepsilon}\ge\frac{K}{2}\Big) = \E_{S_n}\left[\P\Big(B_{K,\varepsilon}\ge\frac{K}{2}\Big|S_n\Big)\right].
		\end{align*}
		\noindent Conditionally given the dataset $S_n$, the indicators $\mathds{1}_{|\widehat U_{M,\J_k}(h)-U_n|\ge\varepsilon/2}$ are i.i.d. Bernoulli random variables with success probability $\P(|\widehat U_{M,\J_k}(h)-U_n|\ge\varepsilon/2|S_n)$. Then, by Markov's inequality:
		\begin{align*}
			\P\Big(B_{K,\varepsilon}\ge\frac{K}{2}\Big|S_n\Big) &\le 2\P\Big(|\widehat U_{M,\J_k}(h)-U_n|\ge\varepsilon/2|S_n\Big) \\
			&\le \frac{8\mathrm{Var}(\widehat U_{M,\J_k}(h)|S_n)}{\varepsilon^2} \qquad\text{(Chebyshev's Inequality)} \\
			&= \frac{8\widehat V_n^2}{M\varepsilon^2}.
		\end{align*}
		\noindent Therefore, we have:
		\begin{align*}
			\P(\bar G^c)&\le \frac{8\E[\widehat V_n^2]}{M\varepsilon^2} \le \frac{8\sigma^2(h)}{M\varepsilon^2}.
		\end{align*}
		\noindent Let $\varepsilon=\frac{2\gamma\sigma_1(h)}{\sqrt{n}}$ and $Z_n=n^\frac{1}{2}\left(\frac{U_n-\theta(h)}{2\sigma_1(h)}\right)$, we have:
		\begin{align*}
			\P\left(\left|\widehat\theta_\mathrm{MoIU}(h)-\theta(h)\right|\ge\frac{\gamma\sigma_1(h)}{\sqrt{n}}\right) &\ge \P(E) - \P(\bar G^c) \\
			&\ge \P\left(n^\frac{1}{2}\left|\frac{U_n-\theta(h)}{2\sigma_1(h)}\right|\ge\gamma\right)- \frac{2n\sigma^2(h)}{M\gamma^2\sigma_1^2(h)} \\
			&= \P(|Z_n|\ge\gamma) - \frac{2n\sigma^2(h)}{M\gamma^2\sigma_1^2(h)}.
		\end{align*}
		\noindent By the initial assumption that $M\ge\frac{4n\sigma^2(h)}{\Phi_\gamma\gamma^2\sigma_1^2(h)}$, we have:
		\begin{align*}
			\P\left(\left|\widehat\theta_\mathrm{MoIU}(h)-\theta(h)\right|\ge\frac{\gamma\sigma_1(h)}{\sqrt{n}}\right) &\ge \P(|Z_n|\ge\gamma) - \frac{\Phi_\gamma}{2} = f_n(\gamma).
		\end{align*}
		\noindent By Theorem \ref{thm:clt_ustats}, we have $\P(|Z_n|\ge\gamma)\to\P(|Z|\ge\gamma)=\Phi_\gamma$ as $n\to\infty$ and $\lim_{n\to\infty}f_n(\gamma)$.
	\end{proof}
	
	\begin{remark}[Applicability for $\widetilde\theta_\mathrm{MoIU}$]
		\label{rem:application-to-tilde-moiu}
		The above $M\ge\Omega(n)$ result's proof will apply immediately for the impossibility result of $\widetilde\theta_\mathrm{MoIU}$. The key is to define the event $\bar G=\Big\{B_{K,\varepsilon}<\frac{K}{2}\Big\}$ where $B_{K,\varepsilon}=\sum_{k=1}^K\mathds{1}_{|\widetilde U_{M,\J_k}(h)-U_n|\ge\varepsilon/2}$ and bound $\P(\bar G^c)$. The only difference from the above proof is that $\P(\bar G^c)$ is now controlled by the conditional variance of $\widetilde U_{M,\J_k}(h)$, which, in expectation is upper bounded by $\sigma^2(h)$ as well. Hence, the above theorem also holds with exact constants for $\widetilde\theta_\mathrm{MoIU}$.
	\end{remark}
	
	\begin{remark}[Applicability for $\widecheck\theta_\mathrm{MoIU}$]
		\label{rem:application-to-check-moiu}
		In this case, the analogous definition of event $\bar G$ will be slightly different in notation. Specifically, we define:
		\begin{align}
			\bar G = \Bigg\{\sum_{k=1}^K Y_{k,\varepsilon}<\frac{K}{2} \Bigg\},\qquad Y_{k,\varepsilon} = \mathds{1}_{|\widecheck U_{M,\J_k}(h)-U_n|\ge\varepsilon/2}.
		\end{align}
		\noindent Here, $\J_1,\dots,\J_K$ are selected \textbf{without replacement across blocks}, making $Y_{k,\varepsilon}$ non-independent across blocks. However, this does not prevent us from bounding $\P(\bar G^c)$ using Markov's inequality in the same spirit as in Proposition \ref{prop:impossibility_result_general}. Specifically:
		\begin{align*}
			\P\left(\sum_{k=1}^KY_{k,\varepsilon}\ge\frac{K}{2}\Big|S_n\right) &\le\frac{2\sum_{k=1}^K\E[Y_{k,\varepsilon}|S_n]}{K}=2\P\Big(|\widecheck U_{M,\J_k}(h)-U_n|\ge\varepsilon/2|S_n\Big).
		\end{align*}
		\noindent As explained before, even though $Y_{k,\varepsilon},\forall k\in[K]$ are dependent via selection without replacement across blocks, they are marginally identically distributed, i.e., Bernoulli variables with the same success probability. By Chebyshev's inequality, we have:
		\begin{align*}
			\P\Big(|\widecheck U_{M,\J_k}(h)-U_n|\ge\varepsilon/2|S_n\Big) &\le \frac{4\mathrm{Var}(\widecheck U_{M,\J_k}(h)|S_n)}{\varepsilon^2}.
		\end{align*}
		\noindent Finally, we have:
		\begin{align*}
			\P(\bar G^c) &= \E_{S_n}\left[\P\left(\sum_{k=1}^KY_{k,\varepsilon}\ge\frac{K}{2}\Big|S_n\right)\right] \le \frac{8}{\varepsilon^2}\E_{S_n}\left[\mathrm{Var}(\widecheck U_{M,\J_k}(h)|S_n)\right] \le \frac{8\sigma^2(h)}{M\varepsilon^2}.
		\end{align*}
		\noindent Then, we lower bound the desired probability following the subsequent steps in Proposition \ref{prop:impossibility_result_general}.
	\end{remark}
	
	\noindent As mentioned in the main text, although the above results do not require any information about the distribution, the constant $\frac{1}{\gamma^2\Phi_\gamma}$ is vacuous even for moderate values of $\gamma$. In the following section, we will derive a lower bound that matches Chebyshev's rate when $M\ge Cn$ (where $C>0$ does not depend on $\gamma$) using a pathological distribution constructed in \citet{article:catoni2012}.
	\subsection{Catoni-type Lower Bound}
	\label{sec:catoni-lowerbound}
	\begin{lemma}[Catoni's Lower Bound, cf. \citet{article:catoni2012}, Proposition 6.2]
		\label{lem:catoni-lowerbound}
		For any $\sigma^2>0$ and deviation level $\delta\le\frac{1}{2e}$, there exists a distribution $\mathcal{P}_{\sigma,\delta}$ such that $\E_{\mathcal{P}_{\sigma,\delta}}[X]=0, \E_{\mathcal{P}_{\sigma,\delta}}[X^2]=\sigma^2$ and:
		\begin{align}
			\P\left(|\widehat\mu_n|\ge\frac{\sigma}{\sqrt{2n\delta}}\left(1-\frac{2e\delta}{n}\right)^\frac{n-1}{2}\right) \ge 2\delta.
		\end{align}
		\noindent Here, $\widehat\mu_n\triangleq\frac{1}{n}\sum_{j=1}^nX_j$, i.e., the sample mean of $X_1,\dots,X_n\sim_\mathrm{i.i.d.}\mathcal{P}_{\sigma,\delta}^{\otimes n}$.
	\end{lemma}
	
	\begin{proof}
		We construct $\mathcal{P}_{\sigma,\delta}$ as a categorical distribution with alphabet $\{-n\xi,0,n\xi\}$ where $\xi>\frac{\sigma}{n}$ depends on $\delta$ and for a random variable $X\sim\mathcal{P}$:
		\begin{align}
			\P(X=n\xi) = \P(X=-n\xi) = \frac{\sigma^2}{2n^2\xi^2},\qquad\P(X=0) = 1- \frac{\sigma^2}{n^2\xi^2}.
		\end{align}
		\noindent Then, we have:
		\begin{align*}
			\P(\widehat\mu_n\ge\xi) \ge \P(\widehat\mu_n=\xi) &\ge \P\Big(\exists j\in[n]: X_j=n\xi, X_k=0,\forall k\ne j\Big) \\
			&= n\times\frac{\sigma^2}{2n^2\xi^2}\left(1-\frac{\sigma^2}{\xi^2n^2}\right)^{n-1} \\
			&= \frac{\sigma^2}{2n\xi^2}\left(1-\frac{\sigma^2}{\xi^2n^2}\right)^{n-1}.
		\end{align*}
		\noindent Similarly, we have $\P(\widehat\mu_n\le-\xi)\ge\frac{\sigma^2}{2n\xi^2}\left(1-\frac{\sigma^2}{\xi^2n^2}\right)^{n-1}$. As a result:
		\begin{align*}
			\P(|\widehat\mu_n|\ge\xi) &= \P(\widehat \mu_n\ge\xi) + \P(\widehat\mu_n\le-\xi) \ge \frac{\sigma^2}{n\xi^2}\left(1-\frac{\sigma^2}{\xi^2n^2}\right)^{n-1}.
 		\end{align*}
 		\noindent \textbf{(i) Picking $\xi^2$}: Let $\frac{\sigma^2}{n\xi^2}\left(1-\frac{\sigma^2}{\xi^2n^2}\right)^{n-1}=2\delta'$, we need to pick $\xi^2$ satisfying $\frac{1}{2e}\ge\delta'\ge\delta$. Assume that there exists $\xi^2$ where $\delta'\ge\delta$ and $\xi^2\ge\frac{\sigma^2}{n}$ (or $\frac{\sigma^2}{n\xi^2}\le 1$) are both satisfied. Then, we have:
 		\begin{align*}
 			\xi^2 &= \frac{\sigma^2}{2n\delta'}\left(1-\frac{\sigma^2}{\xi^2n^2}\right)^{n-1} \ge \frac{\sigma^2}{2n\delta'}\left(1-\frac{1}{n}\right)^{n-1}\ge\frac{\sigma^2}{2en\delta'}.
 		\end{align*}
 		\noindent Therefore, we have:
 		\begin{align*}
 			\xi^2 &\ge  \frac{\sigma^2}{2n\delta'}\left(1-\frac{\sigma^2}{n^2\times\frac{\sigma^2}{2en\delta'}}\right)^{n-1} = \frac{\sigma^2}{2n\delta'}\left(1-\frac{2e\delta'}{n}\right)^{n-1}.
 		\end{align*}
 		\noindent From the assumption that $\delta\le\delta'\le\frac{1}{2e}$, we can pick: 
 		\begin{align*}
 			\xi^2 &= \frac{\sigma^2}{2n\delta}\left(1-\frac{2e\delta}{n}\right)^{n-1} \ge \frac{\sigma^2}{2n\delta'}\left(1-\frac{2e\delta'}{n}\right)^{n-1}.
 		\end{align*}
 		\noindent \textbf{(ii) Verify that $\xi^2\ge\frac{\sigma^2}{n}$}: we have $\xi^2=\frac{\sigma^2}{n}\cdot\frac{1}{2\delta}\left(1-\frac{2e\delta}{n}\right)^{n-1}$ and:
 		\begin{align*}
 			\frac{1}{2\delta}\left(1-\frac{2e\delta}{n}\right)^{n-1} \ge e\left(1-\frac{1}{n}\right)^{n-1} \ge 1\qquad \Big(\text{Since } \delta\le\frac{1}{2e}\Big).
 		\end{align*}
 		
 		\textbf{(iii) Verify that $\frac{\sigma^2}{n\xi^2}\left(1-\frac{\sigma^2}{\xi^2n^2}\right)^{n-1}\ge2\delta$}: First, note that
 		\begin{align*}
 			\xi^2 =  \frac{\sigma^2}{2n\delta}\left(1-\frac{2e\delta}{n}\right)^{n-1} \ge \frac{\sigma^2}{2n\delta}\left(1-\frac{1}{n}\right)^{n-1} \ge \frac{\sigma^2}{2en\delta}.
 		\end{align*}
 		\noindent Plug the value of $\xi^2$ into the lowerbound of $\P(|\widehat \mu_n|\ge\xi)$:
 		\begin{align*}
 			\frac{\sigma^2}{n\xi^2}\left(1-\frac{\sigma^2}{\xi^2n^2}\right)^{n-1} &= 2\delta\left(1-\frac{2e\delta}{n}\right)^{1-n}\left(1-\frac{\sigma^2}{\xi^2n^2}\right)^{n-1} \\
 			&\ge 2\delta\left(1-\frac{2e\delta}{n}\right)^{1-n}\left(1-\frac{\sigma^2}{n^2\times\frac{\sigma^2}{2en\delta}}\right)^{n-1} \qquad \Big(\text{Since }\xi^2\ge\frac{\sigma^2}{2en\delta}\Big) \\
 			&\ge 2\delta\left(1-\frac{2e\delta}{n}\right)^{1-n}\left(1-\frac{2e\delta}{n}\right)^{n-1} = 2\delta.
 		\end{align*}
 		
 		\noindent Finally, for the constructed distribution $\mathcal{P}_{\sigma,\delta}$, we have:
 		\begin{align*}
 			\P(|\widehat{\mu}_n|\ge\xi)=\P\left(|\widehat \mu_n|\ge\frac{\sigma}{\sqrt{2n\delta}}\left(1-\frac{2e\delta}{n}\right)^\frac{n-1}{2}\right) \ge  \frac{\sigma^2}{n\xi^2}\left(1-\frac{\sigma^2}{\xi^2n^2}\right)^{n-1} \ge 2\delta,
 		\end{align*}
 		\noindent as desired.
	\end{proof}
	
	\begin{proposition}[Proposition \ref{prop:catoni-type-lowerbound-main}, Restated]
		\label{prop:catoni-type-lowerbound-appendix}
		Define kernel $h(x,y)=\frac{1}{2}(x+y)$. Then, for all $\gamma\ge \frac{1}{2}$, there exists a distribution $\mathcal{P}_{\sigma,\gamma}$ (from which $X_1,\dots,X_n\sim_\mathrm{i.i.d.}\mathcal{P}^{\otimes n}_{\sigma,\gamma}$) with $\E_{\mathcal{P}_{\sigma,\gamma}}[X]=0$ and $\E_{\mathcal{P}_{\sigma,\gamma}}[X^2]=\sigma^2$ such that:
		\begin{align}
			M\ge 8en\implies \P\left(\left|\widehat\theta_\mathrm{MoIU}(h)\right|\ge\frac{\gamma\sigma}{\sqrt{n}}\right) \ge\frac{1}{8e\gamma^2}.
		\end{align}
	\end{proposition}
	\begin{proof}
		Let $\delta=\frac{1}{8e\gamma^2}$ (then we have $\delta\le\frac{1}{2e}$ since $\gamma\ge\frac{1}{2}$), then we can construct $\mathcal{P}_{\sigma,\gamma}$ as the categorical distribution over the alphabet $\{-n\xi,0,n\xi\}$ where $\xi^2=\frac{\sigma^2}{2n\delta}\left(1-\frac{2e\delta}{n}\right)^{n-1}=\frac{4e\sigma^2\gamma^2}{n}\left(1-\frac{1}{4n\gamma^2}\right)^{n-1}$. From Lemma \ref{lem:catoni-lowerbound}, we have:
		\begin{align*}
			\P\left(|\widehat\mu_n|\ge\frac{\sigma}{\sqrt{2n\delta}}\left(1-\frac{2e\delta}{n}\right)^\frac{n-1}{2}\right) \ge 2\delta,
		\end{align*}
		\noindent where $\widehat\mu_n\triangleq\frac{1}{n}\sum_{j=1}^nX_j$, i.e., the sample mean of $X_1,\dots,X_n$. From the definition of the kernel $h$, we have $U_n\triangleq\frac{1}{\binom{n}{2}}\sum_{(i,j)\in\Lambda_n}h(X_i,X_j)=\widehat\mu_n$. As a result:
		\begin{align*}
			\P\left(|U_n|\ge\frac{\sigma}{\sqrt{2en\delta}}\right) &= \P\left(|\widehat\mu_n|\ge\frac{\sigma}{\sqrt{2en\delta}}\right)\\
			&\ge \P\left(|\widehat\mu_n|\ge\frac{\sigma}{\sqrt{2n\delta}}\left(1-\frac{1}{n}\right)^\frac{n-1}{2}\right) \\
			&\ge \P\left(|\widehat\mu_n|\ge\frac{\sigma}{\sqrt{2n\delta}}\left(1-\frac{2e\delta}{n}\right)^\frac{n-1}{2}\right) \qquad(\text{Since } 2\delta\le 1/e)\\
			&\ge 2\delta.
		\end{align*}
		\noindent Since $\frac{1}{\sqrt{2e\delta}}=2\gamma$, the above inequality can be re-written as $\P\left(|U_n|\ge\frac{2\gamma\sigma}{\sqrt{n}}\right) \ge \frac{1}{4e\gamma^2}$. Let $\varepsilon=2\gamma\sigma/\sqrt{n}$ and define the events $E,\bar G$ as in the proof of Proposition \ref{prop:impossibility_result_general}, we have:
		\begin{align*}
			\P\left(\left|\widehat\theta_\mathrm{MoIU}(h)\right|\ge\frac{\gamma\sigma}{\sqrt{n}}\right) &\ge \P(E)-\P(\bar G^c) \ge \P(|U_n|\ge\varepsilon) - \frac{8\sigma^2(h)}{M\varepsilon^2} \\
			&= \P(|U_n|\ge\varepsilon) - \frac{4\sigma^2}{M\varepsilon^2} \qquad \left(\sigma^2(h) = \mathrm{Var}\left[\frac{1}{2}(X_1+X_2)\right]=\frac{\sigma^2}{2}\right) \\
			&= \P\left(|U_n|\ge\frac{2\gamma\sigma}{\sqrt{n}}\right) - \frac{n}{M\gamma^2} \\
			&\ge \frac{1}{4e\gamma^2} - \frac{n}{M\gamma^2} \\
			&\ge \frac{1}{8e\gamma^2} \qquad (M\ge8en),
		\end{align*}
		\noindent as desired.
	\end{proof}
	\begin{remark}
		The above lower bound easily extends to $\widetilde\theta_\mathrm{MoIU}$ and $\widecheck\theta_\mathrm{MoIU}$. The only difference is in the chain of Markov/Chebyshev inequalities that bounds $\P(\bar G^c)$, which we have discussed at length in Remark \ref{rem:application-to-tilde-moiu} and Remark \ref{rem:application-to-check-moiu}.
	\end{remark}
\end{document}